\documentclass[11pt]{article}
\usepackage[utf8]{inputenc}
\usepackage{amsmath}
\usepackage{amsthm}
\usepackage{amsfonts}
\usepackage{natbib}
\usepackage{bm}
\usepackage{thm-restate}

\usepackage{authblk}

\usepackage{graphicx}
\usepackage{subcaption}
\usepackage{float}

\usepackage{color}   %May be necessary if you want to color links
\usepackage{hyperref}
\hypersetup{
    linktoc=all,     %set to all if you want both sections and subsections linked
    linkcolor=blue,  %choose some color if you want links to stand out
}

\def\reals{{\mathbb R}}
\def\bx{\mathbf{x}}
\def\by{\mathbf{y}}
\def\bu{\mathbf{u}}
\def\bv{\mathbf{v}}
\def\bw{\mathbf{w}}
\def\bz{\mathbf{z}}
\def\calW{\mathcal{W}}
\newcommand{\eps}{\epsilon}
\newcommand{\E}{\mathbb{E}}
\newcommand{\err}{\mathrm{err}}

\usepackage{framed}
\usepackage{setspace}
\usepackage{listings}
\spacing{1.06}
\usepackage[margin=1in]{geometry}

\usepackage{todonotes}

\usepackage[linesnumbered, ruled,vlined]{algorithm2e}

\def\pr{\mathbb{P}}
\newcommand{\ind}{\mathds{1}}
\newcommand{\cgoo}{\textrm{CGOO}}
\def\calA{\mathcal{A}}

\def\calC{\mathcal{C}}
\def\calD{\mathcal{D}}

\def\calH{\mathcal{H}}
\def\calJ{\mathcal{J}}
\def\calK{\mathcal{K}}

\def\calM{\mathcal{M}}

\def\calR{\mathcal{R}}
\def\calT{\mathcal{T}}
\def\calX{\mathcal{X}}

\def\calV{\mathcal{V}}
\def\calW{\mathcal{W}}
\def\calY{\mathcal{Y}}
\def\calZ{\mathcal{Z}}
\def\reals{{\mathbb R}}
\def\norm#1{\mathopen\| #1 \mathclose\|}% use instead of $\|x\|$
\newcommand{\brackets}[1]{\langle #1\rangle}

\def\supp{\text{supp}}

\def\bx{\text{\textbf{x}}}

\theoremstyle{plain}
\newtheorem{theo}{Theorem}[section]
 
\newtheorem{theorem}[theo]{Theorem}
\newtheorem{lemma}[theo]{Lemma}

\newtheorem{corollary}[theo]{Corollary}
\newtheorem{example}[theo]{Example}

\newtheorem{prop}[theo]{Proposition}

\newtheorem{definition}[theo]{Definition}

\newtheorem{thmi}[theo]{Theorem (Informal)}
\newtheorem{cori}[theo]{Corollary (Informal)}

\theoremstyle{remark}
\newtheorem{remark}[theo]{Remark}

\usepackage{dsfont}

\DeclareMathOperator*{\argmin}{argmin}

\DeclareMathOperator*{\diam}{diam}
\DeclareMathOperator*{\poly}{poly}

\DeclareMathOperator*{\cs}{CSC}
\DeclareMathOperator*{\lopt}{LOPT}
\DeclareMathOperator*{\rlopt}{RLOPT}

\DeclareMathOperator*{\smax}{smax}

\DeclareMathOperator*{\VC}{VC}
\DeclareMathOperator*{\sign}{sign}

\def\bw{\text{\textbf{w}}}

\def\Unif{\text{Unif}}
\def\loss{\ell(c, D)}
\def\bloss{\ell(c^*, D)}

\renewcommand \vec [1]{\bm{#1}}

\begin{document}

\pagenumbering{roman}

\title{The Cost of a Reductions Approach to Private Fair Optimization
%\footnote{
%This paper was presented at 
%the 2019 workshop on the Theory and Practice
%of Differential Privacy (TPDP).
%}
}
\date{}
\author[1]{
Daniel Alabi\footnote{School of Engineering and Applied Sciences, Harvard University.
Research partially supported by Facebook AI Research.
Email: \url{alabid@g.harvard.edu}.
}
}
%\affil[1]{}
\maketitle

%\noindent First Version: January 23, 2019\\
%\noindent This Version: \today

\begin{abstract}

Through the lens of information-theoretic reductions,
%[Brassard-Cr{\'{e}}peau-Robert, FOCS 1986],
we examine a reductions approach to fair optimization and
learning
where a black-box optimizer is used to learn a fair model for classification 
% [Agarwal-Beygelzimer-Dud\`ik-Langford-Wallach, 2018]
or 
regression.
% [Alabi-Immorlica-Kalai, 2018]
Quantifying the complexity,
both statistically and computationally, of
making such models satisfy the rigorous definition of
differential privacy is our end
goal. We resolve a few open questions and show
applicability to fair machine learning,
hypothesis testing, and to optimizing non-standard measures
of classification loss.
Furthermore, our sample complexity bounds are tight amongst all 
strategies that jointly minimize a composition of
functions.

The reductions approach to fair optimization can be
abstracted as the
\textit{constrained group-objective optimization} problem where we aim to optimize an
objective that is a function of losses of individual groups,
subject to some constraints. 
We give the first polynomial-time algorithms to solve
the problem with $(\eps, 0)$ or $(\eps, \delta)$ 
differential privacy guarantees
when defined on a convex decision set (for example,
the $\ell_P$ unit ball) with 
convex constraints and losses.
Accompanying information-theoretic lower bounds for the
problem are presented.
In addition, compared to a previous method for ensuring
differential privacy subject to a relaxed form
of the equalized odds fairness constraint,
the $(\eps, \delta)$-differentially private algorithm we present provides asymptotically
better sample complexity guarantees, resulting in an exponential
improvement in certain parameter regimes.
We introduce a class of bounded divergence linear
optimizers, which could be
of independent interest, and specialize to pure and approximate differential privacy.

\end{abstract}

\clearpage

\tableofcontents

\clearpage

\pagenumbering{arabic}
\section{Introduction}

Algorithmic fairness, accountability, and transparency of computer systems have become salient
sub-fields of study within computer science. The incorporation of such values has led to the development
of new models to make existing and state-of-the-art systems more conscious of societal constraints.
But some of these new models do not adhere to
other ethical standards. A standard of utmost importance
is the need to ensure the privacy of the individuals
that constitute the
data used to create the models.

Differential privacy has become a gold standard of 
(individual-level)
privacy in machine learning and data analysis~\citep{DworkMNS06}.
The adoption of this privacy definition
by the U.S. Census Bureau~\citep{Abowd18} and major tech
companies~\citep{CJKLSW18} is evidence of its impact.
Possibilities and implementations of reconstruction and inference
attacks clearly show the relevance of differential privacy
~\citep{Dinur:2003, 10.1371/journal.pgen.1000665, ChoromanskiM12,
SSSS17, GarfinkelAM18, Carlini0EKS19}.
Algorithmic fairness is also an increasingly important
requirement 
for deployed systems, that results in various
statistical tradeoffs~\citep{KleinbergMR17, KleinbergLMS20, FSV21}.
While the literature on differential privacy is guided by
(slight variants of) a singular information-theoretic
definition,
the fairness literature is burgeoning without much
clarity on what definitions are suitable for certain tasks.
In this work, we provide a \textit{generic way} to minimize
group-fairness objectives by reducing to minimizing
linear objectives.
Given the worst-case nature of differential privacy, 
the best estimator for natural objectives -- such as the
least squares objective -- depends on the
properties of the dataset~\citep{Sheffet19, komarova2020identification}.
As a consequence, the major advantage of the reductions
approach is to use existing machinery for differentially
private solvers to solve more general fairness objectives.
\textit{However, what is the cost (both statistically and
computationally) of this reductions approach?}
We aim to answer this question via the lens of
information-theoretic 
reductions~\citep{BrassardCR86, BennettBCM95}, using
tools from
Lagrangian Duality, Optimization, and Differential Privacy.

Our focus is on a reductions approach to fair optimization and
learning where
a black-box optimizer is used to learn a fair model for classification or regression
(\citep{AgarwalBD0W18, AlabiIK18}). We
explore the creation of such fair models
that adhere to differential privacy guarantees. This approach leads to applications other than algorithmic fairness. 
We consider two main suites of use cases:
the first is for optimizing convex performance measures of
the confusion matrix
(such as those derived from the $G$-mean and $H$-mean);
the second is
for satisfying statistical definitions of algorithmic fairness
(such as equalized odds, demographic parity, and the Gini index of
inequality).
\footnote{
In this paper, we only consider group-fair definitions but our framework can potentially
be extended to individual fairness notions~\citep{DworkHPRZ12}.
}
See Section~\ref{sec:reductions} for a few detailed
generic examples of how to apply our results.

We abstract the reductions approach to fair optimization as the
\textit{constrained group-objective optimization} problem where we aim to optimize an
objective that is a function of losses of individual groups,
subject to some constraints.
We present two differentially private algorithms: an $(\eps, 0)$ exponential sampling
algorithm and an $(\eps, \delta)$ algorithm
that uses an approximate linear optimizer to incrementally move toward the best decision.
The privacy and utility guarantees
of these empirical risk minimization algorithms are
presented.
Compared to a previous method for ensuring differential privacy subject to a relaxed form
of the equalized odds fairness constraint,
the $(\eps, \delta)$-differentially private algorithm provides asymptotically
better sample complexity guarantees.
The technique of using a bounded divergence linear optimizer oracle to achieve strong
guarantees of privacy/security and utility might be
applicable to other problems not considered in this paper.
Finally, we show
an algorithm-agnostic information-theoretic lower bound on the excess risk (or equivalently, the sample complexity) of any solution to the problem of
$(\eps, 0)$ or $(\eps, \delta)$ \textit{private constrained group-objective optimization}.

The focus of our work is on differentially private optimization via
empirical risk minimization. Generalization guarantees can be obtained by
taking a large enough sample of the population and of subgroups of the population. Another option is to consider the complexity (via VC dimension, for example) of the hypothesis class to be learned or the stability properties of the
differentially private algorithms since we know that stability implies generalization~\citep{McAllester99, McAllester03, DworkFHPRR15, FeldmanV19}. We do not
state any generalization guarantees in this paper but will motivate our work on empirical risk
minimization in the context of the eventual goal of machine learning -- generalization to unseen examples~\citep{Chervonenkis1971, Valiant84, Littlestone87, 
BlumerEHW89, EhrenfeuchtHKV89, LinialMR91, BousquetE00, KP2000, Vapnik00, Zhang06a}.
The reason for this viewpoint and discussion
is that differentially
private algorithms exhibit (provable) stability properties that
imply generalization.
Our framework, via our notational and definitional setup,
is amenable to analyses for stability properties.

\paragraph{Notation Setup and Example Usage}
Suppose we have a dataset $D$ of size $n$ consisting of i.i.d. draws
from an unknown distribution $\calD$. For example, we could have
$D= \{(\bx_i, a_i, y_i)\}_{i=1}^n$ that consists of non-sensitive
features $\bx_1, \ldots, \bx_n\in\calX$ of individuals, their
corresponding sensitive attribute $a_1, \ldots, a_n\in\calA$, and 
their assigned labels/values (depending on if the resulting task is for classification, regression, etc.) $y_1, \ldots, y_n\in \calY$. Let $\calC_P$ be a decision set 
(e.g., corresponding to a set of $P$-dimensional decision vectors, a set of classifiers that each can be represented by $P$ real numbers, a set of all possible real
coefficients of a polynomial threshold function, or a set of
all possible weights that can be used to represent a specific neural network architecture) where 
$\calC_P \subseteq \{c:(\calX\times\calA)\rightarrow\calY\}$ or $\calC_P \subseteq \{c:\calX\rightarrow\calY\}$.
We use $\calC_P$ to mean that the decisions in $\calC_P$
can be represented with at most $P$ real numbers whether $\calC_P$ consists of
classifiers or regression coefficient vectors. That is, the resulting parameter space lives in $\reals^P$.
For example, for a set of classifiers $\calC_P$ consisting of single-dimensional thresholds, we have $\VC(\calC_P) = P = 1$. As another example, for $P > 1$, let
$L_P = \{c_{\bw, b}~|~\bw\in\reals^{P-1}, b\in\reals\}$ where
$c_{\bw, b}(\bx) = \brackets{\bw, \bx} + b$. Then $L_P$ is parameterized by 
$\bw\in\reals^{P-1}, b\in\reals$. For regression, $\calC_P$ corresponds to the hypothesis class $L_P$ 
we wish to learn. For binary classification, $\calC_P$ would correspond to the class of functions that
result from the composition $\calC_P = \sign\circ~L_P$. For both the
regression and classification problems in the aforementioned example, the hypotheses are parameterized by a
$P$-dimensional vector.
To apply differential privacy, it is important to know the
hypothesis class we wish to learn since our statistical
and computational guarantees must necessarily
depend on properties of the
class we wish to learn~\citep{De12}.

A goal could be to obtain a decision from $\calC_P$ that can be used to classify an 
\textit{unseen} $(\bx_{n+1}, a_{n+1})\in(\calX\times\calA)$ or $\bx_{n+1}\in\calX$.
Typically, the approach is to find a (provably optimal)
predictor from the decision set $\calC_P$ via empirical risk minimization and show that this
predictor generalizes to unseen examples.
Suppose there are at most $K$ groups to which 
any example $(\bx, a, y)\sim\calD$ can belong to.
For any decision
$c\in \calC_P$, we define a loss function $\ell:\calC_P\times(\calX\times\calA\times\calY)^n\rightarrow[0, 1]^K$ to be 
$$\ell(c, D) = (\ell_1(c, D), \ldots, \ell_K(c, D)),$$ with $\ell_k(c, D)\in[0, 1]$ for each $k\in[K]$.
In addition, we also define the \textit{itemized} (per example) loss function
$\ell:\calC_P\times(\calX\times\calA\times\calY)\rightarrow[0, 1]^K$ so that the loss on the dataset $D$
will be the average of the itemized losses on each example for any decision $c\in\calC_P$
i.e., $\ell(c, D) = \frac{1}{n}\sum_{i=1}^n\ell(c, D_i)$.
We assume that $K \leq P$ and that in
most cases (as exemplified by our use cases) we have $K \ll P$. For example, although
a specific neural network architecture might have $P = 1000$ weight parameters,
$K = 20$
would be the maximum number of group statistics
(i.e., false positive rate for each racial or ethnic category)
computed on the examples fed to the neural network.
In our model, $K$ is not necessarily equal to $|\calA|$.
For example, this could happen when
$K$, the number of statistics computed for all groups,
is larger than $|\calA|$, the number of protected attributes.
For any decision $c\in\calC_P$, we let
$\ell_k(c, D)$ correspond to a context-specific or application-specific loss
for individuals that belong to group $k\in[K]$.
We assume that for any group $k\in[K]$, the loss $\ell_k(c, D)$
is an average loss of the form
$\ell_k(c, D) = \frac{1}{n}\sum_{i=1}^n\ell_k(c, D_i)$.
So $\ell_k(c, D)$ applies to the items in group $k\in[K]$ where $\ell_k(c, D_i)$ is the 
loss of $c$ on item $D_i=(\bx_i, a_i, y_i)$.
We denote the induced loss set on dataset $D$ as
$\ell(\calC_P, D) = \{\ell(c, D): c\in\calC_P\}\subseteq[0, 1]^K$.
The iterative linear optimization based private algorithm presented in this paper assumes that
$\ell(\calC_P, D)$ is compact and that we have access to an oracle that approximately optimizes
linear functions on $\ell(\calC_P, D)$. The exponential sampling algorithm assumes we have an 
approach to sampling from the decision set $\calC_P$ which we assume to be convex and
to live in at most $P$ dimensions.
If $\ell$ is convex, this algorithm is
guaranteed to be computationally efficient (i.e., runtime polynomial in $P, K, n$).
If  $\ell$ is
not convex, we cannot make such guarantees of computational efficiency but can still
make statistical efficiency guarantees.
In that case, we assume that the Vapnik–Chervonenkis (VC) dimension of $\calC_P$ ($=\VC(\calC_P)$) is finite and
that $\calC_P$ is a concept class ($\calY = \{0, 1\}$ or $\calY = \{-1, +1\}$).
For any dataset $D$, we can write the dataset as
$D = (D_{\calX\calY}, D_\calA) = (D_\calX, D_\calY, D_\calA)$ where
$D_{\calX\calY}\in(\calX\times\calY)^n, D_\calX\in\calX^n, D_\calY\in\calY^n$
represents the insensitive attributes and
$D_\calA\in\calA^n$ represents the sensitive attributes. 
The main goal of our work is to
guarantee differential privacy with respect to the sensitive attribute.
But if $\calC_P$ is of finite size or $\ell$ is convex,
we can guarantee the privacy of the insensitive attributes
as well.

\paragraph{Definitions}

We now summarize the main definitions that we employ.

\begin{definition}[\cite{ChaudhuriH11, Jagielski18}]
An algorithm $\calM:(\calX\times\calA\times\calY)^n\rightarrow\calR$ 
is \textbf{$(\eps, \delta)$-differentially private in the sensitive attributes}
if for all $D_{\calX\calY}\in(\calX, \calY)^n$ and for all neighboring 
$D_\calA\sim D_{\calA}'\in\calA^n$ and all
$T\subseteq\calR$, we have
$$
\pr[\calM(D_{\calX\calY}, D_\calA)\in T]\leq e^\eps\cdot\pr[\calM(D_{\calX\calY}, D_{\calA}')\in T] + \delta.
$$

The probability is over the coin flips of the algorithm $\calM$.
\end{definition}

Now, let $\calC_P(D_\calX)$
be the set of all possible labellings induced on $D_\calX$ by $\calC_P$. i.e.,
$\calC_P(D_\calX) = \{(c(\bx_1), \ldots, c(\bx_n))~|~c\in\calC_P\}$.
Then by Sauer's Lemma, $|\calC_P(D_\calX)|\leq O(n^{\VC(\calC_P)})$~\citep{mlbook}.
\footnote{Sometimes known as the Sauer–Shelah Lemma.}
In this paper, we
will use $\calC_P(D_\calX)$ as the range of the exponential mechanism so
that even if $\calC_P$ is infinite,
assuming that its VC dimension is finite, we can obtain empirical risk
bounds in terms of $\VC(\calC_P)$. We require that the sensitive attribute be excluded from the
domain of functions in $\calC_P$.\footnote{An assumption
also made in~\citep{Jagielski18}.}

For any $c\in\calC_P$, the true \textit{population} loss on
group $k$ is $\ell_k(c) = \E_{D\sim\calD^n}[\ell_k(c, D)]$ and the true \textit{population} loss for all
groups is
$\ell(c) = (\ell_1(c), \ldots, \ell_K(c))$. 
The goal of
constrained group-objective optimization is to minimize the error function
$f(\ell(c))$ subject to the constraint $g(\ell(c)) \leq 0$ where $f, g:[0, 1]^K\rightarrow\reals$ are
context-specific or application-specific functions specified by the data curator.

Our differential privacy guarantees will be with respect to the centralized
model where a central and trusted curator holds the data (as opposed to the
local or federated model for differentially private computation).
We now define \textit{constrained group-objective optimization} and
\textit{private constrained group-objective optimization}.

\begin{definition}[Constrained Group-Objective Optimization:$\quad\cgoo(\calC_P, n, K, f, g, \ell, D, \alpha)$]

Let $f:[0, 1]^K\rightarrow\reals$ be a function we wish to minimize subject to a
constraint function $g:[0, 1]^K\rightarrow\reals$.
Specifically, for any excess risk parameter $\alpha > 0$, decision set $\calC_P$,
and any dataset $D$ of size $n$,
we wish to obtain a decision $\hat{c}\in \calC_P$ such that 
\begin{enumerate}
\item 
$f(\ell(\hat{c}, D)) \leq \min_{c\in\calC_P: g(\ell(c, D))\leq 0} f(\ell(c, D)) + \alpha$,
\item 
$g(\ell(\hat{c}, D))\leq \alpha.$
\end{enumerate}
\label{eq:go}
Any deterministic or randomized procedure that takes input $D$
and returns a decision $\hat{c}\in\calC_P$ that satisfies
the two conditions above is a
\textit{constrained group-objective optimization} algorithm that solves the
problem specified by $\cgoo(\calC_P, n, K, f, g, \ell, D, \alpha)$.

\label{def:goo}
\end{definition}

Definition~\ref{def:goo} is implicit in the work of~\cite{AlabiIK18}.
This optimization problem differs from ordinary constrained optimization since we are optimizing
with respect to functions of group statistics (e.g., true positives, false positives for examples in
a group) instead of individual examples. In addition, there are two functions: $f$ which is used to control the error as a function of the group statistics and $g$ which can be used to control the
deviations of the group statistics from one another. In later sections, we show specific formulations of optimization problems in terms of Definition~\ref{def:goo}.
A \textit{private constrained group-objective optimization} problem is a constrained group-objective optimization problem where the
resulting decision $\hat{c}\in\calC_P$ is optimized in a differentially
private manner. i.e., satisfying $(\eps, 0)$ or $(\eps, \delta)$-differential privacy or some other notion of data privacy.

We note that Definition~\ref{def:goo} is a special case of the more general
multi-objective optimization problem, where we usually have multiple,
sometimes an exponential number of,
optimal solutions (forming a pareto-optimal set). 
% The
% \textit{constrained} multi-objective optimization problem
% is a more general problem
% that can applied to other
% statistical tasks such as
% hypothesis testing.
\footnote{
See~\citep{Marler2004} for a survey on multi-objective optimization.
}

In this paper, the algorithms we present assume that the functions $f, g$ are convex and 
$O(1)$-Lipschitz.
\footnote{
In the remainder of this paper, we use Lipschitz to mean $O(1)$-Lipschitz in the output
parameter space $\calC_P$.
}
In addition, the Frank-Wolfe based algorithm (in the appendix) assumes that the
gradients of $f, g$ are
Lipschitz.\footnote{
Sometimes referred to as the $\beta$-smooth property.
}
Our main novel contribution is an $(\eps, \delta)$-differentially private algorithm
for solving the constrained group-objective optimization problem and accompanying
techniques in the quest for data privacy. This algorithm essentially implements a differentially private
linear optimization oracle ($\lopt_{\eps, \delta}$ satisfying $(\eps, \delta)$-differential privacy) to solve linear subproblems
approximately in each timestep. The non-private version of this oracle is $\lopt$ which, although not equivalent to the
statistical query model, can be used to simulate such queries
~\citep{Kearns98}.
The specifications of $\lopt$ and $\lopt_{\eps, \delta}$ are in Definitions~\ref{def:lopt} and~\ref{def:loptp}.
In Section~\ref{sec:divergence}, we introduce a more general class of bounded 
divergence linear optimizers that includes both $\lopt$ and $\lopt_{\eps, \delta}$.
\footnote{
$\eps$-differential privacy can be cast as a constraint on the max divergence between two
random variables. Similarly, R\'enyi differential privacy can be cast as a constraint
on the R\'enyi divergence~\citep{Mironov17}.
} But for clarity of exposition, our results will be cast in terms of
$\lopt$, $\lopt_{\eps, 0}$, or $\lopt_{\eps, \delta}$.

\begin{definition}[$\lopt$]
$\lopt$ is an oracle for solving linear subproblems approximately.
Let $\calW\subseteq\reals^K$
(or $\calW\subseteq\reals^K_{\geq 0}$) be a set
of weight vectors. Then for any weight vector $\bw\in\calW$, if $\hat{c} = \lopt(\calC_P, \ell, \bw, D, \tau)$, then
$$
\bw\cdot\ell(\hat{c}, D) \leq \min_{c\in\calC_P} \bw\cdot \ell(c, D) + \tau\norm{\bw},
$$
where $\calC_P$ is the decision set, $D$ is the dataset of size $n$, and $\tau$ is the tolerance parameter of the
oracle.
\label{def:lopt}
\end{definition}

In Definition~\ref{def:lopt}, we also consider restrictions to non-negative vectors since as noted in
~\citep{KakadeKL09, AlabiIK18}, many natural approximation algorithms can only handle non-negative weight
vectors.

\begin{definition}[$\lopt_{\eps, \delta}$, $\lopt^\theta_{\eps, \delta}$]
$\lopt_{\eps, \delta}$ is an $(\eps, \delta)$-differentially private oracle for solving linear subproblems
approximately. Let $\calW\subseteq\reals^K$
(or $\calW\subseteq\reals^K_{\geq 0}$) be a set of weight vectors. Then 
for any weight vector $\bw\in\calW$:

\begin{enumerate}
    \item If $\tilde{c} = \lopt_{\eps, \delta}(\calC_P, \ell, \bw, D, \tau)$, then $\bw\cdot \ell(\tilde{c}, D) \leq \min_{c\in\calC_P} \bw\cdot \ell(c, D) + \tau\norm{\bw}$,
    \item $\forall c\in\calC_P$,
    $\pr[ \lopt_{\eps, \delta}(\calC_P, \ell, \bw, D, \tau) = c] \leq e^{\eps}\cdot\pr[ \lopt_{\eps, \delta}(\calC_P, \ell, \bw, D', \tau) = c] + \delta$,
\end{enumerate}
for any neighboring datasets $D, D'$ of size $n$ where 
$\tau$ is the tolerance parameter of the oracle. The probability is
over the coin flips of the oracle.

When item 1 holds with probability $\geq 1-\theta$, we term this oracle
$\lopt^\theta_{\eps, \delta}$. We sometimes use $\lopt_{\eps, \delta}$ and $\lopt^\theta_{\eps, \delta}$
interchangeably when it is clear from context that the linear subproblems are solved with high probability.

\label{def:loptp}
\end{definition}

We provide a \textit{generic} implementation of $\lopt_{\eps, \delta}$ 
based on the exponential
mechanism. In the case where $\ell$ is convex, we use the computationally efficient 
convex exponential sampling and stochastic gradient descent techniques
of~\cite{BassilyST14} for pure and approximate differential privacy, respectively.
When $\ell$ is not convex, we use the generic
exponential mechanism to sample from $\calC_P$.

%This process is illustrated in Figure~\ref{fig:linopt} where
%$\{\tilde{c}_1, \ldots, \tilde{c}_T\}$ are the output of the
%differentially private linear optimization oracle.

\begin{figure}
  \begin{subfigure}[b]{0.4\textwidth}
    \includegraphics[width=\textwidth]{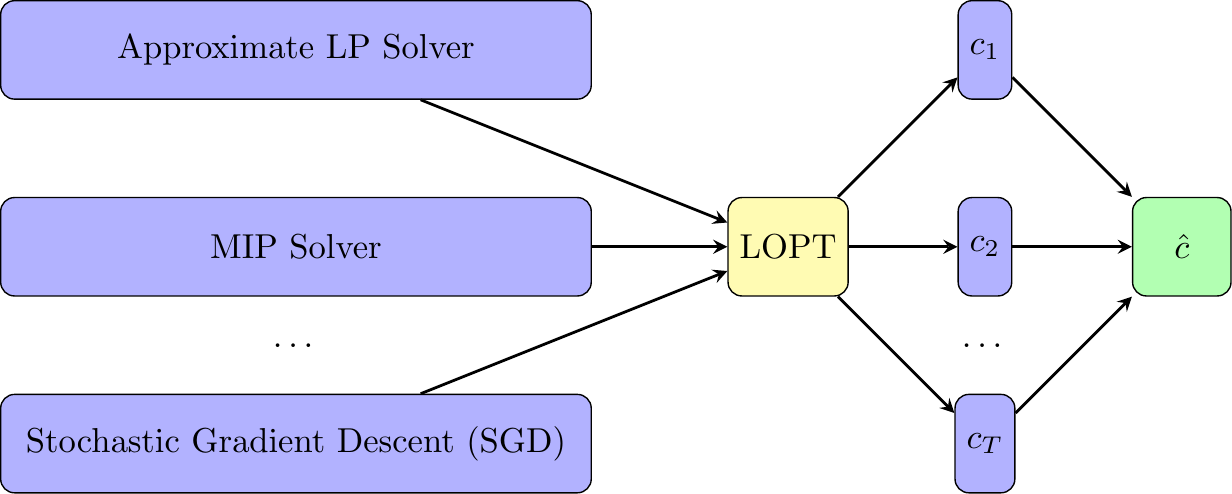}\label{fig:lopt1}
  \end{subfigure}
  \qquad
  \begin{subfigure}[b]{0.4\textwidth}
    \includegraphics[width=\textwidth]{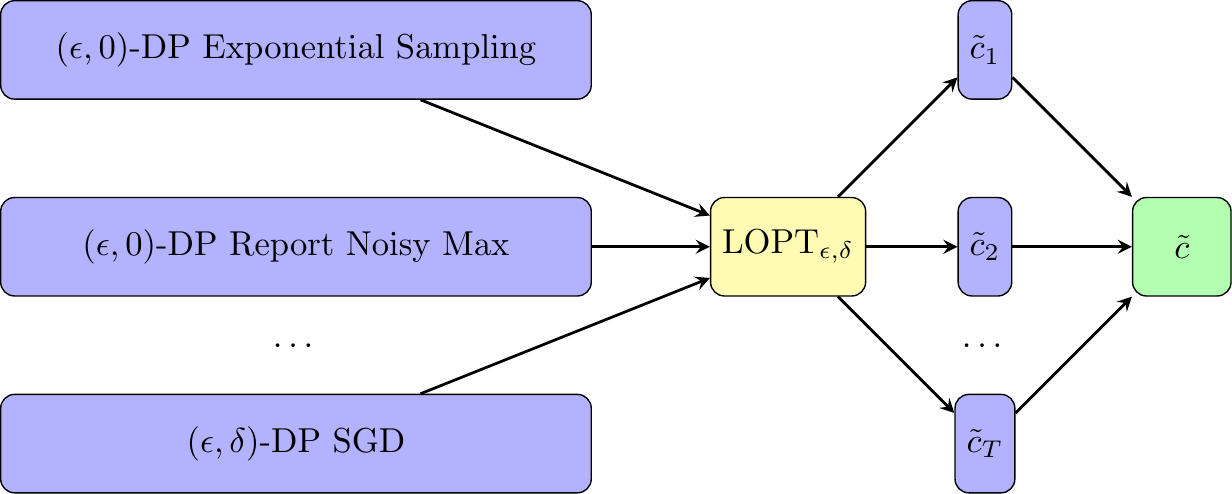}\label{fig:lopt2}
  \end{subfigure}
  \centering
  \caption{Approximate Linear Optimizer Oracles $\lopt$ and $\lopt_{\eps, \delta}$}
  \label{fig:lopt}
\end{figure}

% \begin{figure}
%   \begin{subfigure}[b]{0.8\textwidth}
%     \includegraphics[width=\textwidth]{loptimage}
%     \caption{Approximate Linear Optimizer Oracles $\lopt$ and $\lopt_{\eps, \delta}$.}\label{fig:linopt}
%   \end{subfigure}
%   \centering
% \end{figure}

As Figure~\ref{fig:lopt} illustrates, $\lopt_{\eps, \delta}$ could be implemented via a number
of approaches depending on the specification of the loss function $\ell$. 
For example, in the case where $\ell$ is convex and $\calC_P$ lies in 
the $\ell_2$ unit ball, we can
provide an implementation of a computationally efficient
$\lopt_{\eps, \delta}$ based on the
private stochastic gradient descent algorithm of~\cite{BassilyST14} and use this oracle to, for example, solve weighted
least squares regression~\citep{Sheffet19}.
If $\ell$ is not convex, we could use
more generic implementations of the exponential mechanism.
Without privacy considerations, $\lopt$ can be implemented via the use
of an approximate LP or MIP solver or via a vanilla stochastic gradient
descent~\citep{Bubeck15}. The implementation of $\lopt, \lopt_{\eps, \delta}$ will depend on the decision set $\calC_P$ and its accompanying
loss function $\ell$. In~\citep{AlabiIK18}, the existence of $\lopt$
is assumed and used to solve the $\cgoo$ problem non-privately.
We shall also follow a similar route: assume the existence of $\lopt_{\eps, \delta}$ but, in addition, we will
provide a generic implementation of the private oracle so that we may obtain utility guarantees.

In Section~\ref{sec:reductions}, we show applications of our work to two main suites of uses cases.
The first is for optimizing convex performance measures of
the confusion matrix
(such as those derived from the $G$-mean and $H$-mean);
the second is
for satisfying statistical definitions of algorithmic fairness
(such as equalized odds, demographic parity, and Gini index of
inequality).

\subsection{Summary of Results}

We proceed to state and interpret 
informal versions of some of our main
theorems and corollaries. Through the lens of information-theoretic reductions,
we show the following:

\begin{enumerate}
\item \textbf{Algorithms}: 
Present two generic differentially private algorithms to solve this
problem -- an $(\eps, 0)$ exponential sampling
algorithm and an $(\eps, \delta)$ algorithm
that uses an approximate linear optimizer to 
incrementally move toward the best decision.
\item \textbf{Improvements on Sample Complexity Upper Bound}: Compared to a previous method for ensuring differential privacy subject to a relaxed form
of the equalized odds fairness constraint,
the $(\eps, \delta)$-differentially private algorithm we present provides asymptotically
better sample complexity guarantees, resulting in an exponential
improvement in certain parameter regimes.
\item \textbf{First Polynomial-Time Algorithms}:
Give the first polynomial-time algorithms to solve
the problem with $(\eps, 0)$ or $(\eps, \delta)$ 
differential privacy guarantees
when defined on a convex decision set (for example,
the $\ell_P$ unit ball) with 
convex constraints and losses.
\item\textbf{Bounded Divergence Linear Optimizer Primitive}:
Introduce a class of bounded divergence linear optimizers
and specialize to pure and approximate differential privacy.
The technique of using bounded divergence linear optimizers to
simultaneously achieve privacy/security (and/or other
constraints) and utility might be
applicable to other problems not considered in this paper.
\item\textbf{Lower Bounds}: Finally, we show
an algorithm-agnostic information-theoretic lower bound on the excess risk (or equivalently, the sample complexity) 
of any solution to the problem of
$(\eps, 0)$ or $(\eps, \delta)$ \textit{differentially private constrained group-objective optimization}.
\end{enumerate}

Unlike in~\citep{BassilyST14}, the sample complexity
upper bounds for convex $\ell, f, g$ scale as $O(1/\alpha^2)$ rather than $O(1/\alpha)$
because of the way we apply Lagrangian Duality. We essentially
compose functions $f, g$ into $h = f + \max(0, g)\cdot O(\sqrt{K}/\alpha)$. As a result, to minimize
$h$ to within $O(\alpha)$, we need to minimize $g$ to within $O(\alpha^2/\sqrt{K})$.
As a consequence, our results are tight (with respect to
the accuracy parameter $\alpha$) amongst all strategies that
compose functions.
In addition, when using differential privacy,
the dependence on $P$ or $\VC(\calC_P)$ is necessary because our methods
sample from $P$-dimensional decision sets (while $K$ is the \textit{ambient dimension} due to
function composition).

\begin{thmi}

Suppose we are given 
any constrained group-objective optimization problem (Definition~\ref{def:goo}) where $f$ and $g$
are convex, Lipschitz functions and we wish to
obtain a decision $\tilde{c}\in\calC_P$ in a differentially private manner.

If $\ell$ is a convex function and $f, g$ are non-decreasing,
then let $n_0 = O(\frac{K\cdot P}{\eps\alpha^2})$. If not, let
$n_0 = \tilde{O}(\frac{K\cdot \VC(\calC_P)}{\eps\alpha^2})$. Then there exists
$n_0$ such that for all $n\geq n_0$ and
privacy parameter $\eps > 0$ there is an $\eps$-differentially private
algorithm that, with probability at least 9/10,
returns a decision $\tilde{c}\in\calC_P$ that
solves the $\cgoo(\calC_P, n, K, f, g, \ell, D, \alpha)$ problem.
The algorithm is guaranteed to be computationally efficient in the case where $\ell$ is convex
and $f, g$ are non-decreasing.

\label{thmi:expc}
\end{thmi}

Theorem~\ref{thmi:expc} (more informal version of Theorem~\ref{thm:expc}) shows that
we can use the exponential mechanism to
solve the $\cgoo(\calC_P, n, K, f, g, \ell, D, \alpha)$ problem although an explicit mechanism
to sample from the set $\calC_P$ is not provided. This method provides a pure
$\eps$-differentially private algorithm to solve the problem. The problem is easier and guaranteed to
be computationally efficient when $\ell$ is convex and $f, g$ are non-decreasing because it results
in an efficient construction of a $\lopt_{\eps, \delta}$ oracle with polynomial runtime in $P, K, n$.
If not, we use the generic exponential
mechanism and do not provide any computational efficiency guarantees.

\begin{thmi}
Suppose we are given any constrained group-objective
optimization problem (Definition~\ref{def:goo}) where
$f, g$ are convex, Lipschitz functions.
Then for any $\alpha > 0$, there exists an algorithm that 
after $T = O(\frac{K^4}{\alpha^2})$ calls to $\lopt$ will, with
probability at least 9/10, return a decision $\hat{c}\in\calC_P$ that 
solves the $\cgoo(\calC_P, n, K, f, g, \ell, D, \alpha)$ problem.
\label{thmi:aik}
\end{thmi}

\begin{thmi}

Suppose we are given 
any constrained group-objective optimization problem (Definition~\ref{def:goo}) where $f$ and $g$
are convex, Lipschitz functions and we wish to
obtain a decision $\tilde{c}\in\calC_P$ in a differentially private manner.

Then for any privacy parameters $\eps, \delta\in(0, 1]$, there exists an $(\eps, \delta)$-differentially private linear optimization based algorithm for which there is a setting of
$\eps', \delta'\in(0, 1]$ such that
after $T = O\left(\frac{K^4}{\alpha^2}\right)$ calls to $\lopt_{\eps', \delta'}$, with probability at least 9/10, the algorithm returns a decision $\tilde{c}\in\calC_P$ that 
solves the $\cgoo(\calC_P, n, K, f, g, \ell, D, \alpha)$ problem.
\label{thmi:paik}
\end{thmi}

Theorem~\ref{thmi:aik} (more informal version of Theorem~\ref{thm:aik})
shows that for any accuracy parameter $\alpha > 0$, we can,
after $T = \poly(K, 1/\alpha)$ 
calls to a linear optimization oracle,
solve the constrained group-objective
optimization problem to within $\alpha$, with high probability,
provided that $f, g$ are convex, Lipschitz functions. For this theorem, we require 
access to $\lopt$ in each iteration.
We note that~\cite{AlabiIK18} also achieved this
theorem but we reprove it here more generally (so it is more
amenable to use in our later proofs involving the additional constraint of data privacy).

Theorem~\ref{thmi:paik} (more informal version of Theorem~\ref{thm:paik}), with
privacy guarantees, still relies on calls to a linear optimization oracle albeit its private
counterpart $\lopt_{\eps, \delta}$.
One way to interpret Theorems~\ref{thmi:aik} and~\ref{thmi:paik} is that if the non-private oracle
is replaced with the private oracle, we can still solve the
$\cgoo(\calC_P, n, K, f, g, \ell, D, \alpha)$ problem via the
use of advanced composition~\citep{DworkRV10}. What remains is to show the existence and
construction of the oracle $\lopt_{\eps, \delta}$ and provide utility guarantees for certain
constructions.

\begin{thmi}
For any privacy parameter $\eps > 0$, there is an implementation of $\lopt_{\eps, 0}$ based
on the exponential mechanism.

For any $\tau > 0, \theta\in(0, 1]$, if $\ell$ is convex and $\calW$ is restricted to non-negative vectors, set
$n_0 = \tilde{O}(\frac{\sqrt{K}}{\eps\tau}(P + \log\frac{1}{\theta}))$ and if not
set $n_0 = \tilde{O}(\frac{\sqrt{K}}{\eps\tau}(\VC(\calC_P) + \log\frac{1}{\theta}))$.
Then there exists $n_0$ such that for all $n\geq n_0$, 
we can solve the $\lopt^\theta_{\eps, 0}(\calC_P, \ell, \bw, D, \tau)$ problem. 

\label{thmi:loptexp}
\end{thmi}

Theorem~\ref{thmi:loptexp} (more informal version of Theorem~\ref{thm:loptexp}) shows a generic
construction of the $\lopt^\theta_{\eps, 0}$ oracle. Armed with this, we provide Corollary
~\ref{cori:loptexp}.

\begin{cori}
Suppose we are given 
any constrained group-objective optimization problem (Definition~\ref{def:goo}) where $f$ and $g$
are convex, Lipschitz functions and we wish to
obtain a decision $\tilde{c}\in\calC_P$ in a differentially private manner.

For any privacy parameters $\eps, \delta\in(0, 1]$, there exists an $(\eps, \delta)$-differentially
private linear optimization based algorithm that solves the
$\cgoo(\calC_P, n, K, f, g, \ell, D, \alpha)$ problem.
If $\ell$ is convex and $f, g$ are non-decreasing, set $n_0 = \tilde{O}\left(\frac{K^4P}{\eps\alpha^3}\right)$. If not,
set $n_0 = \tilde{O}\left(\frac{K^4\cdot\VC(\calC_P)}{\eps\alpha^3}\right)$.
Then there exists $n_0$ such that for all $n\geq n_0$, 
with probability at least 9/10, the algorithm will
return a decision $\tilde{c}\in\calC_P$ that 
solves the $\cgoo(\calC_P, n, K, f, g, \ell, D, \alpha)$ problem.
The algorithm uses an $\lopt_{\eps, 0}$ oracle implemented via the exponential mechanism.

\label{cori:loptexp}
\end{cori}

In some ways, the statistical and computational complexity we obtain in 
Corollary~\ref{cori:loptexp} is worst-case since
we implement the $\lopt_{\eps, 0}$ oracle via the exponential mechanism. For specific problems
(e.g., ordinary least squares on the $\ell_2$ ball), there are more computationally 
efficient implementations of the $\lopt_{\eps, \delta}$ oracle as we shall see in 
Section~\ref{sec:applications}.
We show asymptotic convergence guarantees so that the excess risk goes to 0 as
$n\rightarrow\infty$.
For ease of exposition,
the sample complexity guarantees of Theorem~\ref{thmi:paik},~\ref{thmi:loptexp},
and~\ref{cori:loptexp} are in terms of $\tilde{O}(\cdot)$ which
hides polylogarithmic factors (including the polylogarithmic dependence
on $\frac{1}{\delta}$). We ignore these
polylogarithmic factors to obtain cleaner statements.

\begin{thmi}

Suppose we are given 
a constrained group-objective optimization problem (Definition~\ref{def:goo}) where $f$ and $g$
are convex functions.
Let $\eps > 0$. Then for every $\eps$-differentially private algorithm,
there exists a dataset $D = \{\bx_1, \ldots, \bx_n\}$ drawn from the $\ell_2$ unit ball
such that, with probability at least 1/2,
in order to solve the
problem $\cgoo(\calC_P, n, K, f, g, \ell, D, \alpha)$
we need sample size $n \geq \Omega\left(\frac{K}{\eps \alpha}\right)$.
\label{thmi:lowerpure}

\end{thmi}

\begin{thmi}

Suppose we are given 
a constrained group-objective optimization problem (Definition~\ref{def:goo}) 
where $f$ and $g$
are convex functions.
Let $\eps > 0, \delta = o(\frac{1}{n})$. 
Then for every $(\eps, \delta)$-differentially private algorithm,
there exists a dataset $D = \{\bx_1, \ldots, \bx_n\}$ drawn from the $\ell_2$ unit ball such that, with probability at least 1/3,
in order to solve
the problem $\cgoo(\calC_P, n, K, f, g, \ell, D, \alpha)$
we need sample size $n \geq \Omega\left(\frac{\sqrt{K}}{\eps \alpha}\right)$.
\label{thmi:lowerapprox}

\end{thmi}

Theorems~\ref{thmi:lowerpure} and~\ref{thmi:lowerapprox} (more informal versions of
Theorems~\ref{thm:lowerpure} and~\ref{thm:lowerapprox})
show lower bounds on the sample complexity for solving the constrained
group-objective optimization problem in a differentially private manner.
The lower bounds for achieving (pure) $\eps$ and (approximate) $(\eps, \delta)$-differential
privacy to
solve the $\cgoo(\calC_P, n, K, f, g, \ell, D, \alpha)$ problem differs from the upper bounds
(from Theorems~\ref{thmi:expc} and~\ref{thmi:paik}).
Note that this gap is a direct result of the way we minimize
$f$ subject to the constraint of $g$ by jointly minimizing a composition of these functions.
As a consequence, our results are optimal amongst all such
strategies that jointly minimize a composition of
these functions.
%\textit{An open problem is to close these gaps in certain parameter regimes and for
%specific formulations of the $\cgoo$ problem 
%(e.g., for linear regression)}.

We note that~\citep{Jagielski18} considered the problem of differentially private fair
learning in which they present a reductions approach to fair learning but the oracle-based
algorithms they provide are specific modifications of those provided by~\cite{AgarwalBD0W18}.
The algorithm is an exponentiated gradient algorithm for fair
classification that uses a cost-sensitive classification
oracle solver in each iteration, which is only applied 
to the equalized odds definition.
We show an approach that applies to more than one definition. Moreover, the algorithms in our
paper results in asymptotically better
sample complexity guarantees than previous work although under different underlying oracle assumptions and for a smoothed
version of the equalized odds definition.

We hope that the generality of our approaches and techniques here will lead to applications in myriad
domains.

\subsection{Techniques}

% Now we proceed to summarize the major techniques used to show the results in this paper.

We introduce a \textbf{class of bounded divergence linear optimizers} 
(see Section~\ref{sec:divergence}) to simultaneously
achieve \textit{strong privacy} guarantees and solve constrained group-objective optimization
problems (Definition~\ref{def:goo}).
These linear optimizers can be used to solve general
multi-objective problems with one or more
divergence constraints.
For simplicity and clarity of exposition, 
we specialize this linear optimizer to $(\eps, 0)$ and
$(\eps, \delta)$-differential privacy.

Based on the exponential mechanism, we provide an
$(\eps, 0)$-differentially private algorithm to solve
the $\cgoo$ problem. In the case where $\ell$ is convex, we use the computationally efficient
sampling technique of~\cite{BassilyST14} to sample from $\calC_P$. Our $\cgoo$
algorithms rely on the simple observation that if we are given two functions
$f, g:[0, 1]^K\rightarrow\reals$
and aim to minimize the function $f$ subject to the constraint $g$ we could
minimize them jointly via a ``new'' function.
Specifically, we define the function $h:[0, 1]^K\rightarrow\reals$ where
$h(\bx) = f(\bx) + G\cdot\max(0, g(\bx))$ for all $\bx\in[0, 1]^K$ for some setting of
$G > 0$. Then we could optimize $h$ with privacy in mind. 
% The caveat is
% that the noise we add to the gradients of $f, g$ to ensure privacy
% in the iterative algorithm will depend on the setting of $G > 0$ (which itself could be a function
% of the privacy parameters $\eps, \delta$ and of $K, n$).
% In fact, this is what we observe. 
The $(\eps, \delta)$-differentially private algorithm, in each iteration, relies on calls to
the private oracle $\lopt_{\eps, \delta}$.
And to optimize both $f$ and $g$ to within $\alpha$ we
can set $G = O(\frac{\sqrt{K}}{\alpha})$ (for large enough sample size $n$).
Note that this differentially private algorithm is a first-order
iterative optimization algorithm that relies on access to the gradient oracles
$\nabla f, \nabla g$. Optimization with respect to these gradient oracles
is done in a private manner while weighting $\nabla g$ by a multiplicative factor of $G$.
The strategy of differentially private optimization of a function subject to
one or more constraints can be applied to other situations.
The weighting of gradients non-privately to solve the $\cgoo(\calC_P, n, K, f, g, \ell, D, \alpha)$
problem was done by~\cite{AlabiIK18} but without privacy considerations. 
The technique of using the private oracle $\lopt_{\eps, \delta}$ (or an alternative from the
class of bounded divergence linear optimizers) to solve an overall
convex (or non-convex) optimization problem, with or without other constraints, 
might be applicable to other scenarios.

It is known that differentially private iterative algorithms use the crucial property of 
\textit{(advanced) composition of differential privacy}~\citep{DworkRV10, DworkR14}
which come in a variety of forms. The iterative algorithms we provide exploit this property.
The lower bounds we provide for empirical risk minimization are modified versions of the ones
provided by~\cite{BassilyST14}.

\subsection{Applications}
\label{sec:applications}

In computer science, showing that one problem can be reduced 
to another is a staple of proofs, to obtain lower or upper
bounds on complexity measures.~\cite{Karp1972},
famously, showed that there is a many-to-one reduction
from the Boolean Satisfiability problem to 21
graph-theoretical and combinatorial problems. Karp showed that,
as a consequence, these 21 problems are NP-complete.

This approach of using reductions can also be applied to
problems that are not necessarily combinatorial in nature.
We, essentially, reduce a few problems to 
solving bounded-divergence linear subproblems
(see Section~\ref{sec:divergence}).

\begin{table}[]
\begin{tabular}{| l | l |}\hline
Method       &   Desired Guarantees \\\hline
$\ell_0$  & Sparsity, Basis Selection \\\hline
  $\ell_1$ &  Robustness to Outliers, Transfer of Compressed Sensing Techniques       \\\hline
   $\ell_2$         &  Standard (and Faster) Convergence Rates, Allows (Provable) use of SGD Based Methods       \\\hline

\end{tabular}
\caption{Comparing properties of $\ell_0, \ell_1, \ell_2$ minimization~\citep{WipfR04,Donoho06}
}
\label{tab:lnorm}
\end{table}

After using Lagrangian duality to create objectives
that
satisfy one or more sub-criteria, we can rely
on calls to standard optimizers. The choice of the underlying optimizer
depends on the desired properties we want to satisfy.
See Table~\ref{tab:lnorm} for examples.
The use of
robust differentially private estimators
(e.g., $\ell_1$ objectives) 
could provide better utility~\citep{DworkL09}.

In Section~\ref{sec:reductions}, we expand on the breadth of our applications from
optimizing convex measures of the confusion matrix to satisfying certain definitions
from the algorithmic fairness literature.
The linear optimization based algorithm we provide can only be applied to convex, Lipschitz functions $f, g$. 
In the appendix, we provide a Frank-Wolfe based algorithm
that also requires Lipschitz gradients.
However, we note that even if $f, g$ are not convex or smooth there exist
surrogate convex functions and standard smoothing techniques that can be applied
(e.g., see 
\textit{Moreau-Yosida regularization}~\citep{Nesterov05}
and correspondences between $f$-divergences and
surrogate loss functions~\citep{BJM03, NJW05, NWJ09}).
First, we show how to apply our work to the problem of weighted least squares
regression. Then, we show how to satisfy a relaxed form of
the Equalized Odds fairness definition
while returning accurate classifiers on training data.
Finally, we apply our results to the problem of hypothesis
testing.

\subsubsection{Case Study: Reduction to Ordinary Least Squares with Subgroup Weights}

We have stated approaches to solving the $\cgoo(\calC_P, n, K, f, g, \ell, D, \alpha)$
problem using a generic construction of $\lopt_{\eps, \delta}$ oracles via the exponential
mechanism which is not guaranteed to be computationally efficient. Now we proceed to show that
for the specific problem of ordinary least squares (which admits a convex loss), we get an
efficient $\lopt_{\eps, \delta}$. 

Let $B^P_2$ represent the unit ball in $P$ dimensions
i.e., $B^P_2 = \{\bx\in\reals^P: \norm{\bx}_2 = 1\}$.
Suppose we are given $n > 1$ input points $\bx_1, \ldots, \bx_n$ from $B^P_2$ each belonging
to one of $[K]$ groups encoded through a function $d:B^P_2\rightarrow[K]$ (i.e., 
private function known to the data curator). Each
data point $\bx_i$ has a corresponding output point $y_i\in[0, 1]$.

Given a weight vector $\bw\in[0, 1]^K$, 
the goal is to output a $c\in B^P_2$ such that the empirical average squared loss
$$\frac{1}{n}\sum_{k\in[K]}\sum_{i=1}^n w_k\cdot\ell_k(c, \bx_i, y_i) = \frac{1}{n}\sum_{k\in[K]}\sum_{i=1}^n w_k\cdot\ind[d(\bx_i) = k]\cdot(\brackets{c, \bx_i} - y_i)^2$$ 
is minimized. To proceed, a naive method is to translate each $\bx_i\in B^P_2$ into 
$\bar{\bx}_i\in B^{PK}_2$ where $\bx_i$ will occupy coordinates $(k-1)P+1, \ldots, k\cdot P$ of
$\bar{\bx}_i$ if $\bx_i$ belongs to group $k\in[K]$. The remaining coordinates will be set to 0.
We now routinely apply the private stochastic gradient algorithm
\footnote{Since the data points lie in the $\ell_2$ ball, a projection operator is
$\Pi(\bx) = \bx/\norm{\bx}$.} of
~\cite{BassilyST14} to solve the $\lopt_{\eps, \delta}$ problem in time polynomial
in $K, P$ and
thus solve the $\cgoo(\calC_P, n, K, f, g, \ell, D, \alpha)$ problem in time polynomial
in $K, P$. But note that this results in a regression coefficient 
vector in
$B^{PK}_2$ instead of $B^{P}_2$. A similar idea can be used to obtain
coefficient vectors in $B^{P}_2$ instead.

Note that since the loss function 
$\ell_k(c, \bx_i, y_i) = \ind[d(\bx_i) = k]\cdot(\brackets{c, \bx_i} - y_i)^2$ has 
Lipschitz constant at most 2, we get the following corollary by, for example, using
the $(\eps, \delta)$-differentially private stochastic gradient descent algorithm in
~\citep{BassilyST14} to implement $\lopt_{\eps, \delta}$.

\begin{cori}
There exists a polynomial-time $(\eps, \delta)$-differentially private algorithm
that, with probability at least 9/10, returns a decision $\tilde{c}\in\calC_p$ that solves
the $\cgoo(\calC_P, n, K, f, g, \ell, D, \alpha)$ problem when applied to solve
ordinary least squares.
\end{cori}

We have discussed how to obtain efficient oracles for $\lopt, \lopt_{\eps, \delta}$.
But how do different implementations of $\lopt, \lopt_{\eps, \delta}$ perform
(relative to one another) for the weighted least squares regression problem? And how can we use these
$\lopt, \lopt_{\eps, \delta}$ oracles to solve the $\cgoo$ problem?
~\cite{AgarwalDW19} study fair regression via reduction-based
algorithms, an approach that can be instantiated in the
$\cgoo$ framework. The weighted linear regression problem can be solved
differentially privately as shown in~\citep{Sheffet19}.
We defer the study of this problem in detail (with specific applications to
regression) to future work.

\subsubsection{Case Study: Satisfying Equalized Odds}

We proceed to state an informal corollary that illustrates how to use our theorems to
satisfy certain definitions from the algorithmic fairness literature.
The corollary serves to compare the method in this paper to that of 
~\cite{Jagielski18} in satisfying
$\alpha$-equalized odds (see Definition~\ref{def:alphaeo}) which is the only
fairness definition they consider when satisfying both
privacy and fairness.
\footnote{
Although their methods could probably be applied to other
statistical fairness definitions as well.
}
In contrast, the algorithms in this paper can
be applied to more than one kind of fairness definition 
(although under different oracle assumptions).
Also, our linear optimization based algorithm
requires not just convexity of $f, g$ but also that $f, g$ are Lipschitz
so we define a smoothed version of the Equalized Odds
definition.

\begin{definition}[$\alpha$-Equalized Odds~\citep{Jagielski18}]
Let $X, A, Y$ be random variables representing the non-sensitive
features, the sensitive attribute, and the label assigned to an individual, respectively.

Given a dataset of examples $D = \{(\bx_i, a_i, y_i)\}_{i=1}^n\in(\calX, \calA, \{0, 1\})^n$ of size
$n$, we say a classifier $c\in\calC_P$ satisfies $\alpha$-\textit{Equalized Odds} if
\begin{equation}
    \max_{a, a'\in\calA}\{\max(|\hat{FP}_a - \hat{FP}_{a'}|, |\hat{TP}_a - \hat{TP}_{a'}| )\}\leq \alpha
    \label{eq:eqodds}
\end{equation}

where $\hat{FP}_a, \hat{TP}_a$ are empirical estimates of
$FP_a(c) = \pr_{(\bx, a, y)}[c(\bx) = 1 | A = a, y = 0]$,
$TP_a(c) = \pr_{(\bx, a, y)}[c(\bx) = 1 | A = a, y = 1]$ respectively on dataset $D$.
\footnote{
$FP_a(c)$ is usually referred to as the false positive rate on attribute $A=a$.
Likewise, $FN_a(c)$ and $TP_a(c)$ are the false negative and true positive rates
on attribute $A=a$ respectively.
}

\label{def:alphaeo}
\end{definition}

We say a classifier satisfies $\alpha$-\textit{Smoothed Equalized Odds} if the smoothed version of Equation~\ref{eq:eqodds} is satisfied
(i.e., when the maximum and absolute functions in Equation~\ref{eq:eqodds} are replaced with smoothed
versions
\footnote{
For example, the smooth maximum function is a smooth approximation
to the maximum function.
} or using the \text{Moreau-Yosida} regularization technique).

For concreteness, we provide a specific smoothed version of $\alpha$-equalized odds in Definition~\ref{def:alphaseo}.

\begin{definition}[$(\alpha, \eta)$-Smoothed Equalized Odds]
Let $X, A, Y$ be random variables representing the non-sensitive
features, the sensitive attribute, and the label assigned to an individual, respectively.

Given a dataset of examples $D = \{(\bx_i, a_i, y_i)\}_{i=1}^n\in(\calX, \calA, \{0, 1\})^n$ of size
$n$, we say a classifier $c\in\calC_P$ satisfies $(\alpha, \eta)$ \textit{Equalized Odds} if the constraint function
\begin{equation}
    g(\hat{FP}, \hat{FN}, \hat{TP}) = \smax^\eta_{a, a'\in\calA}\{\max(|\hat{FP}_a - \hat{FP}_{a'}|, |\hat{TP}_a - \hat{TP}_{a'}| )\} - \alpha
    \label{eq:eqodds2}
\end{equation}

is less than or equal to 0. $(\hat{FP}, \hat{FN}, \hat{TP})$ corresponds to the $3|\calA|$ empirical estimates of the false
positives, false negatives, and true positives for the
$|\calA|$ groups.
$(\hat{FP}, \hat{TP})$ are used to enforce the equalized odds constraint while
$(\hat{FP}, \hat{FN})$ are used to compute the error of the classifier.
We use the smooth maximum function
$\smax^\eta(y_1, \ldots, y_n) = \frac{\sum_{i=1}^ny_ie^{\eta y_i}}{\sum_{i=1}^ne^{\eta y_i}}$~\citep{LangeZHV14} as a replacement for the non-smooth maximum function. As $\eta\rightarrow\infty$, $\smax^\eta\rightarrow\max$.

\label{def:alphaseo}
\end{definition}

Note that the solutions that satisfy Definition~\ref{def:alphaeo} might differ from the ones that satisfy Definition~\ref{def:alphaseo}
because of the cost of smoothing parameterized by $\eta$. Also, the
gradient of $\smax^\eta$ is given by
$\nabla_{y_i}\smax^\eta(y_1, \ldots, y_n) = \frac{e^{\eta y_i}}{\sum_{j=1}^ne^{\eta y_j}}[1+\eta(y_i-\smax^\eta(y_1, \ldots, y_n))]$.

\begin{cori}
For any privacy parameters $\eps, \delta\in(0, 1]$,
suppose we have a dataset of examples $D = \{(\bx_i, a_i, y_i)\}_{i=1}^n$ of size
$n$ where
$\bx_i\in\calX, y_i\in\{0, 1\}$, $a_i\in\calA$, for all
$i\in[n]$. Assume that
there exists at least one decision in $\calC_P$
(with finite VC dimension of at most $\VC(\calC_P)$) that
satisfies $(\alpha, \eta)$-Smoothed Equalized Odds
(by Definition~\ref{def:alphaseo}) for some $\eta > 0$.

Then there exists $n_0 = \tilde{O}\left(\frac{|\calA|^4\cdot\VC(\calC_P)}{\eps\alpha^3}\right)$
such that
for all $n\geq n_0$, given access to a $\lopt_{\eps, \delta}$ oracle,
we can, with probability at least $9/10$, 
obtain a decision $\tilde{c}\in\calC_P$ satisfying 
$(\alpha, \eta)$-smoothed equalized odds and that is within $\alpha$
away from the most accurate classifier.

\label{cori:eoapprox}
\end{cori}

We provide the proof for Corollary~\ref{cori:eoapprox} 
as Corollary~\ref{cor:eoapprox} in Section~\ref{sec:reductions}.
Corollary~\ref{cori:eoapprox} uses Theorem~\ref{thm:paik}
as the base theorem.
In comparison, in the regime where, in their formulation, $\min_{a, y}\{\hat{q}_{ay}\} \leq \alpha^{(1+r)/2}$
for any $r > 0$
(see Section~\ref{sec:csc} in the Appendix for more details), their
methods can solve the $\cgoo(\calC_P, n, K, f, g, \ell, D, \alpha)$ problem using
sample complexity $\tilde{O}\left(\frac{|\calA|^3\cdot\VC(\calC_P)}{\eps\alpha^{3+r}}\right)$.
\footnote{$\hat{q}_{ay}$ is an empirical estimate for $\pr[A = a, Y=y]$ where $y\in\{0, 1\}$ and $a\in\calA$. A small
$\min_{a, y}\{\hat{q}_{ay}\}$ results when the sample size for
a particular attribute is small.}
In Section~\ref{sec:csc},
we state their main theorem (Theorem~\ref{thm:csc}) and a corollary (Corollary
~\ref{cor:csc}) showing the sample complexity required for their
algorithm to solve the $\cgoo(\calC_P, n, K, f, g, \ell, D, \alpha)$ problem when applied to the
Equalized Odds fairness definition.
On the other hand,
Corollary~\ref{cori:eoapprox} results in sample size
$\tilde{O}\left(\frac{|\calA|^4\cdot\VC(\calC_P)}{\eps\alpha^3}\right)$.
As a result, by Corollary~\ref{cori:eoapprox}, the linear optimization based algorithm for
Theorem~\ref{thm:paik} performs better for all $r > 0$ and $|\calA| < 1/\alpha^r$
(in terms of asymptotic sample
complexity for the accuracy parameter $\alpha > 0$) 
than the \textbf{DP-oracle-learner} (which uses a private version of a cost-sensitive classification oracle $\cs(\calC_P)$ in each iteration of their algorithm) of~\cite{Jagielski18}. 
Comparing the results of~\cite{Jagielski18} to Corollary~\ref{cori:eoapprox},
we see that our results hold under different oracle assumptions
and for a smoothed version of the equalized odds constraint.
As a result, the comparison is not as direct as we would like.

\subsubsection{Case Study: Privately Selecting Powerful Statistical Tests}

Essentially, any problem that can be
simulated via the use of a confusion
matrix (i.e., empirical estimates of Type I, II
error) can be solved using our
framework. 

Our results can also be applied to hypothesis testing to,
for example, select high-power test statistics.
A hypothesis is \textbf{simple} if it completely specifies the
data distribution. The hypothesis $H_i : \theta\in\Omega_i$ is simple when $|\Omega_i| = 1$.
Let $Y$ be the observed data.
For such simple hypothesis, let $p_0, p_1$ denote densities of $Y$ under $H_0$ (the null hypothesis)
and $H_1$ (the alternative) respectively.
When both $H_0, H_1$ are simple then the Neyman-Pearson lemma
completely characterizes all tests on the competing hypothesis via the likelihood ratio
$L(y) = p_1(y)/p_0(y)$~\citep{keener2010theoretical}.
For any $y$, let $p_0(y), p_1(y)$
denote the density of $y$ under $H_0$ and $H_1$ respectively.
Anagolues of the Neyman-Pearson lemma have been studied in the differential privacy
literature~\citep{KairouzOV17, CanonneKMSU19}.

The power function for a  simple test function $\phi$
(that returns the probability of rejecting the null
hypothesis)
has two possible values:
$$
\alpha = \gamma_0 = \E_0\phi = \int\phi(y)p_0(y)dy,\quad
\gamma_1 = \E_1\phi = \int\phi(y)p_1(y)dy,
$$
where the level is $\alpha$ should be as close to zero as possible and $\gamma_1$ should be close to
one.
The goal would be to maximize $\gamma_1$ among all tests $\phi$ with $\alpha = \E_0\phi$.
This is
a constrained maximization problem, for which our work
shows the existence of
oracle-efficient differentially private (empirical risk) solvers
for a fixed dataset $Y$. See the informal Corollary~\ref{cori:test}, which follows from the
formal Corollary~\ref{cor:loptexp} statement.

\begin{prop}[Neyman-Pearson Lemma~\citep{NeymanPearson33}]
Given any level $\alpha = \E_0\phi\in [0, 1]$, there exists a likelihood ratio test $\phi_{\alpha}$
with level $\alpha$ and any likelihood ratio test with level $\alpha$ maximizes
$\E_1\phi$ among all tests with level at most $\alpha$.
\end{prop}

\begin{cori}
There exists an oracle-efficient
$(\eps, \delta)$-differentially private algorithm
that, with probability at least 9/10, returns a 
test statistic with target significance level $\alpha\in(0, 1]$ and is
$\alpha$ away from the most powerful test statistic.
\label{cori:test}
\end{cori}

We defer the explicit
construction of such algorithms (for privately selecting
high-power test statistics)
to future work.

\section{Related Work}

Below we briefly specify a few other works related to the material presented in this paper.

\textbf{Adversarial Prediction}: Adversarial
prediction (via Lagrangian duality, for example) for
multi-objective optimization 
is the main workhorse of most algorithmic
fairness frameworks~\citep{FreundS97, NIPS2015_dfa92d8f}. Multi-objective adversarial prediction
builds off of work of mathematicians
David Blackwell (Blackwell's Approachability Theorem
~\citep{Blackwell56}) and James Hannan~\citep{Hannan57}.
See~\citep{CL06} for a survey on learning and games.

~\cite{AlghamdiAWCWR20} define a model projection
framework which can be viewed via the lens of Lagrangian duality but do not analyze the 
computational efficiency of their solutions.
We aim to delineate the computational
efficiency of such information-theoretic problems.

\textbf{Reductions Approach to Fair Classification and Regression}:
\cite{AgarwalBD0W18} explore the problem of using black-box optimizers to
minimize group-fair
convex objectives subject to constraint functions. 
\cite{AlabiIK18} extend this work to handle any
Lipschitz-continuous group objective of losses
given oracle access to an approximate
linear optimizer in time polynomial in the inverse of the accuracy parameter.
Furthermore, they extend their results to learning
using a polynomial number of examples and access to an agnostic learner.
Our definition of the constrained group-objective optimization problem is inspired by the
work and results of~\cite{AlabiIK18}.
Additionally,~\cite{NarasimhanRS015, Narasimhan18, HBMK19} explore optimizing convex objectives of the
confusion matrix (such as those derived from $G$-mean, $H$-mean typically used for class-imbalanced problems)
and fractional-convex functions of the confusion matrix (such as $F_1$ measure used in text
retrieval).

In this paper, we consider some of the use cases explored by previous works but also add on the
additional constraint of data privacy, an important constraint given that fairness is often
imposed with respect to the sensitive attributes of data subjects.

\textbf{Private Empirical Risk Minimization}: Differentially private empirical risk minimization
in the convex setting has been considered in a variety of settings
\citep{ChaudhuriMS11, KiferST12, BassilyST14, TalwarT014, TalwarTZ15, SteinkeU15, WangYX18, Iyengar19}
with algorithm-specific upper and algorithm-agnostic lower bounds provided in some cases.
We largely build upon these works.

\textbf{Private Fair Learning}:
\cite{Jagielski18} initiate the study of differentially private fair learning
but only consider the equalized odds definition
in the reductions approach to fair learning.
\cite{EkstrandJM18} discuss an agenda for subproblems that should be considered when trying to
achieve data privacy for fair learning.
Last,~\cite{pmlr-v80-kilbertus18a} study how to learn models that are fair by 
encrypting sensitive attributes and using secure multiparty computation.

\section{Preliminaries and Notation}

Here we introduce preliminaries and notation that might be useful to
parse through later sections.

\subsection{Differential Privacy}

For the definitions below, for any two datasets $D, D'\in\calZ^n$, we use
$D\sim D'$ to mean that $D$ and $D'$ are neighboring datasets that differ in exactly one row.

\begin{definition}[(Pure) $\eps$-Differential Privacy~\citep{DworkMNS06}]
For any $\eps\geq 0$, we say that a (randomized) mechanism
$\calM:\calZ^n\rightarrow\calR$ is $\eps$-\textbf{differentially private} if for
every two neighboring datasets $D\sim D'\in\calZ^n$,
we have that
$$
\forall T\subseteq\calR, \pr[\calM(D)\in T] \leq e^{\eps}\cdot\pr[\calM(D')\in T].
$$
\end{definition}

We usually take $\eps$ to be small but not cryptographically small. For example, typically
we set $\eps\in [0.1, 1]$. The smaller $\eps$ is, the more privacy is guaranteed.

\begin{definition}[(Approximate) $(\eps, \delta)$-Differential Privacy]
For any $\eps \geq 0, \delta\in[0, 1]$, we say that a (randomized) mechanism
$\calM:\calZ^n\rightarrow\calR$ is $(\eps, \delta)$-\textbf{differentially private} if for
every two neighboring datasets $D\sim D'\in\calZ^n$, we have that
$$
\forall T\subseteq\calR, \pr[\calM(D)\in T] \leq e^{\eps}\cdot\pr[\calM(D')\in T] + \delta.
$$
\end{definition}

We insist that $\delta$ be cryptographically negligible i.e.,
$\delta \leq n^{-\omega(1)}$. The value $\delta$ can be interpreted as an upper-bound
on the probability of a catastrophic event (such as publishing the entire dataset)\citep{Vadhan17}.
$(\eps, \delta)$-differential privacy can also be interpreted as
``(pure) $\eps$-differential privacy with probability at least $1-\delta$.''
The smaller $\eps$ and $\delta$ are, the more privacy is guaranteed.

\begin{definition}[$\ell_1$-sensitivity of a function]
The $\ell_1$ \textbf{sensitivity} of a function $f:\calZ^n\rightarrow\reals^K$ is
$$
\Delta_1(f) = \max_{D, D'\in\calZ^n: D\sim D'} \norm{f(D) - f(D')}_1,
$$
where $D\sim D'\in\calZ^n$ are neighboring datasets.
\end{definition}

\begin{definition}[$\ell_2$-sensitivity of a function]
The $\ell_2$ \textbf{sensitivity} of a function $f:\calZ^n\rightarrow\reals^K$ is
$$
\Delta_2(f) = \max_{D, D'\in\calZ^n: D\sim D'} \norm{f(D) - f(D')}_2,
$$
where $D\sim D'\in\calZ^n$ are neighboring datasets.
\end{definition}

% \begin{theorem}[Laplace Mechanism~\citep{DworkMNS06}]
% For any privacy parameter $\eps > 0$ and any
% given query function $f:(\calX\times\calA\times\calY)^n\rightarrow\reals^K$ and database
% $S\in(\calX\times\calA\times\calY)^n$, the Laplace mechanism outputs:
% $$
% \tilde{f}_L(S) = f(S) + (R_1, \ldots, R_K),
% $$
% where $R_1, \ldots, R_K\sim\Lap(\frac{\Delta_1(f)}{\eps})$ are i.i.d random variables.

% The Laplace mechanism is $\eps$-differentially private.
% \end{theorem}

\begin{theorem}[Exponential Mechanism~\citep{McSherryT07}]
For any privacy parameter $\eps > 0$ and any
given loss function $h:\calC_P\times\calZ^n\rightarrow\reals$ and database
$D\in\calZ^n$, the Exponential mechanism outputs $c\in\calC_P$ with probability
proportional to $\exp(\frac{-\eps\cdot h(c, D)}{2\Delta h})$ where
$$\Delta h = \max_{c\in\calC_P} \max_{D, D'\in\calZ^n: D\sim D'}|h(c, D)-h(c, D')|$$
is the sensitivity of the loss function $h$.
\label{thm:expmech}
\end{theorem}

\begin{theorem}[Privacy-Utility Tradeoffs of Exponential Mechanism~\citep{McSherryT07}]
For any database $D\in\calZ^n$, 
let $c^* = \argmin_{c\in\calC_P}h(c, D)$ and $\tilde{c}_\eps\in\calC_P$ be the output
of the Exponential Mechanism satisfying $\eps$-differential privacy. Then with 
probability at least $1-\rho$,
$$|h(\tilde{c}_\eps, D) - h(c^*, D)| \leq \log\left(\frac{|\calC_P|}{\rho}\right)\left(\frac{2\Delta h}{\eps}\right).$$
\label{thm:exputil}
\end{theorem}

% \begin{theorem}[Gaussian Mechanism~\citep{DworkR14}]
% For any privacy parameters $\eps > 0, \delta\in(0, 1)$ and any
% given query function $f:\calZ^n\rightarrow\reals^K$ and database
% $D\in\calZ^n$, the Gaussian mechanism outputs:
% $$
% \tilde{f}_G(D) = f(D) + (R_1, \ldots, R_K),
% $$
% where $R_1, \ldots, R_K\sim\calN(0, \sigma^2)$ are i.i.d random variables and
% $\sigma\geq\Delta_2(f)\frac{\sqrt{2\log(1.25/\delta)}}{\eps}$.

% The Gaussian mechanism is $(\eps, \delta)$-differentially private.
% \label{thm:gaussian}
% \end{theorem}

\begin{lemma}[Post-Processing~\citep{DworkMNS06}]
Let $\calM:\calZ^n\rightarrow\calR$ be an $(\eps, \delta)$-differentially private algorithm
and $f:\calR\rightarrow\calT$ be any (randomized) function. Then 
$f\circ\calM:\calZ^n\rightarrow\calT$ is an $(\eps, \delta)$-differentially private algorithm.
\end{lemma}

The exponential mechanism will be used as the main building block
for our differentially private algorithms for constrained group-objective optimization.
The Laplace and Gaussian mechanisms~\citep{DworkR14, DworkMNS06}
are often used when the goal is to output estimates to a query
(e.g., the mean, sum, or median) while the Exponential mechanism is used when the goal
is to output an object (e.g., a regression coefficient vector or classifier)
with minimum loss (or maximum
utility).

\subsection{Convexity, Smoothness, and Optimization Oracles}

\begin{definition}[Convex Set]
A set $\calV\subset\reals^m$ is a \textit{convex set} if it contains
all of its line segments. That is, $\calV$ is convex iff
$$
\forall\,\,(\bx, \by, \gamma)\in\calV\times\calV\times[0, 1],\,\,
(1-\gamma)\bx + \gamma \by\in \calV.
$$
\end{definition}

\begin{definition}[Convex Function]
A function $f:\calV\rightarrow\reals$ is a \textbf{convex function} if it
always lies below its chords. That is, $f$ is convex iff
$$
\forall\,\,(\bx, \by, \gamma)\in\calV\times\calV\times[0, 1],\,\,
f((1-\gamma)\bx + \gamma \by) \leq (1-\gamma)f(\bx) + \gamma f(\by).
$$
\end{definition}

\begin{definition}[Subgradients]
Let $\calV\subset\reals^m$ and define a function $f:\calV\rightarrow\reals$. Then we say that
$\vec{g}\in\reals^m$ is a \textbf{subgradient} of $f$ at $\bx\in\calV$ if for any $\by\in\calV$ we have that
$$
f(\bx) - f(\by) \leq \vec{g}^T(\bx-\by).
$$
We denote $\partial f(x)$ as the set of subgradients of the function $f$ at $x\in\calV$.
\end{definition}

\begin{definition}[Lipschitz Function]
Let $\calV\subset\reals^m$.
A function $f:\calV\rightarrow\reals$ is $L$-Lipschitz on $\calV$
if for all $\bx, \by\in\calV$, we have
$$
|f(\bx) - f(\by)| \leq L\norm{\bx - \by}.
$$
\end{definition}

\begin{definition}[$\beta$-Smooth Function]
Let $\calV\subset\reals^m$.
A function $f:\calV\rightarrow\reals$ is $\beta$-smooth  if
the gradient $\nabla f$ is $\beta$-Lipschitz. That is, for all $\bx, \by\in\calV$,
$$
\norm{\nabla f(\bx) - \nabla f(\by)} \leq \beta\cdot\norm{\bx - \by}.
$$
\end{definition}

Note that if $f$ is twice-differentiable then $f$ being
$\beta$-smooth is equivalent to the eigenvalues of its Hessians
being smaller than $\beta$.

For our iterative algorithms, we assume access to a linear optimizer oracle that can
solve subproblems of the form
$$
y_t \in\argmin_{\by\in\calV}\bw^T\by
$$
whether exactly or approximately for any $\bw\in\calV\subset\reals^m$.
We previously defined non-private and private approximate linear optimizer oracles
$\lopt, \lopt_{\eps, \delta}$. We will assume the existence of $\lopt$ and provide a generic
construction of its private counterpart.

The overall convex optimization problem will be converted into a series of linear subproblems.
A key property of the use of linear optimizers in the (vanilla) Frank-Wolfe algorithm
is that the projection step of projected gradient descent algorithms is replaced with
a linear optimization step over the set $\calV$. In some cases, solving linear
optimization subproblems will be simpler and more computationally efficient to solve than projections into some feasible set.

\section{Constrained Group-Objective Optimization via Weighting}
\label{sec:cgoo}

In this section, we present a key lemma and corollary that will be crucial to
the algorithms we will present in this paper.
The iterative linear optimization based algorithms will solve the
constrained group-objective optimization problem
(Definition~\ref{def:goo}) in the setting where $f, g$ are convex, Lipschitz functions.

For the iterative algorithms we will present, we assume that $\calC_P$ is
\textit{closed under randomization}.
That is, for every $c_1, \ldots, c_T\in\calC_P$, if $c\in\Delta(\{c_i\}_{i=1}^T)$ then
$c\in\calC_P$. 
For any $i\in[T]$,
$c$ will predict $c_i(x)$ with probability $w_i$ where $\sum_{i=1}^T w_i = 1$.
We also assume that
we can return randomized decisions defined over $\Delta(\calC_P)$.

Having settled on a reductionist optimization problem (Definition~\ref{def:goo}), the goal will be to
obtain a decision $\hat{c}\in\calC_P$ for which
\begin{align}
     &\E[f(\ell(\hat{c}, D))] \leq f(\ell(c^*, D)) + \alpha, \,
     &\E[g(\ell(\hat{c}, D))] \leq \alpha \nonumber\\
     &\quad\quad\quad\textbf{OR}\nonumber\\
     \text{ w.p. }\geq 1 - \rho,\, \rho\in(0, 1), \quad
     &f(\ell(\hat{c}, D)) \leq f(\ell(c^*, D)) + \alpha, \,
     &g(\ell(\hat{c}, D)) \leq \alpha
 \label{eq:guarantee}
 \end{align}
where $f, g:[0, 1]^K\rightarrow\reals$ are functions for which
$c^* \in \argmin_{c\in\calC_P:g(\loss)\leq 0}f(\loss) + \alpha$ is the best decision (according to
$f(\cdot)$) that satisfies the constraint function $g(\cdot)$ and $D$ is
a fixed dataset of size $n$.
% \footnote{For any $\alpha_1, \alpha_2 > 0$ where $\alpha_1\neq\alpha_2$,
% we can always obtain $\E[f(\ell(\hat{c}, D))] \leq \alpha_1$ and $\E[g(\ell(\hat{c}, D))]\leq \alpha_2$
% by optimizing both functions to within $\alpha = \min\{\alpha_1, \alpha_2\}$ instead.} 
The expectation or the high probability bound
is over the random coins of the algorithm that chooses $\hat{c}$.

To reach the guarantee in Equation~\eqref{eq:guarantee}, we rely on the following key
lemma and corollary which results in a
weighted private gradients optimization strategy when the
additional constraint of privacy is added in the case of the first-order optimization algorithms. For this
strategy, we essentially optimize two functions simultaneously while ensuring
privacy by weighting the gradients of the functions $f$ and $g$. As a
consequence, in the case of the use of output perturbation,
the standard deviation of the noise distribution used to ensure privacy
will also scale with the weights applied to the gradients of $f$ and $g$.

\begin{lemma}
For any Lipschitz continuous functions $f, g: [0, 1]^K\rightarrow\reals$,
suppose that there exists $\by\in[0, 1]^K$ such that
$g(\by)\leq 0$.

For any $G > 0$, define the function
$h:[0, 1]^K\rightarrow\reals$ as follows:
$h(\bx) = f(\bx) + G\cdot\max(0, g(\bx))$ for any $\bx\in[0, 1]^K$.
Then for all $\bx'\in[0, 1]^K, \alpha > 0$ such that
$h(\bx') \leq \min_{\bx\in[0, 1]^K}h(\bx) + \alpha$, we are guaranteed that
$$f(\bx') \leq \min_{\bx\in[0, 1]^K:g(\bx)\leq 0}f(\bx) + \alpha, \quad\quad g(\bx') \leq \frac{\alpha + L_f\sqrt{K}}{G}$$
where $L_f$ is the Lipschitz constant for the function $f$.
\label{lem:h}
\end{lemma}

\begin{proof}
Let $\alpha > 0$ and $G > 0$. Then
for all $\bx'\in[0, 1]^K$ such that $h(\bx') \leq \min_{\bx\in[0, 1]^K}h(\bx) + \alpha,$
$$h(\bx') = f(\bx') + G\cdot\max(0, g(\bx')) \leq \min_{\bx\in[0, 1]^K:g(\bx)\leq 0}f(\bx) + \alpha$$ implies that
\begin{enumerate}
    \item $f(\bx') \leq \min_{\bx\in[0, 1]^K:g(\bx)\leq 0}f(\bx) + \alpha$;
    \item $g(\bx') \leq \frac{\alpha + L_f\sqrt{K}}{G}$ since by the definition of Lipschitz constants we have
$\max_{\bx, \bx'\in[0, 1]^K}f(\bx)-f(\bx') \leq L_f\norm{\bx - \bx'} \leq L_f\sqrt{K}$ since
$\bx, \bx'\in [0, 1]^K$
by definition.
\end{enumerate}
\end{proof}

\begin{corollary}
Define $h(\bx) = f(\bx) + \frac{\alpha + L_f\sqrt{K}}{\alpha}\max(0, g(\bx))$ for
all $\bx\in[0, 1]^K$.
Then for all $\bx'\in[0, 1]^K, \alpha > 0$ such that
$h(\bx') \leq \min_{\bx\in[0, 1]^K}h(\bx) + \alpha$, we are guaranteed that
$$f(\bx') \leq \min_{\bx\in[0, 1]^K:g(\bx)\leq 0}f(\bx) + \alpha, \quad\quad g(\bx') \leq \alpha.$$
\label{cor:h}
\end{corollary}

\begin{proof}
The corollary follows from Lemma~\ref{lem:h} by setting $G = \frac{\alpha + L_f\sqrt{K}}{\alpha}$.
\end{proof}

Since $f, g$ are Lipschitz continuous and for all $c\in\calC_P$ and datasets $D$ of size $n$, 
$\ell(c, D)\in[0, 1]^K$ (by Definition), we know that using Corollary
~\ref{cor:h} we can achieve Equation~\eqref{eq:guarantee}. 
This will be key
to our constrained group-objective optimization algorithms both in the
privacy-preserving and the non-privacy-preserving cases.

We will go on to show a linear optimization based algorithm to achieve
the guarantee in Equation~\eqref{eq:guarantee} both with and without privacy
guarantees. But first we will present an exponential sampling $(\eps, 0)$-differentially private
algorithm that directly applies Lemma~\ref{lem:h}.

\section{Algorithms for Private Constrained Group-Objective Optimization}
\label{sec:algs}

We present algorithms to solve the constrained group-objective optimization problem\\
$\cgoo(\calC_P, n, K, f, g, \ell, D, \alpha)$.
To simplify analysis and notation, we assume that both functions
$f$ and $g$ are 1-Lipschitz functions (i.e., their Lipschitz constants are
$L_f = L_g = 1$). For general
$L_f$-Lipschitz function $f$ and $L_g$-Lipschitz function $g$, we can run the algorithms on
$f/L_f$ and $g/L_g$ with accuracy parameter $\alpha/\max\{L_f, L_g\}$.

In this section, our goal is to use an algorithmic
approach to privately obtain a decision $\tilde{c}\in\calC_P$ satisfying the
guarantee given in Equation~\ref{eq:guarantee}.
The privacy and utility guarantees will be in terms of a high
probability bound rather than an expectation bound.
The randomness will be taken
over the random coins of the algorithm.
We will go on to analyze the effects of imposing
the additional constraint of $(\eps, 0)$ or
$(\eps, \delta)$-differential privacy in the computation of the decision
$\tilde{c}\in\calC_P$ that will be returned by the empirical risk minimization
algorithms. Upper and lower bounds for the oracle complexity of solving this problem will be presented.

For the iterative algorithms,
we assume that we have oracle access to the convex functions
$f, g: [0, 1]^K\rightarrow\reals$ and their corresponding gradient oracles
$\nabla f, \nabla g: [0, 1]^K\rightarrow\reals^K$ and upper bound the
oracle complexity of obtaining $\tilde{c}\in\calC_P$ in a privacy-preserving manner.
We note that even if $f$ and $g$ are not convex and smooth, there exists techniques
for smoothing the functions (e.g., see 
\textit{Moreau-Yosida regularization}~\citep{Nesterov05} and other
techniques in~\citep{Manning:2008}).

Key to the definition of differential privacy is a notion of 
adjacency (or neighboring) of datasets i.e., datasets that differ in one row.
Let $D, D'$ be neighboring datasets of size $n$. We will use the
relation between $D, D'$ to obtain better noise parameters to ensure 
differential privacy. Samples from the Laplace, Exponential, or Normal
distribution are often used to perturb the output of a function
(or gradient of a function) to ensure privacy. The standard deviation of the noise distribution
from which the samples are drawn will decrease as
$n\rightarrow\infty$.
Suppose that $\beta_f, \beta_g$ are the smoothness parameters of the functions
$f$ and $g$ and $L_f, L_g$ are the Lipschitz constants of $f$ and $g$, 
then for any setting of $G > 0$, we can define the function
$h:[0, 1]^K\rightarrow\reals$ as follows:
$h(\ell(c, D)) = f(\ell(c, D)) + G\cdot\max(0, g(\ell(c, D)))$ for any $c\in\calC_P$ and dataset $D$. 
Then for any neighboring datasets $D, D'$,
by Lemma~\ref{lem:gs}, we can bound
$\norm{\nabla h(\ell(c, D)) - \nabla h(\ell(c, D'))}$
and
$|h(\ell(c, D)) - h(\ell(c, D'))|$.
We will use these bounds for the ($\ell_1$ and $\ell_2$) 
global sensitivities of the functions we will optimize
in a differentially private way.

\begin{lemma}
Let $L_f, L_g$ be the Lipschitz constants of the functions $f:[0, 1]^K\rightarrow\reals$ 
and $g:[0, 1]^K\rightarrow\reals$ respectively. And let
$\beta_f, \beta_g$ be the Lipschitz constants of their gradients $\nabla f, \nabla g$ respectively.
Then for any setting of $G > 0$, define $h(\ell(c, D)) = f(\ell(c, D)) + G\cdot\max(0, g(\ell(c, D)))$.
For any neighboring datasets $D, D'$ and $c\in\calC_P$, we have
$\norm{\nabla h(\ell(c, D)) - \nabla h(\ell(c, D'))} \leq (\beta_f + G\cdot\beta_g)\frac{\sqrt{K}}{n}$
and
$|h(\ell(c, D)) - h(\ell(c, D'))|\leq (L_f+G\cdot L_g)\frac{\sqrt{K}}{n}$
since $D, D'$ are neighboring datasets and $\ell(c, D), \ell(c, D')\in[0, 1]^K$.
\label{lem:gs}
\end{lemma}

\begin{proof}
We proceed to use the definitions of $f, g$ and $\ell$.
Also, recall that we defined
$\ell(c, D)$ as an average of losses over $D$
i.e., $\ell(c, D) = \frac{1}{n}\sum_{i=1}^n\ell(c, D_i)$. 
Then
$$\norm{\nabla h(\ell(c, D)) - \nabla h(\ell(c, D'))} \leq (\beta_f + G\cdot\beta_g)\norm{\ell(c, D) - \ell(c, D')} \leq (\beta_f + G\cdot\beta_g)\frac{\sqrt{K}}{n},$$
since $f, g$ are $\beta_f$-smooth, $\beta_g$-smooth respectively.
Further,
$$|h(\ell(c, D)) - h(\ell(c, D'))| \leq (L_f + G\cdot L_g)\norm{\ell(c, D) - \ell(c, D')} \leq (L_f + G\cdot L_g)\frac{\sqrt{K}}{n},$$
since $f, g$ are $L_f$-Lipschitz, $L_g$-Lipschitz respectively.
\end{proof}

Now we go on to present procedures to obtain a decision $\tilde{c}\in\calC_P$ that
solves the constrained group-objective optimization problem
(Definition~\ref{def:goo}) with and without privacy.
Along with the algorithms, we will present
oracle complexity upper bounds on the excess risk (or equivalently, the sample complexity) 
for these procedures.

\subsection{Exponential Sampling}

Without the use of an optimization oracle (for a specific implementation of the exponential mechanism),
the following is a generic exponential mechanism to solve the constrained group-objective convex
optimization problem with privacy. This method assumes we have an oracle to sample from the set $\calC_P$ -- assumed to be convex -- with a certain probability.

\begin{theorem}
Suppose we are given convex 1-Lipschitz functions
$f, g: [0, 1]^K\rightarrow\reals$, loss function
$\ell: \calC_P\times(\calX\times\calA\times\calY)^n\rightarrow[0, 1]^K$,
privacy parameter $\eps > 0$, and $\calC_P$ (with finite VC dimension $\VC(\calC_P)$ and resulting parameter
space in $\reals^P$).

If $\ell$ is a convex function and $f, g$ are non-decreasing,
then let $n_0 = O(\frac{K\cdot P}{\eps\alpha^2})$. If not, let
$n_0 = \tilde{O}(\frac{K\cdot \VC(\calC_P)}{\eps\alpha^2})$.
Then there exists $n_0$ such that for all
$n \geq n_0$ and $\eps > 0$ if we set $G = O(\frac{\sqrt{K}}{\alpha})$,
Algorithm~\ref{alg:cgoexpc} is an $\eps$-differentially private algorithm
that, with probability at least 9/10,
returns a decision $\tilde{c}\in\calC_P$ with the following guarantee:

$$f(\ell(\tilde{c}, D)) \leq f(\bloss) + \alpha,
\quad\quad
g(\ell(\tilde{c}, D)) \leq \alpha,$$

where $c^* \in \argmin_{c\in\calC_P:g(\loss)\leq 0}f(\loss)$ is the best decision
in the feasible decision set $\calC_P$, given dataset $D$ of size $n$.

The algorithm is guaranteed to be computationally efficient in the case where $\ell$ is convex and $f, g$ are non-decreasing.
\label{thm:expc}
\end{theorem}

\begin{proof}

The proof of privacy follows from a direct application of the
Exponential Mechanism (see Theorem~\ref{thm:expmech}) with loss function
$$
h(\ell(c, D)) = f(\ell(c, D)) + G\cdot\max(0, g(\ell(c, D))),
$$
defined for any $c\in\calC_P$ and dataset $D$. By Lemma~\ref{lem:gs}, the 
sensitivity of this function is at most $\frac{\sqrt{K}(1+G)}{n}$.

First, let us consider the case where $\ell$ is convex and $f, g$ are non-decreasing.
If we naively applied the exponential mechanism utility analysis, we will get
a dependence on the size of either $\calC_P$ (the decision set) or
$\ell(\calC_P, D)$ (see Theorem~\ref{thm:exputil}).
In order to avoid this we will rely on a ``peeling'' argument
of convex optimization already analyzed by~\cite{BassilyST14}.
This argument allows us to get rid of the extra logarithmic factor on the size of the set
$\calC_P$ (which could be infinite).
Even though
their results are written in expectation, we use the high probability version
which gives that with probability at least 9/10,
$$
h(\ell(\tilde{c}, D)) - h(\bloss) = O\left(\frac{P\sqrt{K}(1+G)}{\eps n}\right)
$$
by Corollary~\ref{cor:expsamp} since the sensitivity of $h$
is at most $\frac{\sqrt{K}(1+G)}{n}$ (by Lemma~\ref{lem:gs}).

By Corollary~\ref{cor:h}, to optimize both $f, g$ to within $\alpha$, we set
$G = O(\frac{\sqrt{K}}{\alpha})$. As a result, we obtain that there exists $n_o = O\left(\frac{KP}{\eps\alpha^2}\right)$ such that 
for all $n\geq n_0$, we can
apply Corollary~\ref{cor:h} to obtain the guarantees stated in the theorem.

Now, if $\ell$ is not convex or $f, g$ are not non-decreasing,
we rely on the generic guarantees of the exponential
mechanism (see Theorem~\ref{thm:exputil}) where we use that by Sauer's Lemma, the range of
the exponential mechanism is bounded by $|\calC_P(D_\calX)|\leq O(n^{\VC(\calC_P)})$ where $\VC(\calC_P)$ is the
VC dimension of $\calC_P$.
The VC dimension bound allows us to essentially replace the $P$ in the sample
complexity with $\VC(\calC_P)$ (up to polylogarithmic factors).

\end{proof}

\begin{corollary}
There exists an $(\eps, 0)$-differentially private exponential sampling based convex optimization algorithm
(Algorithm 2 in~\citep{BassilyST14}) that for any convex, non-decreasing
function $h:[0, 1]^K\rightarrow\reals$ and convex loss function $\ell:\calC_P\times(\calX\times\calA\times\calY)^n\rightarrow[0, 1]^K$
outputs a decision
$\tilde{c}\in\calC_P$ such that for all $\theta > 0,$
$$
\pr\left[h(\ell(\tilde{c}, D)) - h(\ell(c^*, D)) \geq \frac{8\Delta}{\eps}((P+1)\log 3 + \theta)\right] \leq e^{-\theta}
$$
where $c^* \in \argmin_{c\in\calC_P}h(\loss)$ and $\Delta$ is an
upper bound on the
sensitivity of $h\circ\ell$.

This theorem holds when $\calC_P$ is a convex set.
\label{cor:expsamp}
\end{corollary}

\begin{proof}
Follows from the high-probability version of Theorem 3.2 in~\citep{BassilyST14} 
(stated as Theorem~\ref{thm:bstexpsamp}) since $\ell$ is convex
and $h$ is convex, non-decreasing so that $h\circ\ell$ is also convex.
\end{proof}

\begin{theorem}
Let $k:\calC_P\times(\calX\times\calA\times\calY)^n\rightarrow\reals$ be any convex, $K$-Lipschitz
function we wish to minimize and $\calC_P$ be a convex decision set.
Then there exists an $(\eps, 0)$-differentially private algorithm that runs in time polynomial in
$n, P$ and outputs $\tilde{c}$ such that
for any $\theta > 0$ and $D\in(\calX\times\calA\times\calY)^n$,
$$
\pr\left[\sum_{i=1}^n k(\tilde{c}, D) - \sum_{i=1}^n k(c^*, D)\geq \frac{8\Delta(K)}{\eps}((P+1)\log 3 + \theta)\right] \leq e^{-\theta},
$$
where $c^* \in \argmin_{c\in\calC_P}\sum_{i=1}^n k(c, D)$ and 
$\Delta(K)$, a function of $K$,
is an upper bound on the
sensitivity of the function $k$.
\label{thm:bstexpsamp}
\end{theorem}

\begin{proof}

Follows from the w.h.p. version of Theorem 3.2 in~\citep{BassilyST14}.

\end{proof}

\begin{algorithm}

\KwIn{$\ell, f, g, D \in (\calX\times\calA\times\calY)^n, \calC_P, G, \eps$}

\

Set $h(\ell(c, D)) = f(\ell(c, D)) + G\cdot\max(0, g(\ell(c, D)))$

Sample $\tilde{c}\in\calC_P$ with probability $\propto\exp\left(-\frac{\eps\cdot n\cdot h(\ell(c, D))}{2\sqrt{K}(1 + G)}\right)$

\Return $\tilde{c}$

\caption{Exponential Sampling for Constrained Group-Objective Optimization.}
\label{alg:cgoexpc}
\end{algorithm}

In later sections, we will show that the sample complexity to solve
constrained group objective optimization is lower-bounded by 
$n = \Omega\left(\frac{K}{\eps\alpha}\right)$ for (pure) $\eps$-differential privacy (with probability at least $1/2$).
As a result, there is a multiplicative gap of $O(\frac{P}{\alpha})$ or
$\tilde{O}(\frac{\VC(\calC_P)}{\alpha})$
between the upper bound and lower bound. This gap is a direct result of the way we minimize
$f$ subject to the constraint of $g$ by jointly minimizing a composition of these functions.
We note that our results are optimal amongst all such
strategies that jointly minimize a composition of
these functions.
%\textit{An open problem is to close this gap in certain parameter regimes or for specific
%problems (e.g. linear regression)}.

\subsection{Linear Optimization Based Algorithm without Privacy}

In this section, we essentially
achieve the same guarantees as in~\citep{AlabiIK18} when
$f, g$ are both convex and Lipschitz-continuous
(see Observation 6 of that paper). We note that the
main theorem in this section is stated and derived in a more general
way than~\citep{AlabiIK18}
so that privacy constraints can be more readily added to the formulation.

As in~\citep{AlabiIK18},
we assume the existence of an approximate linear optimizer oracle solver $\lopt$. We will
translate $\lopt$ with additive error $\tau$ into a $\beta$-multiplicative approximation
algorithm and then apply Theorem~\ref{thm:kkl}. We essentially use the
$\lopt$ oracle to solve the constrained group-objective optimization problem.
The specification of the $\lopt$ oracle is in Definition~\ref{def:lopt}.

\begin{theorem}
Suppose we are given convex 1-Lipschitz functions $f, g: [0, 1]^K\rightarrow\reals$,
loss function $\ell:\calC_P\times(\calX\times\calA\times\calY)^n\rightarrow[0, 1]^K$. Then assuming
we have access to an approximate linear optimizer oracle $\lopt$ (Definition~\ref{def:lopt}),
after $T = O\left(\frac{K^4}{\alpha^2}\right)$ calls to $\lopt$, with probability at least 9/10,
we will obtain a decision $\hat{c}\in\calC_P$ with the following guarantee:
$$f(\ell(\hat{c}, D)) \leq f(\bloss) + \alpha,
\quad\quad
g(\ell(\hat{c}, D)) \leq \alpha,$$

for any $\alpha\in (0, 1]$
where $c^* \in \argmin_{c\in\calC_P:g(\loss)\leq 0}f(\loss)$ is the best decision
in the feasible set $\calC_P$, given dataset $D$ of size $n$ such that $\ell(\calC_P, D)\subset[0, 1]^K$ is
compact.
\label{thm:aik}
\end{theorem}

\begin{proof}

Given the functions $f, g$, we can define the ``new'' function
$h(\ell(c, D)) = f(\ell(c, D)) + G\cdot\max(0, g(\ell(c, D)))$
for any $c\in\calC_P$ and dataset $D$ of size $n$.
Since $f, g$ are 1-Lipschitz and convex we know that
$\norm{\nabla f(\ell(c, D))}, \norm{\nabla g(\ell(c, D))} \leq 1$, 
which implies that $\norm{\nabla h(\ell(c, D))} \leq 1+G$ for all $c, D$.

Now we proceed to do some setup in order to apply Theorem~\ref{thm:kkl}. Let
$\calW = \{\frac{2}{3}\}\times\left[-\frac{1}{3K}, \frac{1}{3K}\right]^K$. Note that
$\norm{\bw}\leq 1$ for all $\bw\in\calW$.
\footnote{As noted in~\citep{KakadeKL09, AlabiIK18}, 
even in the case where $\calW$ is
restricted to consist of only non-negative vectors, our arguments still follow 
through by replacing 
$\calW = \{\frac{2}{3}\}\times\left[-\frac{1}{3K}, \frac{1}{3K}\right]^K$
with
$\calW = \{\frac{2}{3}\}\times\left[0, \frac{1}{3K}\right]^K$.
}
Define
$\Phi(c, D) = (1, \ell(c, D))$ so that $\norm{\Phi(c, D)} \leq \sqrt{1+K}$ and
$\Phi(c, D)\cdot\bw \leq 1$ for all $c\in\calC_P$ and datasets $D$. As required by~\cite{KakadeKL09},
we assume that $\ell(\calC_P, D)$ is compact so that $\Phi(\calC_P, D)$ is also compact.

We have to convert the approximate linear optimizer oracle into a $\beta$-approximation
algorithm $A:\calW\rightarrow\calC_P$. Define $A(\bw) = A(2/3, \bw') =
\lopt(\calC_P, \ell, \bw, D, \tau)$ where $\bw'\in\reals^K$
are the last $K$ coordinates of $\bw\in\calW$.
Now we use that
$\Phi(c, D)\cdot\bw = \frac{2}{3} + \ell(c, D)\cdot\bw' \geq \frac{1}{3}$ for any $\bw\in\calW$ to
conclude that for any dataset $D$,
$$
\Phi(A(w), D)\cdot\bw \leq \min_{c\in\calC_P}\left(\frac{2}{3} + \ell(c, D)\cdot\bw'\right)
+ \tau\norm{\bw'} \leq (1+3\tau\norm{\bw'})\min_{c\in\calC_P}\left(\frac{2}{3} + \ell(c, D)\cdot\bw'\right).
$$

And note that since $\norm{\bw'}\leq 1/3$, $A$ is a $\beta$ approximation algorithm where
$\beta = 1 + \tau$.

Now we can apply Algorithm 3.1 of~\cite{KakadeKL09} to the following sequence:
$\bw_1 = (\frac{1}{3}, 0, \ldots, 0)$,
$\bw_{t+1} = (\frac{2}{3}, \frac{\nabla h(\ell(c_t, D))}{3K(1+G)})$ where
$c_t$ is the decision output in the $t$-th iteration of Algorithm 3.1 in~\citep{KakadeKL09}. In iteration
1, $c_1$ is chosen arbitrarily. Note that for all $t\in[T], \bw_t\in\calW$.
Then we output $\hat{c}=\Unif(\{c_1, \ldots, c_T\})$. If $c^*$ is the best decision in $\calC_P$, by
Theorem~\ref{thm:kkl} we have
$$\frac{1}{T}\sum_{t=1}^T\Phi(c_t, D)\cdot\bw_t \leq (\beta + 2)\sqrt{\frac{1+K}{T}} + \beta\frac{1}{T} \sum_{t=1}^T\Phi(c^*, D)\cdot\bw_t.$$
And since $\frac{1}{T}\sum_{t=1}^T\Phi(c^*, D)\cdot\bw_t \leq 1$ and $(1+K)/T \leq 1$ we have that
$$
\frac{1}{T}\sum_{t=1}^T(\Phi(c_t, D) - \Phi(c^*, D))\cdot\bw_t
\leq (\beta + 2)\sqrt{\frac{1+K}{T}} + \beta - 1
= (3+\tau)\sqrt{\frac{1+K}{T}} + \tau.
$$

Then by the convexity of $f$ and the definitions of $\Phi$ and $\bw_t$ we have
\begin{align}
    \frac{1}{T}\sum_{t=1}^T(\Phi(c_t, D) - \Phi(c^*, D))\cdot\bw_t
    &= \frac{1}{3K(1+G)T}\sum_{t=1}^T(\ell(c_t, D) - \ell(c^*, D))\cdot \nabla h(\ell(c_t, D)) \\
    &\geq \frac{1}{3K(1+G)T}\sum_{t=1}^Th(\ell(c_t, D)) - h(\ell(c^*, D))
\end{align}
so that for $\tau = \frac{\alpha}{6K(1+G)}$ and
$T \geq \frac{36(1+K)K^2(1+G)^2(3+\alpha)^2}{\alpha^2}$ we have
$$
\E[h(\ell(\hat{c}, D))] - h(\ell(c^*, D)) \leq 3K(1+G)\left((3+\tau)\sqrt{\frac{1+K}{T}} + \tau\right) \leq \alpha.
$$

By Markov's inequality we have that with probability at least 9/10,
$h(\ell(\hat{c}, D)) - h(\ell(c^*, D)) \leq \alpha$ after $T = O(K^3(1+G)^2)$ iterations. Then by Corollary~\ref{cor:h}, we can set
$G = \frac{\alpha + \sqrt{K}}{\alpha}$ and obtain that
after $T = O\left(\frac{K^4}{\alpha^2}\right)$ iterations,
$f(\ell(\hat{c}, D)) - f(\ell(c^*, D)) \leq \alpha$ and
$g(\ell(\hat{c}, D)) \leq \alpha$.

\end{proof}

\begin{theorem}[Restatement of Theorem 3.2 in~\citep{KakadeKL09}]
Consider a $(K+1)$-dimensional online linear optimization problem with feasible
set $\calC_P$ and mapping $\Phi: \calC_P\times(\calX\times\calA\times\calY)^n\rightarrow\reals^{K+1}$. 
Let $A$ be an $\beta$-approximation algorithm and take $R, W\geq 0$ such that
$\norm{\Phi(A(\bw), D)} \leq R$ and $\norm{\bw} \leq W$ for all $\bw\in\calW$.

For any $\bw_1, \bw_2, \ldots, \bw_T\in\calW$ and any $T\geq 1$ with learning parameter $\frac{(\beta + 1)R}{W\sqrt{T}}$, approximate projection
tolerance parameter $\frac{(\beta+1)R^2}{T}$, and learning rate
parameter $\frac{(\beta + 1)}{4(\beta+2)^2T}$, Algorithm 3.1 in~\citep{KakadeKL09} achieves
expected $\beta$-regret of at most
$$
\E\left[\frac{1}{T}\sum_{t=1}^Th(c_t, \bw_t)\right] - \beta\min_{c\in\calC_P}\frac{1}{T}\sum_{t=1}^Th(c, \bw_t) \leq \frac{(\beta + 2)RW}{\sqrt{T}}.
$$
where $h:\calC_P\times\calW\rightarrow[0, 1]$ is the cost function defined as
$h(c, \bw) = \Phi(c, D)\cdot\bw$ for any dataset 
$D = \{(x_i, a_i, y_i\}_{i=1}^n\in(\calX\times\calA\times\calY)^n$,
$\bw\in\calW, c\in\calC_P$.

On each period, Algorithm 3.1 in~\citep{KakadeKL09} makes at most
$4(\beta + 2)^2T$ calls to $A$ and $\Phi$. The algorithm also handles the case
where $\calW$ is restricted to contain only non-negative vectors.

\label{thm:kkl}
\end{theorem}

\begin{remark}
Note that all we require out of the use of Theorem~\ref{thm:kkl} is a no-regret optimization
algorithm that can use an approximation algorithm
(in our case, an approximate linear optimizer).
We have chosen to use~\citep{KakadeKL09} but could have used other alternatives that achieve
the same result~\citep{KalaiV03, Hazan16}.
\end{remark}

To use Theorem~\ref{thm:kkl} to minimize any convex function $h:[0, 1]^K\rightarrow\reals$ with
$\norm{\nabla h(\ell(c, D))}\leq (1+G)$ (for all $c\in\calC_P$ and dataset $D$), we will set
$\Phi(c, D) = (1, \ell(c, D))$ and 
$\calW = \{\frac{2}{3}\}\times\left[-\frac{1}{3K}, \frac{1}{3K}\right]^K$
(or $\calW = \{\frac{2}{3}\}\times\left[0, \frac{1}{3K}\right]^K$)
where in each iteration $t \geq 2$, $c_t$ will be chosen by Algorithm 3.1
in~\citep{KakadeKL09} and $\bw_t$ will be $(\frac{2}{3}, \frac{\nabla h(\ell(c, D))}{3K(1+G)})$.
Note that when $T = O(\frac{1}{\beta^2})$, Algorithm 3.1 in~\citep{KakadeKL09}
makes at most $O(1)$ calls to the approximation algorithm (our linear optimization
oracle in this case) in each period.
The crux of the use of Theorem~\ref{thm:kkl} in this paper is to
translate LOPT (Definition~\ref{def:lopt}) with additive error $\tau$ into a
$(1+\tau)$-multiplicative approximation algorithm and then directly apply 
Theorem~\ref{thm:kkl}.

\subsection{Linear Optimization Based Algorithm with Privacy}

In this section we show that there exists
an $(\eps, \delta)$-differentially private algorithm
for the constrained group-objective optimization problem.
Given a large-enough
sample of size $n$, this algorithm
will produce empirical risk bounds that go to 0 as 
$n\rightarrow\infty$. 

In the previous section, we assumed access to
an approximate
linear optimization oracle to incrementally solve our overall convex problem.
Inspired by this approach,
we will first assume access to a differentially private version of this oracle
$\lopt_{\eps, \delta}$
\footnote{For the private algorithms provided in~\citep{Jagielski18}, a differentially private
cost-sensitive classification oracle is assumed.} and subsequently provide an implementation
of this private oracle based on the exponential mechanism.

Algorithm~\ref{alg:poracle} is a differentially private algorithm for solving the
constrained group-objective optimization problem by replacing the non-private linear optimizer
oracle in Algorithm~\ref{alg:cgofwcs} with a private version.

\begin{algorithm}

\KwIn{$\lopt_{\eps', \delta'}, T, \ell, \nabla f, \nabla g, D \in (\calX\times\calA\times\calY)^n, G, \tau, \eps, \delta$}

\

Arbitrarily select decision $c\in\calC_P$ as $\tilde{c}_1$

\

\If{$\lopt_{\eps', \delta'}$ only for pure DP} {
    $\delta' = 0, \eps' = \frac{\eps}{2\sqrt{2T\log(1/\delta)}}$
} \Else {
    $\delta' = \frac{\delta}{2T}$, $\eps' = \frac{\eps}{2\sqrt{2T\log(2/\delta')}}$
}

\

\For {$t=1, \ldots, T-1$} {

  $\vec{r}_t(\tilde{c}_t, D) = \nabla f(\ell(c, D)) +
  \ind[g(\ell(c, D)) \geq 0]\nabla g(\ell(c, D))$

  $\tilde{c}_{t+1} = \lopt_{\eps', \delta'}\left(\calC_P, \ell, \vec{r}_t(\tilde{c}_t, D), D, \tau\right)$

}

\

\Return $\tilde{c}=\Unif(\{\tilde{c}_1, \ldots, \tilde{c}_T\})$
\caption{$(\eps, \delta)$-private algorithm using $\lopt_{\eps', \delta'}$ oracle.}
\label{alg:poracle}
\end{algorithm}

\begin{lemma}
For privacy parameters $\eps, \delta\in (0, 1]$, Algorithm~\ref{alg:poracle} is $(\eps, \delta)$-differentially private.
\label{lem:poracle}
\end{lemma}

\begin{proof}
The proof of privacy follows from the advanced composition result (see Lemma~\ref{lem:ac}) since
if $\delta > 0$ we set $\eps' = \frac{\eps}{2\sqrt{2T\log(1/\delta')}}$ where
$\delta' = \frac{\delta}{2T}$ or can set $\eps' = \frac{\eps}{2\sqrt{2T\log(2/\delta)}}$.
Then since in each iteration $t\in[T]$ we satisfy
$(\eps', \delta')$-differential privacy, we must have that the overall algorithm is
$(\eps, \delta)$-differentially private.

\end{proof}

Algorithm~\ref{alg:poracle} is an oracle-efficient algorithm that relies
on access to $\lopt_{\eps', \delta'}$.
$\nabla f(\ell(c, D)), \nabla g(\ell(c, D))$ are $K\times 1$ column vectors representing the gradients of $f$ and $g$, respectively.
These quantities are used to compute $\nabla h(\ell(c, D))$, fed as
a weight vector to $\lopt_{\eps', \delta'}$.
Assuming such an oracle has
the same utility guarantees as its non-private counterpart, we obtain
the utility guarantees of Theorem~\ref{thm:paik}. In 
Theorem~\ref{thm:loptexp}, we provide a generic 
implementation of such a private
oracle based on exponential sampling and provide utility guarantees
for this implementation.

\begin{theorem}

Suppose we are given convex 1-Lipschitz functions $f, g: [0, 1]^K\rightarrow\reals$ and
loss function $\ell:\calC_P\times(\calX\times\calA\times\calY)^n\rightarrow[0, 1]^K$.
Given access to a differentially private 
approximate linear optimizer oracle
$\lopt_{\eps, \delta}$ (Definition~\ref{def:loptp}), after $T = O(\frac{K^4}{\alpha^2})$ calls to 
$\lopt_{\eps, \delta}$, with probability at least 9/10,
we will obtain a decision $\tilde{c}\in\calC_P$ with the following guarantee:
$$f(\ell(\tilde{c}, D)) \leq f(\bloss) + \alpha,
\quad\quad
g(\ell(\tilde{c}, D)) \leq \alpha,$$

for any $\alpha\in (0, 1]$ and privacy parameters $\eps, \delta\in (0, 1]$
where $c^* \in \argmin_{c\in\calC_P:g(\loss)\leq 0}f(\loss)$ is the best decision
in the feasible set $\calC_P$, given dataset $D$ of size $n$ such that $\ell(\calC_P, D)\subset[0, 1]^K$ is compact.
\label{thm:paik}
\end{theorem}

\begin{proof}
This theorem follows from the privacy proof of Lemma~\ref{lem:poracle}
for Algorithm~\ref{alg:poracle} and the utility guarantees of the
non-private $\lopt$ oracle given in Theorem~\ref{thm:aik}.
\end{proof}

Now, we proceed to show the existence of a $\lopt_{\eps, \delta}$ oracle
based on the exponential mechanism. This is a generic
implementation of such a private oracle that can be used to
solve the constrained group-objective optimization problem. The oracle is efficient when $\ell$ is convex
and $\calW$ consists of only non-negative vectors.

\begin{theorem}

For any privacy parameter $\eps > 0$, there is an implementation of the $\lopt_{\eps, 0}$
oracle (Definition~\ref{def:loptp}) based on the exponential mechanism.

For any $\tau > 0, \theta\in(0, 1]$, if $\ell$ is convex and $\calW$ is restricted to only
non-negative vectors, set
$n_0 = \tilde{O}(\frac{\sqrt{K}}{\eps\tau}(P + \log\frac{1}{\theta}))$ and if not
set $n_0 = \tilde{O}(\frac{\sqrt{K}}{\eps\tau}(\VC(\calC_P) + \log\frac{1}{\theta}))$.
Then there exists $n_0$ such that 
for all $n\geq n_0$ and for any fixed $\bw\in\calW$, if
$\tilde{c} = \lopt^\theta_{\eps, 0}(\calC_P, \ell, \bw, D, \tau)$ then
we have the following utility guarantee:
$$
\pr\left[\bw\cdot\ell(\tilde{c}, D) \leq \min_{c\in\calC_P}\bw\cdot \ell(c, D) + \tau\norm{\bw}\right] \geq 1-\theta.
$$

\label{thm:loptexp}
\end{theorem}

\begin{proof}

First, let us consider the case where $\ell$ is convex and $\calW$ only has
non-negative vectors. Then this
result follows from the use of the $(\eps, 0)$-differentially private exponential sampling convex optimization algorithm.

By Theorem 3.2 in~\citep{BassilyST14}, 
we have that for a fixed non-negative $\bw\in\calW$ and dataset $D$ and for all
$a > 0$, we have
$$
\pr\left[\bw\cdot \ell(\tilde{c}, D) - \bw\cdot\ell(c^*, D) \geq \frac{8\norm{\bw}\sqrt{K}}{\eps n}((P+1)\log 3 + a)\right] \leq e^{-a},
$$
where $c^* \in \argmin_{c\in\calC_P}\loss$ 
since the sensitivity of $\bw\cdot \ell(\tilde{c}, D)$ is at most
$\norm{\bw}\frac{\sqrt{K}}{n}$ by Cauchy-Schwarz (Lemma~\ref{lem:cs}) and
$\ell(\tilde{c}, D) = \frac{1}{n}\sum_{i=1}^n\ell(\tilde{c}, D_i)$.

Rearranging the terms, we get that when 
$n_0 = \frac{8\sqrt{K}}{\eps\tau}((P+1)\log 3 + \log\frac{1}{\theta})$ and for any larger sizes, 
we get the desired guarantees.
If $\ell$ is not convex, we rely on the generic utility guarantees of the exponential mechanism
(see Theorem~\ref{thm:exputil}). By Sauer's Lemma, the range of the exponential
mechanism is bounded by $|\calC_P(D_\calX)|\leq O(n^{\VC(\calC_P)})$ where $\VC(\calC_P)$ is
the VC dimension of $\calC_P$.

\end{proof}

Armed with the construction of $\lopt_{\eps, 0}$ based on the exponential mechanism, we
proceed to show Corollary~\ref{cor:loptexp}.

\begin{corollary}
Suppose we are given convex 1-Lipschitz functions $f, g: [0, 1]^K\rightarrow\reals$,
loss function $\ell:\calC_P\times(\calX\times\calA\times\calY)^n\rightarrow[0, 1]^K$, and
$\calC_P$ (with finite VC dimension $\VC(\calC_P)$ and resulting parameter
space in $\reals^P$).

If $\ell$ is convex and $f, g$ are non-decreasing, set $n_0 = \tilde{O}\left(\frac{K^4P}{\eps\alpha^3}\right)$. If not,
set $n_0 = \tilde{O}\left(\frac{K^4\cdot\VC(\calC_P)}{\eps\alpha^3}\right)$.
For any privacy parameters $\eps, \delta\in(0, 1]$, given
access to an exponential mechanism based differentially private oracle
$\lopt^\theta_{\eps, 0}$ (Definition~\ref{def:loptp}), 
there exists an $(\eps, \delta)$-differentially private algorithm and an $n_0$
such that for all $n\geq n_0$,
with probability at least 9/10,
we will obtain a decision $\tilde{c}\in\calC_P$ with the following guarantee:
$$f(\ell(\tilde{c}, D)) \leq f(\bloss) + \alpha,
\quad\quad
g(\ell(\tilde{c}, D)) \leq \alpha,$$

for any $\alpha\in (0, 1]$ and privacy parameters $\eps, \delta\in (0, 1]$
where $c^* \in \argmin_{c\in\calC_P:g(\loss)\leq 0}f(\loss)$ is the best decision
in the convex feasible set $\calC_P$, 
given dataset $D$ of size $n$.
\label{cor:loptexp}
\end{corollary}

\begin{proof}
This follows from Lemma~\ref{thm:loptexp} and the use of composition
in Algorithm~\ref{alg:poracle}.

First, let us consider the case where $\ell$ is convex and $f, g$ are
non-decreasing. By Lemma~\ref{thm:loptexp}, we could set
$\tau = \tilde{O}(\frac{\sqrt{K}}{\eps n}(P+\log 10))$ since
$\nabla f, \nabla g$ will be non-negative vectors.
By the union bound and the use of
advanced composition in Algorithm~\ref{alg:poracle}, we can set
$\tau = \tilde{O}(\frac{\sqrt{K}}{\eps n}\sqrt{T\log(1/\delta)}(P+\log 10T))$. Then by
Theorem~\ref{thm:aik}, we could set $\tau = \frac{\alpha}{6K(1+G)} = O(\frac{\alpha^2}{K\sqrt{K}})$
where $G = O(\frac{\sqrt{K}}{\alpha})$.
Equating these two, we get that $n = \tilde{O}(\frac{K^4}{\eps\alpha^3}P)$
if we set $T = O(\frac{K^4}{\alpha^2})$ as done for Theorem~\ref{thm:aik}.

If $\ell$ is not convex, then we essentially replace $O(P)$ with
$\tilde{O}(\VC(\calC_P))$ and rely on the generic utility guarantees of the
exponential mechanism.
This completes the proof.

\end{proof}

In later sections, we will show that the sample complexity to solve
constrained group objective optimization is lower-bounded by 
$n = \Omega\left(\frac{\sqrt{K}}{\eps\alpha}\right)$ for (approximate) $(\eps, \delta)$-differential privacy (with probability at least $1/3$).

\section{Lower Bounds for Private Constrained Group-Objective Optimization}

We now proceed to show excess risk lower bounds for private constrained group-objective
optimization. Note that since these bounds are a function of the
dataset size $n$, these results are equivalent to a lower bound on the
sample complexity required to solve the problem.

We ask: 
\textit{
over the randomness of any $(\eps, 0)$ or $(\eps, \delta)$-differentially
private mechanism, for a fixed dataset $D$ of size $n$,
what is a lower bound for the accuracy of the mechanism that solves the constrained group-objective
optimization problem?
}

We show a lower bound on the excess risk for decision set
$\calC_P = B^K_2 = \{c\in\reals^K \,\,:\,\, \norm{c}_2 = 1\}$ assuming the dataset is also drawn from 
$B^K_2$. That is, we consider the case where 
the decisions and datasets lie in the unit ball with $\ell_2$ norm.
We show that for all $n, K\in\mathbb{N}$ and $\eps > 0$
there exists a dataset $D = \{x_i\}_{i=1}^n\subseteq B^K_2$ for which there is a
constrained group-objective optimization problem with functions $f, g$ such that
both $f$ and $g$ will have excess risk lower bounds of 
$\alpha \geq \Omega(\frac{K}{\eps n})$ and
$\alpha \geq \Omega(\frac{\sqrt{K}}{\eps n})$
for any $(\eps, 0)$, $(\eps, \delta)$-differentially private algorithms respectively.

\subsection{$(\eps, 0)$ Lower Bound}

\begin{theorem}

Let $n, K\in\mathbb{N}, \eps > 0$ and $\alpha\in(0, 1]$. For every $\eps$-differentially private algorithm $\calM$
that produces a decision $\hat{c}\in\calC_P$ such
that
$$
f(\ell(\hat{c}, D)) \leq \min_{c\in\calC_P:g(\ell(c, D))\leq 0}f(\ell(c, D)) + \alpha,
\quad\quad g(\ell(\hat{c}, D)) \leq \alpha,
$$
there is a dataset $D = \{\bx_1, \ldots, \bx_n\}\subseteq B^K_2$ such that, with probability at least 1/2,
we must have $\alpha \geq \Omega\left(\frac{K}{\eps n}\right)$
(or equivalently, $n \geq \Omega\left(\frac{K}{\eps\alpha}\right)$)
where $f, g$ are Lipschitz, smooth functions defined as follows:
$$
f(\ell(c, D)) = -\frac{1}{n}\sum_{i=1}^n\brackets{c, \bx_i},\quad\quad
g(\ell(c, D)) = f(\ell(c, D)) + \frac{1}{n}\norm{\sum_{i=1}^n\bx_i}.
$$
for all $c\in B^K_2$.

\label{thm:lowerpure}
\end{theorem}

\begin{proof}
The major idea in the proof is to reduce to the problem of optimizing 1-way marginals (a standard
method for lower bounding the accuracy of differentially private mechanisms).

We have defined $f$ as $f(\ell(c, D)) = -\frac{1}{n}\sum_{i=1}^n\brackets{c, \bx_i}$ which has minimum
$c^* = \frac{\sum_{i=1}^n\bx_i}{\norm{\sum_{i=1}^n\bx_i}}$ by Lemma~\ref{lem:minlower}.
We defined $g$ as $g(\ell(c, D)) = f(\ell(c, D)) + \frac{1}{n}\norm{\sum_{i=1}^n\bx_i} = f(\ell(c, D)) - f(\ell(c^*, D))$
which has minimum $c^*$ so that
the constraint $g(\ell(c^*, D)) \leq 0$ is satisfied.

Now by Lemma~\ref{lem:reducetomarginals}, we have that
$f(\ell(c, D)) - f(\ell(c^*, D)) = \frac{\norm{\sum_{i=1}^n\bx_i}}{2n}\norm{c - c^*}^2$. Now we invoke
Lemma~\ref{lem:marginalspure}.
If $\hat{c}$ is the output of any $\eps$-differentially private mechanism $\calM$ then we must have
that $\norm{c-c^*} = \Omega(1)$. Suppose not. Then that would imply that we can construct a new mechanism
$\calM'$ that outputs $\hat{c}\cdot\frac{\norm{\sum_{i=1}^n\bx_i}}{n}$ which would contradict Lemma~\ref{lem:marginalspure}.
As a result, $\norm{c-c^*} = \Omega(1)$ so that
$f(\ell(\hat{c}, D)) - f(\ell(c^*, D)) = \Omega(\frac{K}{\eps n})$ for the output $\hat{c}$ of any $\eps$
differentially private mechanism.

\end{proof}

\subsection{$(\eps, \delta)$ Lower Bound}

\begin{theorem}

Let $n, K\in\mathbb{N}, \eps > 0, \alpha\in(0, 1],$ and $\delta = o(\frac{1}{n})$. For every $(\eps, \delta)$-differentially private algorithm $\calM$
that produces a decision $\hat{c}\in\calC_P$ such
that
$$
f(\ell(\hat{c}, D)) \leq \min_{c\in\calC_P:g(\ell(c, D))\leq 0}f(\ell(c, D)) + \alpha,
\quad\quad g(\ell(\hat{c}, D)) \leq \alpha,
$$
there is a dataset $D = \{\bx_1, \ldots, \bx_n\}\subseteq B^K_2$ such that, with probability at least 1/3,
we must have $\alpha \geq \Omega\left(\frac{\sqrt{K}}{\eps n}\right)$
(or equivalently, $n \geq \Omega\left(\frac{\sqrt{K}}{\eps\alpha}\right)$)
where $f, g$ are Lipschitz, smooth functions defined as follows:
$$
f(\ell(c, D)) = -\frac{1}{n}\sum_{i=1}^n\brackets{c, \bx_i},\quad\quad
g(\ell(c, D)) = f(\ell(c, D)) + \frac{1}{n}\norm{\sum_{i=1}^n\bx_i}.
$$
for all $c\in B^K_2$.

\label{thm:lowerapprox}
\end{theorem}

\begin{proof}

We follow the steps of the proof for Theorem~\ref{thm:lowerpure} but invoke the lower bound for
1-way marginals in the approximate differential privacy case (and not the pure case).

Again, the way we have defined $f$, by Lemma~\ref{lem:reducetomarginals}, we have that
$f(\ell(c, D))) - f(\ell(c^*, D)) = \frac{\norm{\sum_{i=1}^n\bx_i}}{2n}\norm{c - c^*}^2$. Now we invoke
Lemma~\ref{lem:marginalsapprox}.

If $\hat{c}$ is the output of any $(\eps, \delta)$-differentially
private mechanism $\calM$ then we must have
that $\norm{c-c^*} = \Omega(1)$. Suppose not. Then that would imply that we can construct a new mechanism
$\calM'$ that outputs $\hat{c}\cdot\frac{\norm{\sum_{i=1}^n\bx_i}}{n}$ which would contradict Lemma~\ref{lem:marginalsapprox}.
As a result, $\norm{c-c^*} = \Omega(1)$ so that
$f(\ell(\hat{c}, D)) - f(\ell(c^*, D)) = \Omega(\frac{\sqrt{K}}{\eps n})$ for the output $\hat{c}$ of any
$(\eps, \delta)$-differentially private mechanism.

\end{proof}

\subsection{Helper Lemmas}

\begin{lemma}
Let $c^* = \argmin_{c:\norm{c}\geq 1}-\frac{1}{n}\sum_{i=1}^n\brackets{c, \bx_i}$ where
$\bx_i\in B^K_2$ for all $i\in[n]$, then
$c^* = \frac{\sum_{i=1}^n\bx_i}{\norm{\sum_{i=1}^n\bx_i}}$.
\label{lem:minlower}
\end{lemma}

\begin{proof}
Note that for any $c\in B^K_2$ we have $|\brackets{c, \sum_{i=1}^n \bx_i}|\leq \norm{c}\norm{\sum_{i=1}^n \bx_i} = \norm{\sum_{i=1}^n \bx_i}$
by Cauchy-Schwarz and this is tight when $c = \frac{\sum_{i=1}^n\bx_i}{\norm{\sum_{i=1}^n\bx_i}}$ or
$c = -\frac{\sum_{i=1}^n\bx_i}{\norm{\sum_{i=1}^n\bx_i}}$. As a result, the minimum of $-\frac{1}{n}\sum_{i=1}^n\brackets{c, \bx_i}$ is attained at $c = \frac{\sum_{i=1}^n\bx_i}{\norm{\sum_{i=1}^n\bx_i}}$.

\end{proof}

\begin{lemma}
Let $f(\ell(c, D)) = -\frac{1}{n}\sum_{i=1}^n\brackets{c, \bx_i}$ where $c, \bx_i\in B^K_2$ for all $i\in[n]$, then
$$f(\ell(c, D)) - f(\ell(c^*, D)) = \frac{\norm{\sum_{i=1}^n\bx_i}}{2n}\norm{c - c^*}^2$$ for any $c\in B^K_2$ and
$c^* = \frac{\sum_{i=1}^n\bx_i}{\norm{\sum_{i=1}^n\bx_i}}$.
\label{lem:reducetomarginals}
\end{lemma}

\begin{proof}
We have that
\begin{align}
    f(\ell(c, D)) - f(\ell(c^*, D)) &= \frac{1}{n}\sum_{i=1}^n\left(\brackets{c^*, \bx_i} - \brackets{c, \bx_i}\right) \\
    &= \frac{1}{n}\left(\brackets{c^*, \sum_{i=1}^n \bx_i} - \brackets{c, \sum_{i=1}^n \bx_i}\right)  \\
    &= \frac{1}{n}\left(\norm{\sum_{i=1}^n \bx_i} - \brackets{c, \sum_{i=1}^n \bx_i}\right) \\
    &= \frac{\norm{\sum_{i=1}^n \bx_i}}{n}\left(1-\brackets{c, c^*}\right)\\
    &= \frac{\norm{\sum_{i=1}^n \bx_i}}{2n}\norm{c-c^*}^2
\end{align}
where we have used that
$\norm{c-c^*}^2 = \norm{c}^2 + \norm{c^*}^2 - 2\brackets{c, c^*} = 2 - 2\brackets{c, c^*}$ 
and 
$c^* = \frac{\sum_{i=1}^n\bx_i}{\norm{\sum_{i=1}^n\bx_i}}$. 
\end{proof}

We now state lower bound lemmas for 1-way marginals. 
Lemma~\ref{lem:marginalspure} shows the lower bound for 1-way marginals for $\eps$-differentially
private algorithms and Lemma~\ref{lem:marginalsapprox} is for $(\eps, \delta)$-differentially
private algorithms.

\begin{lemma}[Part 1 of Lemma 5.1 in~\citep{BassilyST14}]
Let $n, K\in\mathbb{N}$ and $\eps > 0$. There exists a number 
$M = \Omega(\min(n, \frac{K}{\eps}))$ such that
for every $\eps$-differentially private algorithm $\calM$ there is a dataset
$D = \{\bx_1, \ldots, \bx_n\}\subseteq B^K_2$ with 
$\norm{\sum_{i=1}^n\bx_i}_2\in[M-1, M+1]$ such that, with probability at least 1/2 (over the randomness
of the algorithm), we have
$$
\norm{\calM(D) - q(D)}_2 = \Omega(\min(1, \frac{K}{\eps n})),
$$
where $q(D) = \frac{1}{n}\sum_{i=1}^n\bx_i$.
\label{lem:marginalspure}
\end{lemma}

\begin{lemma}[Part 2 of Lemma 5.1 in~\citep{BassilyST14}]
Let $n, K\in\mathbb{N}$, $\eps > 0$, and $\delta = o(\frac{1}{n})$. There is a number
$M = \Omega(\min(n, \frac{\sqrt{K}}{\eps}))$ such that for every
$(\eps, \delta)$-differentially private algorithm $\calM$, there is a dataset
$D = \{\bx_1, \ldots, \bx_n\}\subseteq B^K_2$ with
$\norm{\sum_{i=1}^n\bx_i}_2\in[M-1, M+1]$ such that, with probability at least 1/3 (over the randomness
of the algorithm), we have
$$
\norm{\calM(D) - q(D)}_2 = \Omega(\min(1, \frac{\sqrt{K}}{\eps n})),
$$
where $q(D) = \frac{1}{n}\sum_{i=1}^n\bx_i$.
\label{lem:marginalsapprox}
\end{lemma}

\section{Bounded Divergence Linear Optimizers}
\label{sec:divergence}

We introduce a class of bounded divergence linear optimizers.~\cite{CuffY16} explore various
definitions of differential privacy through the lens of mutual information constraints.
In a similar
vein, we introduce some information-theoretic definitions of linear optimizers based on
the divergence between two random variables.

These oracles can be used in multi-objective applications that do not necessarily apply to
algorithmic fairness.
For example,~\cite{BallBDDRVV20} show how to lift ``hardness'' through bounded mutual information
reductions (i.e., potentially lossy reductions). In some cases, these reductions might need to
optimize more than one loss function or constraint (e.g., optimizing both language or code length
and the regularity of code words).

It is known that $\eps$-differential privacy can be cast as a max divergence
bound. Similar to how min-entropy is a worst-case analog of Shannon Entropy, the
max divergence is a worst-case analog of KL-divergence~\citep{Vadhan17}.
In fact, 
it turns out that most relaxations of differential privacy can be cast as a bound on an
information-theoretic divergence. We use this insight to provide the following
definitions for bounded divergence linear optimizer oracles 
($\lopt_{\eps, 0}, \lopt_{\eps, \delta}, \rlopt_{\eps, \phi}$).
The randomness is over the coin flips of these oracles.

%\subsection{$\eps$-Differential Privacy Linear Optimizers}

\begin{definition}[$\lopt_{\eps, 0}$]
Let $\calW\subseteq\reals^K$ (or $\calW\subseteq\reals^K_{\geq 0}$) be a set of weight
vectors. Then for any weight vector $\bw\in\calW$ and for
all $\bz, \bz'\in\calZ^n$ that differ in one row:
\begin{enumerate}
\item If $\tilde{c} = \lopt_{\eps, 0}(\calC_P, \ell, \bw, \bz, \tau)$, then
$\bw\cdot \ell(\tilde{c}, \bz) \leq \min_{c\in\calC_P}\bw\cdot\ell(c, \bz) + \tau\norm{\bw}$,
\item $D_\infty(\lopt_{\eps, 0}(\calC_P, \ell, \bw, \bz, \tau)\;\;\Vert\;\;\lopt_{\eps, 0}(\calC_P, \ell, \bw, \bz', \tau)) \leq \eps$,
\item $D_\infty(\lopt_{\eps, 0}(\calC_P, \ell, \bw, \bz', \tau)\;\;\Vert\;\;\lopt_{\eps, 0}(\calC_P, \ell, \bw, \bz, \tau)) \leq \eps$,
\end{enumerate}
where $D_\infty(Y\Vert Z) = \max_{S\subseteq \supp(Y)}\left[\ln
\frac{\pr[Y\in S]}{\pr[Z\in S]}\right]$ is the max divergence
between random variables $Y$ and $Z$ with the same support.
\label{def:purelopt}
\end{definition}

%\subsection{$(\eps, \delta)$-Differential Privacy Linear Optimizers}

\begin{definition}[$\lopt_{\eps, \delta}$]
Let $\calW\subseteq\reals^K$ (or $\calW\subseteq\reals^K_{\geq 0}$) be a set of weight
vectors. 
Then for any weight vector $\bw\in\calW$ and for
all $\bz, \bz'\in\calZ^n$ that differ in one row:
\begin{enumerate}
\item If $\tilde{c} = \lopt_{\eps, \delta}(\calC_P, \ell, \bw, \bz, \tau)$, then
$\bw\cdot \ell(\tilde{c}, \bz) \leq \min_{c\in\calC_P}\bw\cdot\ell(c, \bz) + \tau\norm{\bw}$,
\item $D_\infty^\delta(\lopt_{\eps, \delta}(\calC_P, \ell, \bw, \bz, \tau)\;\;\Vert\;\;\lopt_{\eps, \delta}(\calC_P, \ell, \bw, \bz', \tau)) \leq \eps$,
\item $D_\infty^\delta(\lopt_{\eps, \delta}(\calC_P, \ell, \bw, \bz', \tau)\;\;\Vert\;\;\lopt_{\eps, \delta}(\calC_P, \ell, \bw, \bz, \tau)) \leq \eps$,
\end{enumerate}
where $D_\infty^\delta(Y\Vert Z) = \max_{S\subseteq \supp(Y):\pr[Y\in S]\geq \delta}\left[\ln
\frac{\pr[Y\in S]-\delta}{\pr[Z\in S]}\right]$ is the 
smoothed max divergence
between random variables $Y$ and $Z$ with the same support.
\label{def:approxlopt}
\end{definition}

%\subsection{$\rho$-Zero-Concentrated Differential Privacy Linear Optimizers}

\begin{definition}[$\rlopt_{\eps, \phi}$]
Let $\calW\subseteq\reals^K$ (or $\calW\subseteq\reals^K_{\geq 0}$) be a set of weight
vectors. 
Then for any weight vector $\bw\in\calW$ and for
all $\bz, \bz'\in\calZ^n$ that differ in one row:
\begin{enumerate}
\item If $\tilde{c} = \rlopt_{\eps, \phi}(\calC_P, \ell, \bw, \bz, \tau)$, then
$\bw\cdot \ell(\tilde{c}, \bz) \leq \min_{c\in\calC_P}\bw\cdot\ell(c, \bz) + \tau\norm{\bw}$,
\item $D_\phi(\rlopt_{\eps, \phi}(\calC_P, \ell, \bw, \bz, \tau)\;\;\Vert\;\;\rlopt_{\eps, \phi}(\calC_P, \ell, \bw, \bz', \tau)) \leq \eps$,
\item $D_\phi(\rlopt_{\eps, \phi}(\calC_P, \ell, \bw, \bz', \tau)\;\;\Vert\;\;\rlopt_{\eps, \phi}(\calC_P, \ell, \bw, \bz, \tau)) \leq \eps$,
\end{enumerate}
where $D_\phi(Y\Vert Z)$ is the $\phi$-R\'enyi divergence of order $\phi > 1$
between random variables $Y$ and $Z$ defined as
$D_\phi(Y\Vert Z) = \frac{1}{\phi - 1}\ln\E_{x\sim Z}(\frac{Y(x)}{Z(x)})^\phi$.
\label{def:renyilopt}
\end{definition}

Definitions~\ref{def:purelopt} and~\ref{def:approxlopt} are approximate linear optimizers
that satisfy pure and approximate differential privacy respectively.
Definition~\ref{def:renyilopt} is an analog for R\'enyi differential privacy
~\citep{Mironov17}. As $\phi\rightarrow 1$, the R\'enyi divergence is equal to the
Kullback-Leibler divergence (relative entropy) and as $\phi\rightarrow\infty$, 
the R\'enyi divergence is the max-divergence.
\begin{remark}
$(\eps, \delta)$-differential privacy allows for use of advanced
composition and a tighter analyses for the composition of $(\eps, 0)$-differentially private
mechanisms. And the R\'enyi differential privacy, amongst many advantages,
allows for simpler analysis and use of the Gaussian Mechanism.
In this paper, we mainly use $\lopt_{\eps, 0}, \lopt_{\eps, \delta}$ for our results.
We can also extend this framework to handle general
$f$-divergences~\citep{SasonV16}.
\end{remark}

\section{Reductions Approach to Optimization and Learning}
\label{sec:reductions}

The reductions approach in machine learning 
~\citep{LangfordOZ06, BeygelzimerLZ09} has been widely studied 
and applied in different scenarios. 
Applications to ranking, regression,
classification, and importance-weighted classification are
particularly well-known. The crux of the reductions approach
to optimization and learning is to use the machinery -- both
theory and practice -- of solutions to one machine learning 
problem in order to solve another learning problem by 
reducing one problem to another.

A concrete example of the use of the reductions approach is by~\cite{AgarwalBD0W18}. 
The authors present a
systematic approach to reduce the problem of fair classification
to cost-sensitive classification problems. We will first review
applications of the reductions approach to optimization 
and learning and then
explain how to make this approach differentially private through
the linear optimization based algorithm presented in this paper.

Furthermore, we will focus on the problem of empirical
risk minimization where we are given a finite-sized
training sample from an unknown distribution and will optimize
with respect to this finite sample. Generalization guarantees
can be derived based on draws of a large enough sample from
the distribution (or knowledge of the complexity of the hypothesis class to be learned)
and knowledge of proportion of the population
belonging to a specific subgroup.
\footnote{
Which can also be estimated from a large enough sample drawn
from the distribution.
}
We do not focus on generalization in this paper
but rather on the problem of empirical risk minimization.

First, we discuss how the reductions approach can be applied to
optimize convex measures of the \textit{confusion matrix} and then
discuss how it can be applied to a few other definitions from the
algorithmic fairness literature.

\subsection{Convex Measures of Confusion Matrix}

\begin{definition}[Confusion Matrix]
The \textbf{confusion matrix} (sometimes referred to as the \textbf{contingency table})
$C^\mu[h]\in[0, 1]^{L\times L}$ of an
hypothesis $h$ with respect to a distribution $\mu$ over examples 
is
defined as
$$
C^\mu_{pq}[h] = C^\mu_{pq}\circ h = \pr_{x, y\sim\mu}\left[y = p, h(x) = q\right].
$$
We shall sometimes refer to $C^\mu[h]$ as $C[h]$ or $C$. $L$ is the number of possible
labellings that $h(x)$ or $y$ can be for any example $(x, y)\sim\mu$.
\label{def:cmatrix}
\end{definition}

Our algorithms in this paper to solve the problem of
constrained group-objective optimization assume that our functions $f$ and
$g$ are convex functions of the loss vectors. The performance
measures $G$-mean and $H$-mean are both concave functions of
the confusion matrix~\citep{Narasimhan18}. As a result, their negatives are convex.

\begin{example}[$G$-mean]
The \textbf{$G$-mean} performance measure is used to measure the quality
of both multi-class and binary classifiers in settings of 
severe class imbalance. It is defined as
$$
GMean(C) = \left(\prod_{i=1}^L\frac{C_{ii}}{\sum_{j=1}^L C_{ij}}\right)^{1/L},
$$
for some confusion matrix $C = C^\mu[h]$ defined on 
distribution $\mu$
and hypothesis $h$.
\label{ex:gmean}
\end{example}

\begin{example}[$H$-mean]
The \textbf{$H$-mean} performance measure is defined as
$$
HMean(C) = L\left(\sum_{i=1}^L\frac{\sum_{j=1}^L C_{j i}}{C_{i i}}\right)^{-1},
$$
for some confusion matrix $C = C^\mu[h]$ defined on
distribution $\mu$
and hypothesis $h$.
\label{ex:hmean}
\end{example}

Since we do not have access to the true confusion matrix
$C^\mu[h]$ which requires access to the distribution $\mu$ 
itself -- 
not just finite samples -- we must rely on empirical estimates
of $C^D[h]$ as follows:
$$
\hat{C}^D_{p q}[h] = \frac{1}{n}\sum_{i=1}^n\ind\left[y_i = p, h(x_i) = q\right],
$$
where $D\sim\mu^n$ is a finite sample of size $n$ and $h$ is an
hypothesis. We
term $\hat{C}^D$, the empirical confusion matrix.

Note that the empirical confusion matrix $\hat{C}^D[h]\in(0, 1]^{L\times L}$ 
can be written in
terms of constrained group-objective optimization as follows.
For a finite sample $D = \{(x_i, y_i)\}_{i=1}^n$, we define the
loss vector $\ell(h, D)\in(0, 1]^K$ where $K=L^2$ as
$$
\ell_k(h, D) =\frac{1}{n}\sum_{i=1}^n\ind[y_i = p\,\,\wedge\,\,h(x_i) = q],
\quad k = L\cdot (p-1) + q,
$$
for any $k\in [L^2]$ so that
$\hat{C}^D_{p q}[h]$ can be mapped to the specific entry $\ell_k(h, D)$.
In other words, for any $i\in[n]$ and hypothesis $h$, $(x_i, y_i)$ belongs to group $pq\in[L^2]$ iff
$\ind[y_i = p\,\,\wedge\,\,h(x_i) = q]$ is 1.

Note that the entries of
$\hat{C}^D$ are defined in terms of the 0-1 loss which is non-convex and thus hard
to optimize. As such, we could instead use
a ``smoothed'' versions of this loss. For example, the hinge loss is a convex surrogate loss and the
``smoothed'' hinge loss is a convex and smooth loss~\citep{Rennie05}.
Also, the $G$-mean, $H$-mean (as defined above) are concave measures so we can optimize with respect to their
negatives (which are convex).

In Figure~\ref{fig:means}, we show how the G-mean and H-mean
performance measures behave as the number of classes/labellings
$L$ increases. In Figure~\ref{fig:eq}, each entry of the confusion
matrix has the same weight i.e., $C_{ij} = C_{ji}$ for all $i, j\in[L]$.
In Figure~\ref{fig:un}, the weight in each entry decreases 
geometrically. Specifically, from one entry to the next, there is a
multiplicative decrease of a factor of $1/3$. In both cases, we require
that the sum of the entries of the confusion matrix is 1. In both
figures, the H-mean performance measure is larger than the G-mean.
We see from Figure~\ref{fig:eq} that the G-mean decreases faster than
the H-mean when the confusion matrix is balanced. From
Figure~\ref{fig:un}, we see that the H-mean increases faster than
the G-mean decreases when the confusion matrix is severely unbalanced.
\footnote{
These figures are a simple illustration of the behavior of H-mean
and G-mean and are not meant to provide conclusive evidence of
the behavior of these performance measures.
}

Using Corollary~\ref{cor:loptexp}, one could, for example, minimize average error subject to $G$-mean or
$H$-mean constraints.

\begin{figure}
  \begin{subfigure}[b]{0.4\textwidth}
    \includegraphics[width=\textwidth]{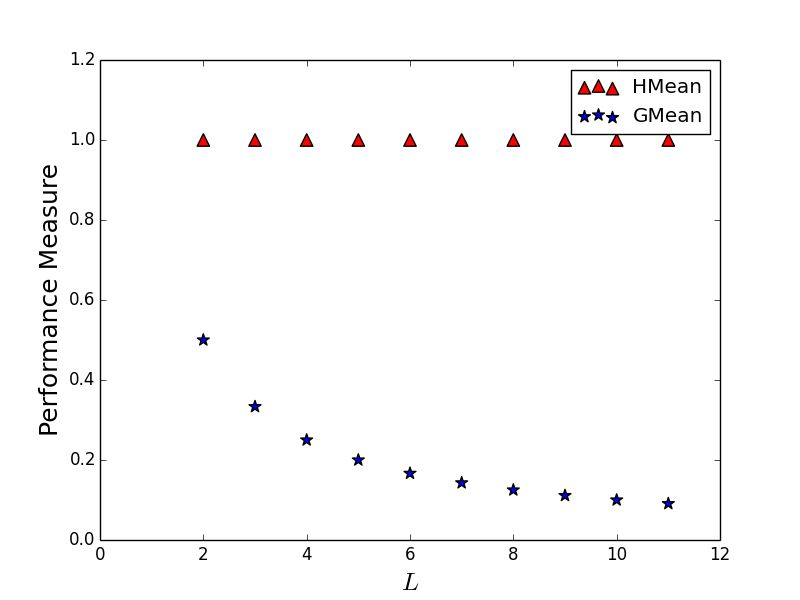}
    \caption{Balanced confusion matrix weights.}\label{fig:eq}
  \end{subfigure}
  \qquad
  \begin{subfigure}[b]{0.4\textwidth}
    \includegraphics[width=\textwidth]{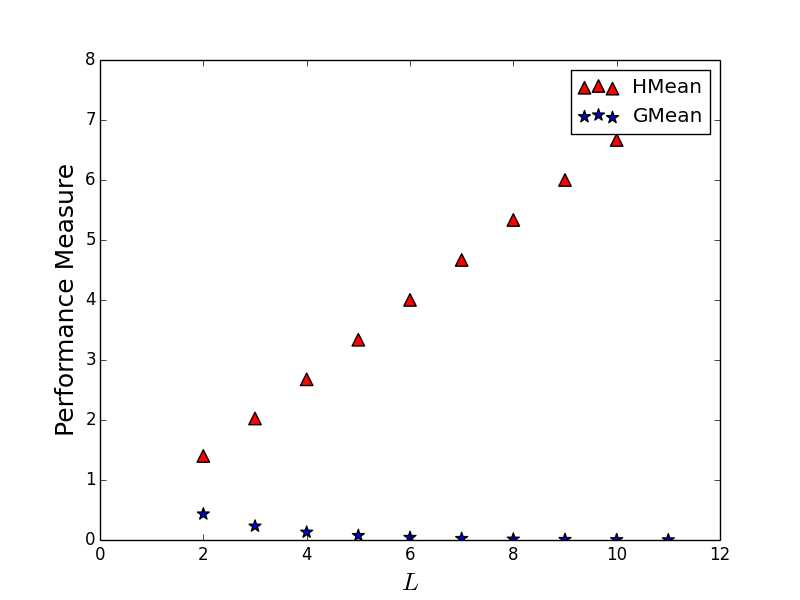}
    \caption{Severely unbalanced confusion matrix weights.}\label{fig:un}
  \end{subfigure}
  \centering
  \caption{Illustration of H-mean and G-mean performance as the number of classes increases.}
  \label{fig:means}
\end{figure}

\subsection{Algorithmic Fairness Definitions}

In this section, we discuss how certain statistical definitions from the
algorithmic fairness literature can be written in terms of
constrained group-objective optimization. The first two -- equalized odds
and demographic parity -- are used for ensuring some notion of
fairness in classification and the third -- Gini index of
inequality -- can be used for income analysis of inequality amongst
subgroups of a population.

For the first two definitions, the setup is as follows:
The goal is to learn an accurate classifier
$h:\calX\rightarrow\{0, 1\}$ from some family of classifiers
(e.g., decision trees, neural networks, or polynomial threshold
functions) while satisfying some definition of statistical fairness.
We assume that we are given training examples
$D = \{(\bx_i, a_i, y_i)\}_{i=1}^n\in(\calX, \calA, \calY)^n$ typically representing $n$
individuals drawn i.i.d. over the joint distribution $(\bx, a, y)\sim \mu$.
For each $i\in[n]$, $\bx_i\in\calX$ is the features of individual $i$,
$a_i\in \calA$ is the protected attribute of the individual (e.g.,
race or gender), $y_i\in\calY$ is the label. Using
constrained group-objective optimization, the chosen hypothesis $h$ need
not have access to nor knowledge of the protected attribute
$A$ during testing or deployment.

\begin{definition}[Equalized Odds~\citep{HardtPNS16}]
A classifier $h$ satisfies \textbf{Equalized Odds} under a distribution
over $(\bx, a, y)$ if $h(X)$, its prediction,
is conditionally independent of the protected attribute $A$ given
the label $Y$. 

In notation, we have that for all $a\in\calA, \hat{y}\in\calY, x\in\calX$
$$
\pr_{(\bx, a, y)\sim\mu}\left[h(\bx) = \hat{y} | A = a, Y = y\right] = \pr\left[h(\bx) = \hat{y} | Y = y\right].
$$

\label{def:eo}
\end{definition}

\begin{definition}[Demographic Parity~\citep{AgarwalBD0W18}]
A classifier $h$ satisfies \textbf{Demographic Parity} under a distribution
over $(\bx, a, y)$ if $h(X)$, its prediction, is
statistically independent of the protected attribute $A$.

In notation, we have that for all $a\in\calA, \hat{y}\in\calY$
$$
\pr_{(\bx, a, y)\sim\mu}\left[h(\bx) = \hat{y} | A = a\right] = \pr\left[h(\bx) = \hat{y}\right].
$$
\label{def:demo}
\end{definition}

For optimization purposes, we often cannot satisfy either
equalized odds or demographic parity exactly so we must instead pursue
relaxations of equalized odds and demographic parity.
For example,~\cite{Jagielski18} pursue
$\gamma$ Equalized Odds which is defined in Definition~\ref{def:alphaeo2}, stated
in terms of false positives and false negatives of a hypothesis $h$.

\begin{definition}[Gini Index of Inequality~\citep{Busa-FeketeSWM17}]

Suppose there are $K$ subgroups in a population of individuals
earning income. The \textbf{Gini Index of Inequality} is given by
$$
I(l) =  \frac{\sum_{i, j}|l_i - l_j|}{2n\sum_i l_i} \in [0, 1],
$$
where $l\in[0, 1]^K$ and $l_i$ could represent the percentile
average income of individuals in subgroup $i$.

The Gini index is not convex but quasi-convex which means that
its level sets are convex. For any given $\theta\in[0, 1]$,
$I(l) \leq \theta$ is equivalent to
$$
\sum_{i, j}|l_i - l_j| - 2n\theta\sum l_i \leq 0,
$$
which is a convex constraint~\citep{AlabiIK18}.
\label{def:gi}
\end{definition}

Definition~\ref{def:gi} makes no distributional assumptions on the loss vector
$l\in[0, 1]^K$.

\cite{AgarwalBD0W18} show how to convert
the empirical risk minimization problem for satisfying
either demographic
parity or equalized odds into the following problem:
\begin{equation}
    \min_{h\in\Delta}\hat{err}(h)\,\,\text{ s.t. }\,\,
    M\hat{\mu}(h) \leq \hat{c}
\label{eq:agarwal}
\end{equation}

where the matrix $M\in\reals^{|\calK|\times|\calJ|}$ and the
vector $\hat{c}\in\reals^{|\calK|}$ specify linear constraints for the problem
and $\hat\mu(h)\in\reals^{|\calJ|}$ is a vector of conditional moments taken
over the the distribution on $(\calX, A, \calY)$.

To convert into the form of constrained group-objective optimization, we
set $f(\ell(h, D)) = \hat{err}(h)$ and
$g(\ell(h, D)) = (M\hat{\mu}(h) - \hat{c})\cdot \textbf{1}$
where $\textbf{1}$ is the all-ones vector or
$g(\ell(h, D)) = \max_{i\in[\calK]}(M\hat{\mu}(h) - \hat{c})_i$.
We leave out details of how to convert the definition of
Demographic Parity and Equalized Odds into Equation~\ref{eq:agarwal} as this is already done in Section 2 (termed ``Problem Formulation'') in
~\citep{AgarwalBD0W18}.
Last, for the Gini index of inequality, we can convert into
constrained group-objective optimization by setting the constraint function
$g$ to $g(l) = \sum_{i, j}|l_i - l_j| - 2n\theta\sum l_i \leq 0$ for some $\theta\in[0, 1]$
and setting $f(l) = -\sum_i l_i$ for all $l\in[0, 1]^K$.

We note here that $K$, the number of groups, is not a constant and could
vary depending on the context, application, and matters of intersectionality
\citep{BuolamwiniG18, Hebert-JohnsonK18}.

\subsubsection{Satisfying Equalized Odds}

Now, we show how to satisfy approximate notions of equalized odds while respecting
differential privacy constraints.

\begin{definition}[$\alpha$-Equalized Odds~\citep{Jagielski18}]
Let $X, A, Y$ be random variables representing the non-sensitive
features, the sensitive attribute, and the label assigned to an individual, respectively.

Given a dataset of examples $D = \{(\bx_i, a_i, y_i)\}_{i=1}^n\in(\calX, \calA, \{0, 1\})^n$ of size
$n$, we say a classifier $c\in\calC_P$ satisfies $\alpha$ \textbf{Equalized Odds} if
\begin{equation}
    \max_{a, a'\in\calA}\{\max(|\hat{FP}_a - \hat{FP}_{a'}|, |\hat{TP}_a - \hat{TP}_{a'}| )\}\leq \alpha
    \label{eq:eqodds2}
\end{equation}

where $\hat{FP}_a, \hat{TP}_a$ are empirical estimates of
$FP_a(c) = \pr_{(x, y, a)}[c(\bx) = 1 | A = a, y = 0]$,
$TP_a(c) = \pr_{(x, y, a)}[c(\bx) = 1 | A = a, y = 1]$ respectively on dataset $D$.
\footnote{
$FP_a(c)$ is usually referred to as the false positive rate on attribute $A=a$.
Likewise, $FN_a(c)$ and $TP_a(c)$ are the false negative and true positive rates
on attribute $A=a$ respectively.
}

\label{def:alphaeo2}
\end{definition}

We say a classifier satisfies $\alpha$-\textbf{Smoothed Equalized Odds}
if the smoothed version of Equation~\ref{eq:eqodds2} is satisfied
(i.e., when the maximum and absolute functions in Equation~\ref{eq:eqodds2} are replaced with smoothed
versions
\footnote{
For example, the smooth maximum function is a smooth approximation
to the maximum function.
} or using the \text{Moreau-Yosida} regularization technique).

For concreteness, we provide a specific smoothed version of $\alpha$-equalized odds in Definition~\ref{def:alphaseo2}.

\begin{definition}[$(\alpha, \eta)$-Smoothed Equalized Odds]
Let $X, A, Y$ be random variables representing the non-sensitive
features, the sensitive attribute, and the label assigned to an individual, respectively.

Given a dataset of examples $D = \{(\bx_i, a_i, y_i)\}_{i=1}^n\in(\calX, \calA, \{0, 1\})^n$ of size
$n$, we say a classifier $c\in\calC_P$ satisfies $(\alpha, \eta)$ \textbf{Equalized Odds} if the constraint function
\begin{equation}
    g(\hat{FP}, \hat{FN}, \hat{TP}) = \smax^\eta_{a, a'\in\calA}\{\max(|\hat{FP}_a - \hat{FP}_{a'}|, |\hat{TP}_a - \hat{TP}_{a'}| )\} - \alpha
    \label{eq:eqodds2}
\end{equation}

is less than or equal to 0. $(\hat{FP}, \hat{FN}, \hat{TP})$ corresponds to the $3|\calA|$ empirical estimates of the false
positives, false negatives, and true positives for the
$|\calA|$ groups.
$(\hat{FP}, \hat{TP})$ are used to enforce the equalized odds constraint while
$(\hat{FP}, \hat{FN})$ are used to compute the error of the classifier.
We use the smooth maximum function
$\smax^\eta(y_1, \ldots, y_n) = \frac{\sum_{i=1}^n y_ie^{\eta y_i}}{\sum_{i=1}^ne^{\eta y_i}}$~\citep{LangeZHV14}. As $\eta\rightarrow\infty$,
$\smax^\eta\rightarrow\max$.

\label{def:alphaseo2}
\end{definition}

\begin{corollary}
For any privacy parameters $\eps, \delta\in(0, 1]$,
suppose we have a dataset of examples $D = \{(\bx_i, a_i, y_i)\}_{i=1}^n$ of size
$n$ where
$\bx_i\in\calX, y_i\in\{0, 1\}$, $a_i\in\calA$, for all
$i\in[n]$. Assume that
there exists at least one decision in $\calC_P$
(with VC dimension $\VC(\calC_P)$ in parametric space $\reals^P$) that
satisfies $(\alpha, \eta)$-Smoothed Equalized Odds
(by Definition~\ref{def:alphaeo2}) for some $\eta > 0$.

Then there exists $n_0 = \tilde{O}\left(\frac{|\calA|^4\cdot\VC(\calC_P)}{\eps\alpha^3}\right)$ such that
for all $n\geq n_0$, with probability at least $9/10$, we can
obtain a decision $\tilde{c}\in\calC_P$ satisfying $(\alpha, \eta)$-smoothed 
equalized odds and that is within $\alpha$
away from the most accurate classifier.

\label{cor:eoapprox}
\end{corollary}

\begin{proof}

For any classifier $c\in\calC_P$ we set the loss vector of $c$ on $D$ to
$$
\ell(c, D) = (\ell_1(c, D), \ldots, \ell_K(c, D)),
$$
with $K = 3|\calA|$ where
$$\ell_1(c, D) = \hat{FP}_1, \ldots, \ell_{|A|}(c, D) = \hat{FP}_{|\calA|},$$
$$\ell_{|\calA|+1}(c, D) = \hat{TP}_{1}, \ldots, \ell_{2|A|}(c, D) = \hat{TP}_{|\calA|},$$
and
$$\ell_{2|\calA|+1}(c, D) = \hat{FN}_{1}, \ldots, \ell_{3|A|}(c, D) = \hat{FN}_{|\calA|}.$$

Then the empirical error of any classifier $c\in\calC_P$ is $f(\ell(c, D)) = \sum_{a\in\calA}\hat{FP}_a+\hat{FN}_a$
(the sum of the false positives and false negatives) and the $\alpha$-equalized odds is
enforced via $$g(\ell(c, D)) = \max_{a, a'\in\calA}\{\max(|\hat{FP}_a - \hat{FP}_{a'}|, |\hat{TP}_a - \hat{TP}_{a'}| )\} - \alpha,$$ so that $g(\ell(c, D)) \leq 0$ for all classifiers $c\in\calC_P$ that
satisfy $\alpha$-equalized odds. We use a smoothed version of $g(\ell(c, D))$
(i.e., replace the maximum and absolute functions in that equation with the
smooth approximations of those functions or using the \text{Moreau-Yosida} regularization technique).

Now, assuming there exists at least one choice $c\in\calC_P$ such that
$g(\ell(c, D)) \leq 0$, we can just directly apply Corollary~\ref{cor:loptexp} to 
$f,g$ to obtain the desired corollary.
\end{proof}

Corollary~\ref{cor:eoapprox} results in sample size
$\tilde{O}\left(\frac{|\calA|^4\cdot\VC(\calC_P)}{\eps\alpha^3}\right)$.
As a result, the linear optimization based algorithm for
Theorem~\ref{thm:paik} performs better (in terms of asymptotic sample
complexity) than the \textbf{DP-oracle-learner} of~\cite{Jagielski18}
which requires sample size
$\tilde{O}\left(\frac{|\calA|^3\VC(\calC_P)}{\eps\alpha^{3+r}}\right)$ for any $r > 0$
provided that
$\min_{a, y}\{\hat{q}_{ay}\} \leq \alpha^{(1+r)/2}$ where
$\hat{q}_{ay}$ is the empirical estimate of $\pr[A = a, Y = y]$.

Example~\ref{ex:eoapprox} shows how to use Corollary~\ref{cor:eoapprox} on specific
smooth approximations to the absolute value and maximum functions. Note that this is
one of many ways to apply smoothing and this specific version might not apply
optimal smoothing. Smoothing of non-smooth functions for specific tasks is
an entire research agenda (e.g., see~\citep{BachJMO12, LangeZHV14} for smooth approximations to the max and absolute value functions) and not the focus of this research paper.

\begin{example}
Suppose we have two groups and can obtain empirical estimates of the false positive, 
false negative, and true positive rates of these groups for any $c\in\calC_P$ and
fixed dataset of labeled examples $D=\{(\bx_i, a_i, y_i)\}_{i=1}^n\in(\calX\times\{1, 2\}\times\{0, 1\})^n$.
The goal is to perform binary classification under $(\alpha, \eta)$ equalized odds constraints.
We let $\calC_P$ be the set of single-parametric threshold classifiers over the reals so that
$\VC(\calC_P) = 1$.
The loss vector is
$$\ell(c, D) = (\hat{FP}_1(c, D), \hat{FP}_2(c, D), \hat{FN}_1(c, D), \hat{FN}_2(c, D), \hat{TP}_1(c, D), \hat{TP}_2(c, D)).$$
The ``error'' function $f$ can be computed as the linear function $f(\ell(c, D)) = \hat{FP}_1(c, D) + \hat{FP}_2(c, D) + \hat{FN}_1(c, D) + \hat{FN}_2(c, D)$ and so
is Lipschitz, smooth.

Now, define a smooth
approximation to the max function as
$\smax^\eta(y_1, \ldots, y_n) = \frac{\sum_{i=1}^ny_ie^{\eta y_i}}{\sum_{i=1}^ne^{\eta y_i}}$.
Then if for all decisions $c\in\calC_P$ and for a fixed $\eta, D$,
$\smax^\eta(|\hat{FP}_1(c, D)-\hat{FP}_2(c, D)|), |\hat{TP}_1(c, D)-\hat{TP}_2(c, D)|)$ is 1-Lipschitz, we can apply Corollary~\ref{cor:eoapprox}
to obtain the desired guarantees.

More generally, 
the possible empirical rates for any given dataset depends on how complex the decision set $\calC_P$ is.
The decision set $\calC_P$ could represent all possible parameterizations of a specific
neural network architecture, all possible weights for a polynomial threshold
function to be used for classification, or a constant separator threshold defined over the reals.
\label{ex:eoapprox}
\end{example}

\begin{remark}

We have shown how to apply our methods to satisfy certain statistical 
definitions of algorithmic fairness and some convex performance measures of the
confusion matrix. There are potentially other use cases 
we have not explored that we leave for future work.

\end{remark}

\section{Conclusion}

In this paper, we have explored the problem of private
fair optimization via a reductions approach.
We provided an $(\eps, 0)$-differentially private
exponential sampling algorithm
and an $(\eps, \delta)$-differentially private linear optimization
based algorithm as a solution. As a side-effect perhaps, we introduce
a class of bounded divergence linear optimizers, which could be
useful for solving more general (Lipschitz-continuous)
multi-objective problems. We also provide a lower bound on the excess
risk (or equivalently, on the sample complexity)
for any $\eps$ or $(\eps, \delta)$-differentially
private algorithm that solves the constrained group-objective optimization problem.

Our framework can also be used to solve model
projection~\citep{AlghamdiAWCWR20} for certain
constraint sets constructed via linear inequalities.
Constructively showing connections between the
model projection problem (for certain $f$-divergences) 
and linear optimizers with divergence constraints is left
to future work. Furthermore, linear optimizers share the
same form as a Rectified Linear Unit (ReLU) used in
deep learning. Showing constructive equivalences, especially
by delineating time/accuracy tradeoffs~\citep{GoelKK19}, is
left to future work.

\section{Acknowledgements}

Thanks to Salil Vadhan for providing detailed
comments on this paper at various stages of the project and to
Cynthia Dwork, Adam Smith, and other members of the
Boston-Area Data Privacy group for
comments that improved the writing of this work.
We are grateful to Shahab Asoodeh, Boaz Barak,
Flavio du Pin Calmon, 
Nicole Immorlica, Adam Tauman Kalai,
Sanmi Koyejo, Ankur Moitra, Jelani Nelson, and
Chris Wiggins
for illuminating discussions related to this work.
The Computational Social Science group at Facebook AI
Research and participants of the 
``Recent Themes in Resource Tradeoffs: Privacy, Fairness, and Robustness'' workshop at the 
University of Minnesota
provided many helpful comments and suggestions that
improved the writing of this work.
Finally, we appreciate the guidance of
the anonymous referees who helped to improve the presentation
of this work.

%\bibliographystyle{alpha}
%\bibliography{main}
\clearpage

\bibliographystyle{plainnat}
\bibliography{main}

\begin{thebibliography}{96}
\providecommand{\natexlab}[1]{#1}
\providecommand{\url}[1]{\texttt{#1}}
\expandafter\ifx\csname urlstyle\endcsname\relax
  \providecommand{\doi}[1]{doi: #1}\else
  \providecommand{\doi}{doi: \begingroup \urlstyle{rm}\Url}\fi

\bibitem[Abowd(2018)]{Abowd18}
John~M. Abowd.
\newblock The u.s. census bureau adopts differential privacy.
\newblock In \emph{Proceedings of the 24th ACM SIGKDD International Conference
  on Knowledge Discovery \& Data Mining}, KDD '18, page 2867, New York, NY,
  USA, 2018. Association for Computing Machinery.
\newblock ISBN 9781450355520.
\newblock URL \url{https://doi.org/10.1145/3219819.3226070}.

\bibitem[Agarwal et~al.(2018)Agarwal, Beygelzimer, Dud{\'{\i}}k, Langford, and
  Wallach]{AgarwalBD0W18}
Alekh Agarwal, Alina Beygelzimer, Miroslav Dud{\'{\i}}k, John Langford, and
  Hanna~M. Wallach.
\newblock A reductions approach to fair classification.
\newblock In \emph{Proceedings of the 35th International Conference on Machine
  Learning, {ICML} 2018, Stockholmsm{\"{a}}ssan, Stockholm, Sweden, July 10-15,
  2018}, pages 60--69, 2018.

\bibitem[Agarwal et~al.(2019)Agarwal, Dud{\'{\i}}k, and Wu]{AgarwalDW19}
Alekh Agarwal, Miroslav Dud{\'{\i}}k, and Zhiwei~Steven Wu.
\newblock Fair regression: Quantitative definitions and reduction-based
  algorithms.
\newblock In \emph{Proceedings of the 36th International Conference on Machine
  Learning, {ICML} 2019, 9-15 June 2019, Long Beach, California, {USA}}, pages
  120--129, 2019.
\newblock URL \url{http://proceedings.mlr.press/v97/agarwal19d.html}.

\bibitem[Alabi et~al.(2018)Alabi, Immorlica, and Kalai]{AlabiIK18}
Daniel Alabi, Nicole Immorlica, and Adam~Tauman Kalai.
\newblock Unleashing linear optimizers for group-fair learning and
  optimization.
\newblock In \emph{Conference On Learning Theory, {COLT} 2018, Stockholm,
  Sweden, 6-9 July 2018.}, pages 2043--2066, 2018.
\newblock URL \url{http://proceedings.mlr.press/v75/alabi18a.html}.

\bibitem[Alghamdi et~al.(2020)Alghamdi, Asoodeh, Wang, Calmon, Wei, and
  Ramamurthy]{AlghamdiAWCWR20}
Wael Alghamdi, Shahab Asoodeh, Hao Wang, Fl{\'{a}}vio~P. Calmon, Dennis Wei,
  and Karthikeyan~Natesan Ramamurthy.
\newblock Model projection: Theory and applications to fair machine learning.
\newblock In \emph{{IEEE} International Symposium on Information Theory, {ISIT}
  2020, Los Angeles, CA, USA, June 21-26, 2020}, pages 2711--2716. {IEEE},
  2020.

\bibitem[Bach et~al.(2012)Bach, Jenatton, Mairal, and Obozinski]{BachJMO12}
Francis~R. Bach, Rodolphe Jenatton, Julien Mairal, and Guillaume Obozinski.
\newblock Optimization with sparsity-inducing penalties.
\newblock \emph{Foundations and Trends in Machine Learning}, 4\penalty0
  (1):\penalty0 1--106, 2012.
\newblock URL \url{https://doi.org/10.1561/2200000015}.

\bibitem[Ball et~al.(2020)Ball, Boyle, Degwekar, Deshpande, Rosen,
  Vaikuntanathan, and Vasudevan]{BallBDDRVV20}
Marshall Ball, Elette Boyle, Akshay Degwekar, Apoorvaa Deshpande, Alon Rosen,
  Vinod Vaikuntanathan, and Prashant~Nalini Vasudevan.
\newblock Cryptography from information loss.
\newblock In \emph{11th Innovations in Theoretical Computer Science Conference,
  {ITCS} 2020, January 12-14, 2020, Seattle, Washington, {USA}}, pages
  81:1--81:27, 2020.

\bibitem[Bartlett et~al.(2006)Bartlett, Jordan, and McAuliffe]{BJM03}
Peter~L Bartlett, Michael~I Jordan, and Jon~D McAuliffe.
\newblock Convexity, classification, and risk bounds.
\newblock \emph{Journal of the American Statistical Association}, 101\penalty0
  (473):\penalty0 138--156, 2006.
\newblock \doi{10.1198/016214505000000907}.

\bibitem[Bassily et~al.(2014)Bassily, Smith, and Thakurta]{BassilyST14}
Raef Bassily, Adam~D. Smith, and Abhradeep Thakurta.
\newblock Private empirical risk minimization: Efficient algorithms and tight
  error bounds.
\newblock In \emph{55th {IEEE} Annual Symposium on Foundations of Computer
  Science, {FOCS} 2014, Philadelphia, PA, USA, October 18-21, 2014}, pages
  464--473, 2014.

\bibitem[Bennett et~al.(1995)Bennett, Brassard, Cr{\'{e}}peau, and
  Maurer]{BennettBCM95}
Charles~H. Bennett, Gilles Brassard, Claude Cr{\'{e}}peau, and Ueli~M. Maurer.
\newblock Generalized privacy amplification.
\newblock \emph{{IEEE} Trans. Inf. Theory}, 41\penalty0 (6):\penalty0
  1915--1923, 1995.
\newblock URL \url{https://doi.org/10.1109/18.476316}.

\bibitem[Beygelzimer et~al.(2009)Beygelzimer, Langford, and
  Zadrozny]{BeygelzimerLZ09}
Alina Beygelzimer, John Langford, and Bianca Zadrozny.
\newblock Tutorial summary: Reductions in machine learning.
\newblock In Andrea~Pohoreckyj Danyluk, L{\'{e}}on Bottou, and Michael~L.
  Littman, editors, \emph{Proceedings of the 26th Annual International
  Conference on Machine Learning, {ICML} 2009, Montreal, Quebec, Canada, June
  14-18, 2009}, volume 382 of \emph{{ACM} International Conference Proceeding
  Series}, page~12. {ACM}, 2009.

\bibitem[Blackwell(1956)]{Blackwell56}
David Blackwell.
\newblock An analog of the minimax theorem for vector payoffs.
\newblock \emph{Pacific Journal of Mathematics}, 6\penalty0 (1):\penalty0 1--8,
  1956.

\bibitem[Blumer et~al.(1989)Blumer, Ehrenfeucht, Haussler, and
  Warmuth]{BlumerEHW89}
Anselm Blumer, Andrzej Ehrenfeucht, David Haussler, and Manfred~K. Warmuth.
\newblock Learnability and the vapnik-chervonenkis dimension.
\newblock \emph{J. {ACM}}, 36\penalty0 (4):\penalty0 929--965, 1989.
\newblock URL \url{https://doi.org/10.1145/76359.76371}.

\bibitem[Bousquet and Elisseeff(2000)]{BousquetE00}
Olivier Bousquet and Andr{\'{e}} Elisseeff.
\newblock Algorithmic stability and generalization performance.
\newblock In \emph{Advances in Neural Information Processing Systems 13, Papers
  from Neural Information Processing Systems {(NIPS)} 2000, Denver, CO, {USA}},
  pages 196--202, 2000.
\newblock URL
  \url{https://proceedings.neurips.cc/paper/2000/hash/49ad23d1ec9fa4bd8d77d02681df5cfa-Abstract.html}.

\bibitem[Brassard et~al.(1986)Brassard, Cr{\'{e}}peau, and
  Robert]{BrassardCR86}
Gilles Brassard, Claude Cr{\'{e}}peau, and Jean{-}Marc Robert.
\newblock Information theoretic reductions among disclosure problems.
\newblock In \emph{27th Annual Symposium on Foundations of Computer Science,
  Toronto, Canada, 27-29 October 1986}, pages 168--173. {IEEE} Computer
  Society, 1986.
\newblock URL \url{https://doi.org/10.1109/SFCS.1986.26}.

\bibitem[Bubeck(2015)]{Bubeck15}
S{\'{e}}bastien Bubeck.
\newblock Convex optimization: Algorithms and complexity.
\newblock \emph{Foundations and Trends in Machine Learning}, 8\penalty0
  (3-4):\penalty0 231--357, 2015.

\bibitem[Buolamwini and Gebru(2018)]{BuolamwiniG18}
Joy Buolamwini and Timnit Gebru.
\newblock Gender shades: Intersectional accuracy disparities in commercial
  gender classification.
\newblock In \emph{Conference on Fairness, Accountability and Transparency,
  {FAT} 2018, 23-24 February 2018, New York, NY, {USA}}, pages 77--91, 2018.
\newblock URL \url{http://proceedings.mlr.press/v81/buolamwini18a.html}.

\bibitem[Busa{-}Fekete et~al.(2017)Busa{-}Fekete, Sz{\"{o}}r{\'{e}}nyi, Weng,
  and Mannor]{Busa-FeketeSWM17}
R{\'{o}}bert Busa{-}Fekete, Bal{\'{a}}zs Sz{\"{o}}r{\'{e}}nyi, Paul Weng, and
  Shie Mannor.
\newblock Multi-objective bandits: Optimizing the generalized gini index.
\newblock In \emph{Proceedings of the 34th International Conference on Machine
  Learning, {ICML} 2017, Sydney, NSW, Australia, 6-11 August 2017}, pages
  625--634, 2017.
\newblock URL \url{http://proceedings.mlr.press/v70/busa-fekete17a.html}.

\bibitem[Canonne et~al.(2019)Canonne, Kamath, McMillan, Smith, and
  Ullman]{CanonneKMSU19}
Cl{\'{e}}ment~L. Canonne, Gautam Kamath, Audra McMillan, Adam~D. Smith, and
  Jonathan Ullman.
\newblock The structure of optimal private tests for simple hypotheses.
\newblock In \emph{Proceedings of the 51st Annual {ACM} {SIGACT} Symposium on
  Theory of Computing, {STOC} 2019, Phoenix, AZ, USA, June 23-26, 2019}, pages
  310--321, 2019.

\bibitem[Carlini et~al.(2019)Carlini, Liu, Erlingsson, Kos, and
  Song]{Carlini0EKS19}
Nicholas Carlini, Chang Liu, {\'{U}}lfar Erlingsson, Jernej Kos, and Dawn Song.
\newblock The secret sharer: Evaluating and testing unintended memorization in
  neural networks.
\newblock In Nadia Heninger and Patrick Traynor, editors, \emph{28th {USENIX}
  Security Symposium, {USENIX} Security 2019, Santa Clara, CA, USA, August
  14-16, 2019}, pages 267--284. {USENIX} Association, 2019.
\newblock URL
  \url{https://www.usenix.org/conference/usenixsecurity19/presentation/carlini}.

\bibitem[Cesa{-}Bianchi and Lugosi(2006)]{CL06}
Nicol{\`{o}} Cesa{-}Bianchi and G{\'{a}}bor Lugosi.
\newblock \emph{Prediction, learning, and games}.
\newblock Cambridge University Press, 2006.

\bibitem[Chaudhuri and Hsu(2011)]{ChaudhuriH11}
Kamalika Chaudhuri and Daniel~J. Hsu.
\newblock Sample complexity bounds for differentially private learning.
\newblock In \emph{{COLT} 2011 - The 24th Annual Conference on Learning Theory,
  June 9-11, 2011, Budapest, Hungary}, pages 155--186, 2011.
\newblock URL
  \url{http://proceedings.mlr.press/v19/chaudhuri11a/chaudhuri11a.pdf}.

\bibitem[Chaudhuri et~al.(2011)Chaudhuri, Monteleoni, and
  Sarwate]{ChaudhuriMS11}
Kamalika Chaudhuri, Claire Monteleoni, and Anand~D. Sarwate.
\newblock Differentially private empirical risk minimization.
\newblock \emph{Journal of Machine Learning Research}, 12:\penalty0 1069--1109,
  2011.

\bibitem[Chervonenkis and Vapnik(1971)]{Chervonenkis1971}
A.~Ya. Chervonenkis and V.~N. Vapnik.
\newblock {On the Uniform Convergence of Relative Frequencies of Events to
  Their Probabilities}, 1971.
\newblock ISSN 0040-585X.

\bibitem[Choromanski and Malkin(2012)]{ChoromanskiM12}
Krzysztof Choromanski and Tal Malkin.
\newblock The power of the dinur-nissim algorithm: breaking privacy of
  statistical and graph databases.
\newblock In \emph{Proceedings of the 31st {ACM} {SIGMOD-SIGACT-SIGART}
  Symposium on Principles of Database Systems, {PODS} 2012, Scottsdale, AZ,
  USA, May 20-24, 2012}, pages 65--76, 2012.
\newblock URL \url{https://doi.org/10.1145/2213556.2213570}.

\bibitem[Consortium et~al.(2009)Consortium, Church, Heeney, Hawkins, de~Vries,
  Boddington, Kaye, Bobrow, and Weir]{10.1371/journal.pgen.1000665}
P3G Consortium, George Church, Catherine Heeney, Naomi Hawkins, Jantina
  de~Vries, Paula Boddington, Jane Kaye, Martin Bobrow, and Bruce Weir.
\newblock Public access to genome-wide data: Five views on balancing research
  with privacy and protection.
\newblock \emph{PLOS Genetics}, 5\penalty0 (10):\penalty0 1--4, 10 2009.
\newblock URL \url{https://doi.org/10.1371/journal.pgen.1000665}.

\bibitem[Cormode et~al.(2018)Cormode, Jha, Kulkarni, Li, Srivastava, and
  Wang]{CJKLSW18}
Graham Cormode, Somesh Jha, Tejas Kulkarni, Ninghui Li, Divesh Srivastava, and
  Tianhao Wang.
\newblock Privacy at scale: Local differential privacy in practice.
\newblock In \emph{Proceedings of the 2018 International Conference on
  Management of Data}, SIGMOD '18, page 1655–1658, New York, NY, USA, 2018.
  Association for Computing Machinery.
\newblock ISBN 9781450347037.
\newblock URL \url{https://doi.org/10.1145/3183713.3197390}.

\bibitem[Cuff and Yu(2016)]{CuffY16}
Paul Cuff and Lanqing Yu.
\newblock Differential privacy as a mutual information constraint.
\newblock In \emph{Proceedings of the 2016 {ACM} {SIGSAC} Conference on
  Computer and Communications Security, Vienna, Austria, October 24-28, 2016},
  pages 43--54, 2016.

\bibitem[De(2012)]{De12}
Anindya De.
\newblock Lower bounds in differential privacy.
\newblock In Ronald Cramer, editor, \emph{Theory of Cryptography}, pages
  321--338, Berlin, Heidelberg, 2012. Springer Berlin Heidelberg.
\newblock ISBN 978-3-642-28914-9.

\bibitem[Dinur and Nissim(2003)]{Dinur:2003}
Irit Dinur and Kobbi Nissim.
\newblock Revealing information while preserving privacy.
\newblock In \emph{Proceedings of the Twenty-second ACM SIGMOD-SIGACT-SIGART
  Symposium on Principles of Database Systems}, PODS '03, pages 202--210, New
  York, NY, USA, 2003. ACM.
\newblock ISBN 1-58113-670-6.

\bibitem[Donoho(2006)]{Donoho06}
D.L. Donoho.
\newblock Compressed sensing.
\newblock \emph{IEEE Transactions on Information Theory}, 52\penalty0
  (4):\penalty0 1289--1306, 2006.
\newblock \doi{10.1109/TIT.2006.871582}.

\bibitem[Dwork and Lei(2009)]{DworkL09}
Cynthia Dwork and Jing Lei.
\newblock Differential privacy and robust statistics.
\newblock In \emph{Proceedings of the 41st Annual {ACM} Symposium on Theory of
  Computing, {STOC} 2009, Bethesda, MD, USA, May 31 - June 2, 2009}, pages
  371--380, 2009.

\bibitem[Dwork and Roth(2014)]{DworkR14}
Cynthia Dwork and Aaron Roth.
\newblock The algorithmic foundations of differential privacy.
\newblock \emph{Foundations and Trends in Theoretical Computer Science},
  9\penalty0 (3-4):\penalty0 211--407, 2014.

\bibitem[Dwork et~al.(2006)Dwork, McSherry, Nissim, and Smith]{DworkMNS06}
Cynthia Dwork, Frank McSherry, Kobbi Nissim, and Adam~D. Smith.
\newblock Calibrating noise to sensitivity in private data analysis.
\newblock In \emph{Theory of Cryptography, Third Theory of Cryptography
  Conference, {TCC} 2006, New York, NY, USA, March 4-7, 2006, Proceedings},
  pages 265--284, 2006.

\bibitem[Dwork et~al.(2010)Dwork, Rothblum, and Vadhan]{DworkRV10}
Cynthia Dwork, Guy~N. Rothblum, and Salil~P. Vadhan.
\newblock Boosting and differential privacy.
\newblock In \emph{51th Annual {IEEE} Symposium on Foundations of Computer
  Science, {FOCS} 2010, October 23-26, 2010, Las Vegas, Nevada, {USA}}, pages
  51--60, 2010.

\bibitem[Dwork et~al.(2012)Dwork, Hardt, Pitassi, Reingold, and
  Zemel]{DworkHPRZ12}
Cynthia Dwork, Moritz Hardt, Toniann Pitassi, Omer Reingold, and Richard~S.
  Zemel.
\newblock Fairness through awareness.
\newblock In \emph{Innovations in Theoretical Computer Science 2012, Cambridge,
  MA, USA, January 8-10, 2012}, pages 214--226, 2012.

\bibitem[Dwork et~al.(2015)Dwork, Feldman, Hardt, Pitassi, Reingold, and
  Roth]{DworkFHPRR15}
Cynthia Dwork, Vitaly Feldman, Moritz Hardt, Toniann Pitassi, Omer Reingold,
  and Aaron~Leon Roth.
\newblock Preserving statistical validity in adaptive data analysis.
\newblock In \emph{Proceedings of the Forty-Seventh Annual {ACM} on Symposium
  on Theory of Computing, {STOC} 2015, Portland, OR, USA, June 14-17, 2015},
  pages 117--126, 2015.

\bibitem[Ehrenfeucht et~al.(1989)Ehrenfeucht, Haussler, Kearns, and
  Valiant]{EhrenfeuchtHKV89}
Andrzej Ehrenfeucht, David Haussler, Michael~J. Kearns, and Leslie~G. Valiant.
\newblock A general lower bound on the number of examples needed for learning.
\newblock \emph{Inf. Comput.}, 82\penalty0 (3):\penalty0 247--261, 1989.
\newblock URL \url{https://doi.org/10.1016/0890-5401(89)90002-3}.

\bibitem[Ekstrand et~al.(2018)Ekstrand, Joshaghani, and
  Mehrpouyan]{EkstrandJM18}
Michael~D. Ekstrand, Rezvan Joshaghani, and Hoda Mehrpouyan.
\newblock Privacy for all: Ensuring fair and equitable privacy protections.
\newblock In \emph{Conference on Fairness, Accountability and Transparency,
  {FAT} 2018, 23-24 February 2018, New York, NY, {USA}}, pages 35--47, 2018.
\newblock URL \url{http://proceedings.mlr.press/v81/ekstrand18a.html}.

\bibitem[Feldman and Vondr{\'{a}}k(2019)]{FeldmanV19}
Vitaly Feldman and Jan Vondr{\'{a}}k.
\newblock High probability generalization bounds for uniformly stable
  algorithms with nearly optimal rate.
\newblock In Alina Beygelzimer and Daniel Hsu, editors, \emph{Conference on
  Learning Theory, {COLT} 2019, 25-28 June 2019, Phoenix, AZ, {USA}}, volume~99
  of \emph{Proceedings of Machine Learning Research}, pages 1270--1279. {PMLR},
  2019.
\newblock URL \url{http://proceedings.mlr.press/v99/feldman19a.html}.

\bibitem[Freund and Schapire(1997)]{FreundS97}
Yoav Freund and Robert~E. Schapire.
\newblock A decision-theoretic generalization of on-line learning and an
  application to boosting.
\newblock \emph{J. Comput. Syst. Sci.}, 55\penalty0 (1):\penalty0 119--139,
  1997.
\newblock URL \url{https://doi.org/10.1006/jcss.1997.1504}.

\bibitem[Friedler et~al.(2021)Friedler, Scheidegger, and
  Venkatasubramanian]{FSV21}
Sorelle~A. Friedler, Carlos Scheidegger, and Suresh Venkatasubramanian.
\newblock The (im)possibility of fairness: Different value systems require
  different mechanisms for fair decision making.
\newblock \emph{Commun. ACM}, 64\penalty0 (4):\penalty0 136–143, March 2021.
\newblock ISSN 0001-0782.
\newblock \doi{10.1145/3433949}.
\newblock URL \url{https://doi.org/10.1145/3433949}.

\bibitem[Garfinkel et~al.(2018)Garfinkel, Abowd, and Martindale]{GarfinkelAM18}
Simson~L. Garfinkel, John~M. Abowd, and Christian Martindale.
\newblock Understanding database reconstruction attacks on public data.
\newblock \emph{{ACM} Queue}, 16\penalty0 (5):\penalty0 50, 2018.
\newblock URL \url{https://doi.org/10.1145/3291276.3295691}.

\bibitem[Goel et~al.(2019)Goel, Karmalkar, and Klivans]{GoelKK19}
Surbhi Goel, Sushrut Karmalkar, and Adam~R. Klivans.
\newblock Time/accuracy tradeoffs for learning a relu with respect to gaussian
  marginals.
\newblock In \emph{Advances in Neural Information Processing Systems 32: Annual
  Conference on Neural Information Processing Systems 2019, NeurIPS 2019,
  December 8-14, 2019, Vancouver, BC, Canada}, pages 8582--8591, 2019.

\bibitem[Hannan(1957)]{Hannan57}
James Hannan.
\newblock Approximation to bayes risk in repeated play.
\newblock \emph{Contributions to the Theory of Games, III, Ann. Math. Study
  Number 39}, pages 97--139, 1957.

\bibitem[Hardt et~al.(2016)Hardt, Price, and Srebro]{HardtPNS16}
Moritz Hardt, Eric Price, and Nati Srebro.
\newblock Equality of opportunity in supervised learning.
\newblock In \emph{Advances in Neural Information Processing Systems 29: Annual
  Conference on Neural Information Processing Systems 2016, December 5-10,
  2016, Barcelona, Spain}, pages 3315--3323, 2016.
\newblock URL
  \url{http://papers.nips.cc/paper/6374-equality-of-opportunity-in-supervised-learning}.

\bibitem[Hazan(2016)]{Hazan16}
Elad Hazan.
\newblock Introduction to online convex optimization.
\newblock \emph{Found. Trends Optim.}, 2\penalty0 (3-4):\penalty0 157--325,
  2016.

\bibitem[H{\'{e}}bert{-}Johnson et~al.(2018)H{\'{e}}bert{-}Johnson, Kim,
  Reingold, and Rothblum]{Hebert-JohnsonK18}
{\'{U}}rsula H{\'{e}}bert{-}Johnson, Michael~P. Kim, Omer Reingold, and Guy~N.
  Rothblum.
\newblock Multicalibration: Calibration for the (computationally-identifiable)
  masses.
\newblock In \emph{Proceedings of the 35th International Conference on Machine
  Learning, {ICML} 2018, Stockholmsm{\"{a}}ssan, Stockholm, Sweden, July 10-15,
  2018}, pages 1944--1953, 2018.
\newblock URL \url{http://proceedings.mlr.press/v80/hebert-johnson18a.html}.

\bibitem[Hiranandani et~al.(2019)Hiranandani, Boodaghians, Mehta, and
  Koyejo]{HBMK19}
Gaurush Hiranandani, Shant Boodaghians, Ruta Mehta, and Oluwasanmi~O Koyejo.
\newblock Multiclass performance metric elicitation.
\newblock In H.~Wallach, H.~Larochelle, A.~Beygelzimer, F.~dAlch\'{e} Buc,
  E.~Fox, and R.~Garnett, editors, \emph{Advances in Neural Information
  Processing Systems 32}, pages 9356--9365. Curran Associates, Inc., 2019.
\newblock URL
  \url{http://papers.nips.cc/paper/9133-multiclass-performance-metric-elicitation.pdf}.

\bibitem[Iyengar et~al.(2019)Iyengar, Near, Song, Thakkar, Thakurta, and
  Wang]{Iyengar19}
R.~Iyengar, J.~P. Near, D.~Song, O.~Thakkar, A.~Thakurta, and L.~Wang.
\newblock Towards practical differentially private convex optimization.
\newblock In \emph{2019 2019 IEEE Symposium on Security and Privacy (SP)}, Los
  Alamitos, CA, USA, may 2019. IEEE Computer Society.
\newblock \doi{10.1109/SP.2019.00001}.
\newblock URL \url{https://doi.ieeecomputersociety.org/10.1109/SP.2019.00001}.

\bibitem[Jaggi(2013)]{Jaggi13}
Martin Jaggi.
\newblock Revisiting frank-wolfe: Projection-free sparse convex optimization.
\newblock In \emph{Proceedings of the 30th International Conference on Machine
  Learning, {ICML} 2013, Atlanta, GA, USA, 16-21 June 2013}, pages 427--435,
  2013.
\newblock URL \url{http://jmlr.org/proceedings/papers/v28/jaggi13.html}.

\bibitem[Jagielski et~al.(2018)Jagielski, Kearns, Mao, Oprea, Roth,
  Sharifi{-}Malvajerdi, and Ullman]{Jagielski18}
Matthew Jagielski, Michael~J. Kearns, Jieming Mao, Alina Oprea, Aaron Roth,
  Saeed Sharifi{-}Malvajerdi, and Jonathan Ullman.
\newblock Differentially private fair learning.
\newblock \emph{CoRR}, abs/1812.02696, 2018.
\newblock URL \url{http://arxiv.org/abs/1812.02696}.

\bibitem[Kairouz et~al.(2017)Kairouz, Oh, and Viswanath]{KairouzOV17}
Peter Kairouz, Sewoong Oh, and Pramod Viswanath.
\newblock The composition theorem for differential privacy.
\newblock \emph{{IEEE} Trans. Inf. Theory}, 63\penalty0 (6):\penalty0
  4037--4049, 2017.
\newblock URL \url{https://doi.org/10.1109/TIT.2017.2685505}.

\bibitem[Kakade et~al.(2009)Kakade, Kalai, and Ligett]{KakadeKL09}
Sham~M. Kakade, Adam~Tauman Kalai, and Katrina Ligett.
\newblock Playing games with approximation algorithms.
\newblock \emph{{SIAM} J. Comput.}, 39\penalty0 (3):\penalty0 1088--1106, 2009.

\bibitem[Kalai and Vempala(2003)]{KalaiV03}
Adam Kalai and Santosh~S. Vempala.
\newblock Efficient algorithms for online decision problems.
\newblock In \emph{Computational Learning Theory and Kernel Machines, 16th
  Annual Conference on Computational Learning Theory and 7th Kernel Workshop,
  COLT/Kernel 2003, Washington, DC, USA, August 24-27, 2003, Proceedings},
  pages 26--40, 2003.

\bibitem[Karp(1972)]{Karp1972}
Richard~M. Karp.
\newblock \emph{Reducibility among Combinatorial Problems}, pages 85--103.
\newblock Springer US, Boston, MA, 1972.
\newblock URL \url{https://doi.org/10.1007/978-1-4684-2001-2_9}.

\bibitem[Kearns(1998)]{Kearns98}
Michael~J. Kearns.
\newblock Efficient noise-tolerant learning from statistical queries.
\newblock \emph{J. {ACM}}, 45\penalty0 (6):\penalty0 983--1006, 1998.
\newblock URL \url{https://doi.org/10.1145/293347.293351}.

\bibitem[Keener(2010)]{keener2010theoretical}
R.W. Keener.
\newblock \emph{Theoretical Statistics: Topics for a Core Course}.
\newblock Springer Texts in Statistics. Springer New York, 2010.
\newblock ISBN 9780387938394.

\bibitem[Kifer et~al.(2012)Kifer, Smith, and Thakurta]{KiferST12}
Daniel Kifer, Adam~D. Smith, and Abhradeep Thakurta.
\newblock Private convex optimization for empirical risk minimization with
  applications to high-dimensional regression.
\newblock In \emph{{COLT} 2012 - The 25th Annual Conference on Learning Theory,
  June 25-27, 2012, Edinburgh, Scotland}, pages 25.1--25.40, 2012.
\newblock URL
  \url{http://www.jmlr.org/proceedings/papers/v23/kifer12/kifer12.pdf}.

\bibitem[Kilbertus et~al.(2018)Kilbertus, Gascon, Kusner, Veale, Gummadi, and
  Weller]{pmlr-v80-kilbertus18a}
Niki Kilbertus, Adria Gascon, Matt Kusner, Michael Veale, Krishna Gummadi, and
  Adrian Weller.
\newblock Blind justice: Fairness with encrypted sensitive attributes.
\newblock In Jennifer Dy and Andreas Krause, editors, \emph{Proceedings of the
  35th International Conference on Machine Learning}, volume~80 of
  \emph{Proceedings of Machine Learning Research}, pages 2630--2639,
  Stockholmsmässan, Stockholm Sweden, 10--15 Jul 2018. PMLR.
\newblock URL \url{http://proceedings.mlr.press/v80/kilbertus18a.html}.

\bibitem[Kleinberg et~al.(2017)Kleinberg, Mullainathan, and
  Raghavan]{KleinbergMR17}
Jon~M. Kleinberg, Sendhil Mullainathan, and Manish Raghavan.
\newblock Inherent trade-offs in the fair determination of risk scores.
\newblock In Christos~H. Papadimitriou, editor, \emph{8th Innovations in
  Theoretical Computer Science Conference, {ITCS} 2017, January 9-11, 2017,
  Berkeley, CA, {USA}}, volume~67 of \emph{LIPIcs}, pages 43:1--43:23. Schloss
  Dagstuhl - Leibniz-Zentrum f{\"{u}}r Informatik, 2017.
\newblock URL \url{https://doi.org/10.4230/LIPIcs.ITCS.2017.43}.

\bibitem[Kleinberg et~al.(2020)Kleinberg, Ludwig, Mullainathan, and
  Sunstein]{KleinbergLMS20}
Jon~M. Kleinberg, Jens Ludwig, Sendhil Mullainathan, and Cass~R. Sunstein.
\newblock Algorithms as discrimination detectors.
\newblock \emph{Proc. Natl. Acad. Sci. {USA}}, 117\penalty0 (48):\penalty0
  30096--30100, 2020.
\newblock URL \url{https://doi.org/10.1073/pnas.1912790117}.

\bibitem[Koltchinskii and Panchenko(2000)]{KP2000}
Vladimir Koltchinskii and Dmitriy Panchenko.
\newblock Rademacher processes and bounding the risk of function learning.
\newblock In Evarist Gin{\'e}, David~M. Mason, and Jon~A. Wellner, editors,
  \emph{High Dimensional Probability II}, pages 443--457, Boston, MA, 2000.
  Birkh{\"a}user Boston.
\newblock ISBN 978-1-4612-1358-1.

\bibitem[Komarova and Nekipelov(2020)]{komarova2020identification}
Tatiana Komarova and Denis Nekipelov.
\newblock Identification and formal privacy guarantees, 2020.

\bibitem[Lange et~al.(2014)Lange, Z{\"{u}}hlke, Holz, and Villmann]{LangeZHV14}
Mandy Lange, Dietlind Z{\"{u}}hlke, Olaf Holz, and Thomas Villmann.
\newblock Applications of lp-norms and their smooth approximations for gradient
  based learning vector quantization.
\newblock In \emph{22th European Symposium on Artificial Neural Networks,
  {ESANN} 2014, Bruges, Belgium, April 23-25, 2014}, 2014.
\newblock URL
  \url{http://www.elen.ucl.ac.be/Proceedings/esann/esannpdf/es2014-153.pdf}.

\bibitem[Langford et~al.(2006)Langford, Oliveira, and Zadrozny]{LangfordOZ06}
John Langford, Roberto Oliveira, and Bianca Zadrozny.
\newblock Predicting conditional quantiles via reduction to classification.
\newblock In \emph{{UAI} '06, Proceedings of the 22nd Conference in Uncertainty
  in Artificial Intelligence, Cambridge, MA, USA, July 13-16, 2006}, 2006.

\bibitem[Linial et~al.(1991)Linial, Mansour, and Rivest]{LinialMR91}
Nathan Linial, Yishay Mansour, and Ronald~L. Rivest.
\newblock Results on learnability and the vapnik-chervonenkis dimension.
\newblock \emph{Inf. Comput.}, 90\penalty0 (1):\penalty0 33--49, 1991.
\newblock URL \url{https://doi.org/10.1016/0890-5401(91)90058-A}.

\bibitem[Littlestone(1987)]{Littlestone87}
Nick Littlestone.
\newblock Learning quickly when irrelevant attributes abound: {A} new
  linear-threshold algorithm.
\newblock \emph{Mach. Learn.}, 2\penalty0 (4):\penalty0 285--318, 1987.
\newblock URL \url{https://doi.org/10.1007/BF00116827}.

\bibitem[Manning et~al.(2008)Manning, Raghavan, and
  Sch{\"{u}}tze]{Manning:2008}
Christopher~D. Manning, Prabhakar Raghavan, and Hinrich Sch{\"{u}}tze.
\newblock \emph{Introduction to information retrieval}.
\newblock Cambridge University Press, 2008.

\bibitem[Marler and Arora(2004)]{Marler2004}
R.T. Marler and J.S. Arora.
\newblock Survey of multi-objective optimization methods for engineering.
\newblock \emph{Structural and Multidisciplinary Optimization}, 26\penalty0
  (6):\penalty0 369--395, Apr 2004.
\newblock ISSN 1615-1488.
\newblock \doi{10.1007/s00158-003-0368-6}.
\newblock URL \url{https://doi.org/10.1007/s00158-003-0368-6}.

\bibitem[McAllester(1999)]{McAllester99}
David~A. McAllester.
\newblock Pac-bayesian model averaging.
\newblock In \emph{Proceedings of the Twelfth Annual Conference on
  Computational Learning Theory, {COLT} 1999, Santa Cruz, CA, USA, July 7-9,
  1999}, pages 164--170, 1999.

\bibitem[McAllester(2003)]{McAllester03}
David~A. McAllester.
\newblock Simplified pac-bayesian margin bounds.
\newblock In \emph{Computational Learning Theory and Kernel Machines, 16th
  Annual Conference on Computational Learning Theory and 7th Kernel Workshop,
  COLT/Kernel 2003, Washington, DC, USA, August 24-27, 2003, Proceedings},
  pages 203--215, 2003.

\bibitem[McSherry and Talwar(2007)]{McSherryT07}
Frank McSherry and Kunal Talwar.
\newblock Mechanism design via differential privacy.
\newblock In \emph{48th Annual {IEEE} Symposium on Foundations of Computer
  Science {(FOCS} 2007), October 20-23, 2007, Providence, RI, USA,
  Proceedings}, pages 94--103, 2007.

\bibitem[Mironov(2017)]{Mironov17}
Ilya Mironov.
\newblock R{\'{e}}nyi differential privacy.
\newblock In \emph{30th {IEEE} Computer Security Foundations Symposium, {CSF}
  2017, Santa Barbara, CA, USA, August 21-25, 2017}, pages 263--275, 2017.

\bibitem[Narasimhan(2018)]{Narasimhan18}
Harikrishna Narasimhan.
\newblock Learning with complex loss functions and constraints.
\newblock In \emph{International Conference on Artificial Intelligence and
  Statistics, {AISTATS} 2018, 9-11 April 2018, Playa Blanca, Lanzarote, Canary
  Islands, Spain}, pages 1646--1654, 2018.
\newblock URL \url{http://proceedings.mlr.press/v84/narasimhan18a.html}.

\bibitem[Narasimhan et~al.(2015)Narasimhan, Ramaswamy, Saha, and
  Agarwal]{NarasimhanRS015}
Harikrishna Narasimhan, Harish~G. Ramaswamy, Aadirupa Saha, and Shivani
  Agarwal.
\newblock Consistent multiclass algorithms for complex performance measures.
\newblock In \emph{Proceedings of the 32nd International Conference on Machine
  Learning, {ICML} 2015, Lille, France, 6-11 July 2015}, pages 2398--2407,
  2015.
\newblock URL \url{http://jmlr.org/proceedings/papers/v37/narasimhanb15.html}.

\bibitem[Nesterov(2005)]{Nesterov05}
Yurii Nesterov.
\newblock Smooth minimization of non-smooth functions.
\newblock \emph{Math. Program.}, 103\penalty0 (1):\penalty0 127--152, 2005.

\bibitem[Neyman et~al.(1933)Neyman, Pearson, and Pearson]{NeymanPearson33}
Jerzy Neyman, Egon~Sharpe Pearson, and Karl Pearson.
\newblock Ix. on the problem of the most efficient tests of statistical
  hypotheses.
\newblock \emph{Philosophical Transactions of the Royal Society of London.
  Series A, Containing Papers of a Mathematical or Physical Character},
  231\penalty0 (694-706):\penalty0 289--337, 1933.
\newblock \doi{10.1098/rsta.1933.0009}.
\newblock URL
  \url{https://royalsocietypublishing.org/doi/abs/10.1098/rsta.1933.0009}.

\bibitem[Nguyen et~al.(2005)Nguyen, Wainwright, and Jordan]{NJW05}
XuanLong Nguyen, Martin~J. Wainwright, and Michael~I. Jordan.
\newblock On divergences, surrogate loss functions, and decentralized
  detection.
\newblock \emph{CoRR}, abs/math/0510521, 2005.
\newblock URL \url{http://arxiv.org/abs/math/0510521}.

\bibitem[Nguyen et~al.(2009)Nguyen, Wainwright, and Jordan]{NWJ09}
XuanLong Nguyen, Martin~J. Wainwright, and Michael~I. Jordan.
\newblock {On surrogate loss functions and f-divergences}.
\newblock \emph{The Annals of Statistics}, 37\penalty0 (2):\penalty0 876 --
  904, 2009.
\newblock \doi{10.1214/08-AOS595}.
\newblock URL \url{https://doi.org/10.1214/08-AOS595}.

\bibitem[Rennie and Srebro(2005)]{Rennie05}
Jason Rennie and Nathan Srebro.
\newblock Loss functions for preference levels: Regression with discrete
  ordered labels.
\newblock \emph{Proceedings of the IJCAI Multidisciplinary Workshop on Advances
  in Preference Handling}, 01 2005.

\bibitem[Sason and Verd{\'{u}}(2016)]{SasonV16}
Igal Sason and Sergio Verd{\'{u}}.
\newblock f-divergence inequalities.
\newblock \emph{{IEEE} Trans. Inf. Theory}, 62\penalty0 (11):\penalty0
  5973--6006, 2016.
\newblock \doi{10.1109/TIT.2016.2603151}.
\newblock URL \url{https://doi.org/10.1109/TIT.2016.2603151}.

\bibitem[Shalev-Shwartz and Ben-David(2014)]{mlbook}
Shai Shalev-Shwartz and Shai Ben-David.
\newblock \emph{Understanding Machine Learning: From Theory to Algorithms}.
\newblock Cambridge University Press, 2014.

\bibitem[Shamir and Zhang(2013)]{Shamir013}
Ohad Shamir and Tong Zhang.
\newblock Stochastic gradient descent for non-smooth optimization: Convergence
  results and optimal averaging schemes.
\newblock In \emph{Proceedings of the 30th International Conference on Machine
  Learning, {ICML} 2013, Atlanta, GA, USA, 16-21 June 2013}, pages 71--79,
  2013.
\newblock URL \url{http://jmlr.org/proceedings/papers/v28/shamir13.html}.

\bibitem[Sheffet(2019)]{Sheffet19}
Or~Sheffet.
\newblock Differentially private ordinary least squares.
\newblock \emph{J. Priv. Confidentiality}, 9\penalty0 (1), 2019.
\newblock URL \url{https://doi.org/10.29012/jpc.654}.

\bibitem[{Shokri} et~al.(2017){Shokri}, {Stronati}, {Song}, and
  {Shmatikov}]{SSSS17}
R.~{Shokri}, M.~{Stronati}, C.~{Song}, and V.~{Shmatikov}.
\newblock Membership inference attacks against machine learning models.
\newblock In \emph{2017 IEEE Symposium on Security and Privacy (SP)}, pages
  3--18, 2017.
\newblock \doi{10.1109/SP.2017.41}.

\bibitem[Steinke and Ullman(2015)]{SteinkeU15}
Thomas Steinke and Jonathan Ullman.
\newblock Between pure and approximate differential privacy.
\newblock \emph{CoRR}, abs/1501.06095, 2015.
\newblock URL \url{http://arxiv.org/abs/1501.06095}.

\bibitem[Talwar et~al.(2014)Talwar, Thakurta, and Zhang]{TalwarT014}
Kunal Talwar, Abhradeep Thakurta, and Li~Zhang.
\newblock Private empirical risk minimization beyond the worst case: The effect
  of the constraint set geometry.
\newblock \emph{CoRR}, abs/1411.5417, 2014.
\newblock URL \url{http://arxiv.org/abs/1411.5417}.

\bibitem[Talwar et~al.(2015)Talwar, Thakurta, and Zhang]{TalwarTZ15}
Kunal Talwar, Abhradeep Thakurta, and Li~Zhang.
\newblock Nearly optimal private {LASSO}.
\newblock In \emph{Advances in Neural Information Processing Systems 28: Annual
  Conference on Neural Information Processing Systems 2015, December 7-12,
  2015, Montreal, Quebec, Canada}, pages 3025--3033, 2015.

\bibitem[Vadhan(2017)]{Vadhan17}
Salil~P. Vadhan.
\newblock The complexity of differential privacy.
\newblock In \emph{Tutorials on the Foundations of Cryptography.}, pages
  347--450. 2017.

\bibitem[Valiant(1984)]{Valiant84}
Leslie~G. Valiant.
\newblock A theory of the learnable.
\newblock \emph{Commun. {ACM}}, 27\penalty0 (11):\penalty0 1134--1142, 1984.

\bibitem[Vapnik(2000)]{Vapnik00}
Vladimir Vapnik.
\newblock \emph{The Nature of Statistical Learning Theory}.
\newblock Statistics for Engineering and Information Science. Springer, 2000.

\bibitem[Wang et~al.(2018)Wang, Ye, and Xu]{WangYX18}
Di~Wang, Minwei Ye, and Jinhui Xu.
\newblock Differentially private empirical risk minimization revisited: Faster
  and more general.
\newblock \emph{CoRR}, abs/1802.05251, 2018.
\newblock URL \url{http://arxiv.org/abs/1802.05251}.

\bibitem[Wang et~al.(2015)Wang, Xing, Asif, and Ziebart]{NIPS2015_dfa92d8f}
Hong Wang, Wei Xing, Kaiser Asif, and Brian Ziebart.
\newblock Adversarial prediction games for multivariate losses.
\newblock In C.~Cortes, N.~Lawrence, D.~Lee, M.~Sugiyama, and R.~Garnett,
  editors, \emph{Advances in Neural Information Processing Systems}, volume~28.
  Curran Associates, Inc., 2015.
\newblock URL
  \url{https://proceedings.neurips.cc/paper/2015/file/dfa92d8f817e5b08fcaafb50d03763cf-Paper.pdf}.

\bibitem[Wipf and Rao(2004)]{WipfR04}
David~P. Wipf and Bhaskar~D. Rao.
\newblock L{\_}0-norm minimization for basis selection.
\newblock In \emph{Advances in Neural Information Processing Systems 17 [Neural
  Information Processing Systems, {NIPS} 2004, December 13-18, 2004, Vancouver,
  British Columbia, Canada]}, pages 1513--1520, 2004.
\newblock URL
  \url{https://proceedings.neurips.cc/paper/2004/hash/b1c00bcd4b5183705c134b3365f8c45e-Abstract.html}.

\bibitem[Zhang(2006)]{Zhang06a}
Tong Zhang.
\newblock Information-theoretic upper and lower bounds for statistical
  estimation.
\newblock \emph{{IEEE} Trans. Inf. Theory}, 52\penalty0 (4):\penalty0
  1307--1321, 2006.

\end{thebibliography}

\appendix

\section{Some Inequalities and Advanced Composition Result}

% \begin{lemma}
% Let $Z\sim\calN(0, \mathbb{I}_{K\times K})$. Then for any (fixed) vector $v\in\reals^K,$
% $$
% \brackets{Z, v} \sim \calN(0, \norm{v}^2_2 )
% $$
% \label{lem:1}
% \end{lemma}

% \begin{lemma}
% Let $Z\sim\calN(0, 1)$. Then for any $t > 1,$
% $$
% \pr[|Z| > t] \leq 2e^{-t^2/2}
% $$
% \label{lem:2}
% \end{lemma}

\begin{lemma}[Markov's Inequality]
Let $Z$ be a non-negative random variable. Then for all $a \geq 0$,
$$
\pr[Z \geq a]\leq \frac{\E[Z]}{a},
$$
where $\E[Z] = \int_{x=0}^\infty\pr[Z\geq x]dx$.
\end{lemma}

\begin{lemma}[Cauchy-Schwarz Inequality]
For any two vectors $\bu, \bv$ of an inner product space we have that
$$
|\brackets{\bu, \bv}| \leq \norm{\bu}\norm{\bv}.
$$
\label{lem:cs}
\end{lemma}

\begin{lemma}[Jensen's Inequality]
For any $t\in[0, 1]$ and convex function $f:\calX\rightarrow\reals$, the following holds:
$$f(t\bx_1 + (1-t)\bx_2)\leq tf(\bx_1) + (1-t)f(\bx_2),$$
for any $\bx_1, \bx_2\in\calX$.
\label{lem:ji}
\end{lemma}

\begin{lemma}[Advanced Composition (\cite{DworkRV10})]
For all $\eps, \delta, \bar{\delta} \geq 0$, the class of $(\eps, \delta)$-differentially private mechanisms
results in $(\bar{\eps}, k\delta + \bar\delta)$-differential privacy under $k$-fold adaptive composition for:
$$
\bar\eps = \eps\sqrt{2k\log(1/\bar\delta)} + k\eps(e^\eps-1).
$$
\label{lem:ac}
\end{lemma}

A corollary of Lemma~\ref{lem:ac} states that we can set
$\eps = \frac{\bar\eps}{2\sqrt{2k\log(1/\bar\delta)}}$ to ensure
$(\bar{\eps}, k\delta + \bar\delta)$ differential privacy overall for target
privacy parameters $\bar{\eps}\in (0, 1), \bar{\delta}\in (0, 1]$.

\section{Projection-Free Convex Optimization using Frank-Wolfe}

In previous sections, we presented linear optimization based algorithms to solve the
constrained group-objective optimization problem with and without privacy. In this section,
we present another linear optimization based algorithm that is an adaptation of
the Frank-Wolfe projection-free algorithm.
Algorithm~\ref{alg:fw} is the Frank-Wolfe algorithm presented in
~\citep{Jaggi13} which allows for use of approximate linear optimizers
to solve the subproblems in each iteration $t\in[T]$. The original
version of Frank-Wolfe only allowed for exact linear optimizers.

Lemma~\ref{lem:iters} presents the convergence guarantees of the
algorithm after $k$ iterations in terms of the curvature constant of the
function $h: \reals^K\rightarrow\reals$ 
to be optimized (see Definition~\ref{def:cf}) and the accuracy of
the linear optimizer oracle.
We will rely on a modified version of Algorithm~\ref{alg:fw} to solve
the \textbf{constrained group-objective optimization
problem} (Definition~\ref{def:goo}).

The execution of Algorithm~\ref{alg:fw} in each iteration relies on the availability of
an approximate linear oracle of the form
$\min_{\hat{s}\in\calC}\brackets{\hat{s}, \nabla h(y^{(t)})}$.
In~\citep{AlabiIK18}, the authors provide an algorithm that can optimize any
Lipschitz-continuous function $h: \reals^K\rightarrow\reals$ using an approximate
linear optimizer as an oracle solver.

\begin{algorithm}
Let $y^{(0)}\in\calC$

\For {$t=0, \dots, T$} {
    Let $\gamma = \frac{2}{t+2}$

    Find $s\in\calC$ s.t. $\brackets{s, \nabla h(y^{(t)})} \leq \min_{\hat{s}\in\calC}\brackets{\hat{s}, \nabla h(y^{(t)})} + \frac{1}{2}\rho\gamma C_h$

    Update $y^{(t+1)} = (1-\gamma)y^{(t)} + \gamma s$

}

\caption{Frank-Wolfe algorithm~\citep{Jaggi13}}\label{alg:fw}
\end{algorithm}

For the results in this paper that rely on solving approximate linear subproblems,
$y^{(t)}$ in Algorithm~\ref{alg:fw} would correspond to the $K$-dimensional
loss vector $\ell(c, D)\in[0, 1]^K$ defined for a decision
$c\in\calC_P$ on dataset $D$ of size $n$.

\begin{definition}
The curvature constant $C_h$ of a convex and differentiable function $h:\reals^K\rightarrow\reals$, with
respect to a compact domain $\calC$ is defined as
$$
C_h = \sup_{x, s\in\calC, \gamma\in[0, 1], y=x+\gamma(s-x)}\frac{2}{\gamma^2}\left(h(y)-h(x)-\brackets{y-x, \nabla h(x)}\right),
$$
where $\nabla h$ is the gradient of the function $h$.
\label{def:cf}
\end{definition}

\begin{lemma}
Let $h:\calC\rightarrow\reals$ be any convex function, then using the Frank-Wolfe algorithm
(Algorithm~\ref{alg:fw}), we have that
for any $t \geq 1$ and iterates $y^{(t)},$
$$
h(y^{(t)}) - h(y^*) \leq \frac{2C_h}{t+2}(1+\rho),
$$
where $y^*\in\argmin_{y\in\calC}h(y)$, $\rho\geq 0$ is the accuracy to which the internal linear
subproblems are solved, and $C_h$ is the curvature constant of the function $h$.
\label{lem:iters}
\end{lemma}

\begin{lemma}
Let $h$ be a convex and differentiable Lipschitz-continuous function with gradient $\nabla h$ w.r.t. some
norm $\norm{\cdot}$ over domain $\calC$. If $\nabla h$ has Lipschitz constant $\beta_h > 0$, then
$$
C_h \leq \diam_{\norm{\cdot}}(\calC)^2\beta_h,
$$
where $\diam_{\norm{\cdot}}(\calC)$ is the diameter of $\calC$.
\label{lem:cf}
\end{lemma}

Now we describe how to solve constrained group objective optimization using
the Frank-Wolfe algorithm as opposed to the algorithm of
~\cite{AlabiIK18}.

\subsection{Frank-Wolfe Algorithm for Constrained Group-Objective Optimization}

In this section, we present a 
projection-free algorithm based on Frank-Wolfe that solves the
constrained group-objective optimization problem without privacy considerations.

In each iteration $t\in[T]$, Algorithm~\ref{alg:cgofwcs} solves the following linear sub-problem
$$c_{t+1} = \argmin_{c\in \calC_P}\brackets{\,\, (\nabla f(\ell(c_t, D)) + G\cdot\ind[g(\ell(c_t, D)) \geq 0]\cdot \nabla g(\ell(c_t, D))), \loss\,\,},$$
and then ``moves'' towards $c_{t+1}$ by a multiplicative factor of $\frac{2}{t+2}$. We show
that after $T = O\left(\frac{K\sqrt{K}}{\alpha^2}\log \frac{K\sqrt{K}}{\alpha^2}\right)$ iterations
we will get a decision $\hat{c}$ that is within $\alpha$ (in terms of $f, g$) of the optimal
decision $c^* \in \argmin_{c\in\calC_P:g(\loss)\leq 0}f(\loss)$.
The function $\Unif(\{c_1, \ldots, c_T\})$ returns $\hat{c}$ that predicts with decision
$c_i$ with probability $\frac{1}{T}$ for any $i\in[T]$.

\begin{algorithm}

\KwIn{$\argmin_{\calC_P}\brackets{\cdot, \cdot}, T, \ell, \nabla f, \nabla g, D \in (\calX\times\calA\times\calY)^n, G$}

\

Pick any decision $c\in\calC_P$ as $c_1$ with $l_1(D) = \ell(c, D)$

\

\For {$t=1, \ldots, T-1$} {

  $c_{t+1} = \argmin_{c\in \calC_P}\brackets{\,\,  \loss, (\nabla f(\ell(c_t, D)) + G\cdot\ind[g(\ell(c_t, D)) \geq 0]\cdot \nabla g(\ell(c_t, D)))\,\,}$

  $l_{t+1}(D) = \left(1-\frac{2}{t+2}\right)l_t(D) + \frac{2}{t+2}\ell(c_{t+1}, D)$
}

\

\Return $\hat{c}=\Unif(\{c_1, \ldots, c_T\})$
\caption{Frank-Wolfe algorithm for Constrained Group-Objective Convex Optimization.}
\label{alg:cgofwcs}
\end{algorithm}

\begin{lemma}

Suppose we are given convex, smooth functions
$f, g: [0, 1]^K\rightarrow\reals$ and loss function
$\ell: \calC_P\times(\calX\times\calA\times\calY)^n\rightarrow[0, 1]^K$.
Then for any setting of $G > 0, T\geq 3$, Algorithm~\ref{alg:cgofwcs}
returns a decision $\hat{c}\in\calC_P$ with the following guarantee:

$$
\E[f(\ell(\hat{c}, D))] \leq f(\bloss) + \frac{2K(\beta_f+G\beta_g)\log T}{T}
,\quad\quad
\E[g(\ell(\hat{c}, D))] \leq \frac{2K(\beta_f+G\beta_g)\log T}{T}\cdot\frac{1}{G} + \frac{\sqrt{K}}{G},
$$

where
$c^* \in \argmin_{c\in\calC_P:g(\loss)\leq 0}f(\loss)$ is the best decision
in the feasible set $\calC_P$, $D$ is a dataset of size $n$, $\beta_f$ is the smoothness parameter
of the function $f$ and $\beta_g$ is the smoothness parameter of the function $g$.
This result holds assuming
access to a linear optimizer oracle.

\label{lem:fw}
\end{lemma}

\begin{proof}

The intuition is to optimize w.r.t. a ``new'' convex, smooth function
$h(\ell(c, D)) = f(\ell(c, D)) + G\cdot\max(0, g(\ell(c, D)))$ for any $c\in\calC_P$ and dataset $D$.
We rely on the primal convergence guarantees of the Frank-Wolfe algorithm.
By Lemma~\ref{lem:iters}, we have that in iteration $t\geq 1$, the following holds:
$$h(\ell(c_t, D)) \leq h(\ell(c^*, D)) + \frac{2C_h}{t+2},$$
where $C_h$ is the curvature constant of $h$. And by Lemma~\ref{lem:cf}, we have $C_h \leq K(\beta_f + G\beta_g)$ where $(\beta_f + G\beta_g)$ is the smoothness parameter of the function $h$.
As a result we have that
$h(\ell(c_t, D)) \leq h(\ell(c^*, D)) + \frac{2K(\beta_f + G\beta_g)}{t+2}$.

Then by Jensen's inequality (Lemma~\ref{lem:ji}) and convexity
of $h\circ\ell$, we have
\begin{align}
    \E_{c\sim\{c_i\}_{i=1}^T}[h(\ell(c, D))] &\leq \frac{1}{T}\sum_{t=1}^T h(\ell(c_t, D)) \\
    &\leq h(\ell(c^*, D)) + \frac{2K(\beta_f + G\beta_g)}{T}\left(\frac{1}{3}+\frac{1}{4}+\frac{1}{5}+\ldots+\frac{1}{T+2}\right)  \\
    &\leq h(\ell(c^*, D)) + \frac{2K(\beta_f + G\beta_g)\log T}{T}
\end{align}
where we used that the harmonic number $H_T$ is upper bounded by $\log T + 1$ and that
$\frac{1}{T+1} + \frac{1}{T+2} < 1/2$ for $T \geq 3$.

We apply Lemma~\ref{lem:h} to obtain that
$\E[f(\ell(\hat{c}, D))] \leq f(\bloss) + \frac{2K(\beta_f + G\beta_g)\log T}{T}$ and
$\E[g(\ell(\hat{c}, D))] \leq \frac{2K(\beta_f + G\beta_g)\log T}{T}\cdot\frac{1}{G} + \frac{\sqrt{K}}{G}$.
\end{proof}

\begin{corollary}

Suppose we are given convex, smooth functions
$f, g: [0, 1]^K\rightarrow\reals$ and loss function
$\ell: \calC_P\times(\calX\times\calA\times\calY)^n\rightarrow[0, 1]^K$.
Let $\beta_f, \beta_g$ be the Lipschitz constants of the gradients of $f, g$ respectively.
Then for any setting of $T\geq 3, \alpha > 0$, 
Algorithm~\ref{alg:cgofwcs}
returns a decision $\hat{c}\in\calC_P$ with the following guarantee:
$$\E[f(\ell(\hat{c}, D))] \leq f(\bloss) + \alpha, \text{ and } \E[g(\ell(\hat{c}, D))] \leq \alpha,$$
provided that
$G \geq \frac{\alpha + \sqrt{K}}{\alpha}$ and $T \geq \frac{2K(\beta_f+G\beta_g)}{\alpha}\log\frac{2K(\beta_f+G\beta_g)}{\alpha}$ where
$c^* \in \argmin_{c\in\calC_P:g(\loss)\leq 0}f(\loss)$ is the best decision
in the feasible set $\calC_P$, $D$ is a dataset of size $n$, $\beta_f$ is the smoothness parameter
of the function $f$ and $\beta_g$ is the smoothness parameter of the function $g$.
This result holds assuming
access to a linear optimizer oracle.
\label{cor:fwcs}

\end{corollary}

\begin{proof}
The corollary follows by applying Lemmas~\ref{cor:h} and~\ref{lem:fw} where we set $G \geq \frac{\alpha + \sqrt{K}}{\alpha}$ and 
$$T \geq \frac{2K(\beta_f+G\beta_g)}{\alpha}\log\frac{2K(\beta_f+G\beta_g)}{\alpha}$$
in Algorithm~\ref{alg:cgofwcs}.
\end{proof}

We adapted the Frank-Wolfe algorithm to solve the constrained
group-objective optimization problem (via calls to a linear optimizer). Given the versatility of
(stochastic) gradient descent, we could also
solve the problem via the use of gradient descent.
For example, Theorem 3.7 from~\citep{Bubeck15} gives guarantees for projected gradient
descent optimization of convex, smooth functions. And Theorem 2 from~\cite{Shamir013} gives
guarantees for convex (but not necessarily smooth) functions.

\section{Using a Private Cost Sensitive Classification Oracle $\cs_{\eps'}$}
\label{sec:csc}

\cite{AgarwalBD0W18} present an exponentiated gradient algorithm for fair
classification. \cite{Jagielski18} essentially modify this algorithm, making it
differentially private, and term the new algorithm
\textbf{DP-oracle-learner}. \textbf{DP-oracle-learner} satisfies $(\eps, \delta)$-differential privacy relying on a private cost sensitive classification oracle
$\cs_{\eps'}(\calH)$ in each iteration where
$\eps' = \frac{\eps}{4\sqrt{T\log(1/\delta)}}$ and $\calH$ is the hypothesis class to be
learned.

\textbf{DP-oracle-learner} solves the $\gamma$-fair empirical risk minimization (ERM)
problem
given by:
\begin{framed}
\noindent $\min_{Q\in\Delta(\calH)}\hat{\err}(Q)$ \\
\noindent s.t. $\forall\, 0\neq a\in\calA$: $\Delta\hat{FP}_a(Q)=|\hat{FP}_a(Q)-\hat{FP}_0(Q)|\leq \gamma$,\,\,\,$\Delta\hat{TP}_a(Q)=|\hat{TP}_a(Q)-\hat{TP}_0(Q)|\leq \gamma$\\

\end{framed}
where $\hat{FP}_a(Q), \hat{TP}_a(Q)$ are empirical estimates of
$FP_a(Q) = \pr_{(x, y, a)}[Q(x) = 1 | A = a, y = 0]$,
$TP_a(Q) = \pr_{(x, y, a)}[Q(x) = 1 | A = a, y = 1]$ respectively
and group 0 is used as an anchor.
$\calA$ is the set of labels for all protected/sensitive attributes and $A$ is the random
variable over $\calA$.

\cite{AgarwalBD0W18} solve the following specific Lagrangian min-max problem:
$$
\min_{Q\in\Delta(\calH)}\max_{\vec{\lambda}\in\Lambda} L(Q, \vec{\lambda})
= \hat{\err}(Q) + \vec{\lambda}^T\hat{\vec{r}}(Q),
$$
where $\Delta(\calH)$ is the set of all randomized classifiers that can be obtained
by hypotheses in $\calH$, $\hat{\vec{r}}(Q)$ is a vector representing the
fairness violations of the classifier $Q$ across all groups,
$\vec{\lambda}\in\Lambda = \{\vec{\lambda}:\norm{\vec{\lambda}}_1 \leq B\}$, and
the bound $B$ is chosen to ensure convergence.

Now we state a main theorem from~\citep{Jagielski18}.

\begin{theorem}[Theorem 4.4 from~\citep{Jagielski18}]
Let $(\tilde{Q}, \tilde{\vec{\lambda}})$ be the output of \textbf{DP-oracle-learner},
an $(\eps, \delta)$-differentially private algorithm, in
~\citep{Jagielski18} and let $Q^*$ be a solution to the non-private $\gamma$-fair
ERM problem (see above definition). Then with probability at least
$0.99,$
$$
\hat{\err}(\tilde{Q}) \leq \hat{\err}(Q^*) + 2\nu,
$$
and for all $a\neq 0$,
$$
\Delta\hat{FP}_a(\tilde{Q}) \leq \gamma + \frac{1+2\nu}{B},
$$
$$
\Delta\hat{TP}_a(\tilde{Q}) \leq \gamma + \frac{1+2\nu}{B},
$$
where
$\nu = \tilde{O}\left(\frac{B}{\min_{a, y}\hat{q}_{ay}}\sqrt{\frac{|\calA|\cdot\VC(\calH)}{n\eps}}\right),$
$n$ is the  number of training examples $\{(\bx_i, a_i, y_i)\}^n_{i=1}$
fed to the Algorithm, $\VC(\calH)$ is the VC dimension of $\calH$,
and $\hat{q}_{ay}$ is an empirical estimate for
$\pr[A = a, Y = y]$.
\label{thm:csc}
\end{theorem}

As is done in their paper, to get rid of the algorithmic-specific dependence on the bound $B$ (for example, in Theorem 4.6 of their paper),
we set $B = |\calA|$.

\begin{corollary}
Let $(\tilde{Q}, \tilde{\vec{\lambda}})$ be the output of Algorithm 3,
an $(\eps, \delta)$-differentially private algorithm, in
~\citep{Jagielski18} and let $Q^*$ be a solution to the non-private $\alpha$-fair
ERM problem (see above definition) where $\min_{a, y}\{\hat{q}_{ay}\} \leq \alpha^{(1+r)/2}$
for any $r > 0$.
Then with probability at least 9/10, in order to
solve the ERM problem with classifier error $\alpha = (\hat{\err}(\tilde{Q}) - \hat{\err}(Q^*))$ and maximum fairness violation
of $\alpha = \max_{a\in\calA}\max(\Delta\hat{FP}_a(\tilde{Q}), \Delta\hat{TP}_a(\tilde{Q}))$, we could use training examples of size
$$
n = \tilde{O}\left(\frac{|\calA|^{3}\cdot\VC(\calH)}{\eps\alpha^{3+r}}\right).
$$
\label{cor:csc}
\end{corollary}

\begin{proof}
Set $\min_{a, y}\{\hat{q}_{ay}\} \leq \alpha^{(1+r)/2}$ in Theorem~\ref{thm:csc}.
\end{proof}

Corollary~\ref{cor:csc} is obtained directly from Theorem~\ref{thm:csc}
by solving for $n$ in terms of the excess risk and
will be used as the point of comparison to compare our work to the
work of~\cite{Jagielski18} for solving the $\alpha$-fair ERM
problem (for the Equalized Odds problem) with classifier error of $\alpha$ with
probability at least 9/10. Further note that $K = O(|\calA|)$ is the number of groups
(in terminology used in other parts of this paper).
A result of this corollary is that to solve the
$\cgoo(\calC_P, n, K, f, g, \ell, D, \alpha)$ problem when the constraint is to satisfy
$\alpha$-equalized odds, we need sample size 
$n = \tilde{O}\left(\frac{|\calA|^{3}\cdot\VC(\calH)}{\eps\alpha^{3+r}}\right)$.
In comparison, using a generic implementation of the private oracle $\lopt_{\eps, \delta}$, we
need sample size $n = \tilde{O}\left(\frac{|\calA|^{4}\cdot\VC(\calH)}{\eps\alpha^{3}}\right)$ which is
asymptotically better (in terms of the accuracy parameter $\alpha > 0$) for all $r > 0$.

Note that according to 
Assumption C.1 in~\citep{Jagielski18}, the $\cs(\calH)$ oracle and its private
counterpart are often implemented via learning heuristics.

\end{document}